\documentclass[12pt]{article}

\usepackage{amsmath}
\usepackage{times}
\usepackage[dvipsone]{graphicx}
\usepackage{color}
\usepackage{multirow}
\usepackage[authoryear]{natbib}
\usepackage{rotating}
\usepackage{bbm}
\usepackage{latexsym}
\usepackage{epstopdf}   

\usepackage{romannum}
\AtBeginDocument{\pagenumbering{arabic}}

\makeatletter
\@addtoreset{equation}{section}
\makeatother
\renewcommand{\theequation}{\arabic{section}.\arabic{equation}}

\newcommand{\captionfonts}{\normalsize}

\makeatletter
\long\def\@makecaption#1#2{%
  \vskip\abovecaptionskip
  \sbox\@tempboxa{{\captionfonts #1: #2}}%
  \ifdim \wd\@tempboxa >\hsize
    {\captionfonts #1: #2\par}
  \else
    \hbox to\hsize{\hfil\box\@tempboxa\hfil}%
  \fi
  \vskip\belowcaptionskip}
\makeatother

\usepackage{amsmath}
\usepackage{mathrsfs}

\usepackage{algorithm}
\usepackage{algorithmic}

\usepackage{amsthm}

\usepackage{pmat}
\usepackage{subfig}
\usepackage{caption}

\usepackage{amssymb}

\usepackage[retainorgcmds]{IEEEtrantools}

\renewcommand{\theequationdis}{{\normalfont (\theequation)}}
\renewcommand{\theIEEEsubequationdis}{{\normalfont (\theIEEEsubequation)}}

\makeatletter
\newcommand*{\rom}[1]{\expandafter\@slowromancap\romannumeral #1@}
\makeatother

\newtheorem{thm}{Theorem}
\newtheorem{lem}{Lemma}
\newtheorem{dfn}{Definition}
\newtheorem{prp}{Proposition}
\newtheorem*{rmk}{Remark}
\newtheorem{rmk-1}{Remark}
\newtheorem{rmk-2}{Remark}
\newtheorem{rmk-3}{Remark}
\newtheorem{rmk-4}{Remark}
\newtheorem{rmk-5}{Remark}
\newtheorem{rmk-6}{Remark}
\newtheorem{rmk-7}{Remark}
\newtheorem{rmk-8}{Remark}
\newtheorem{cl}{Corollary}

\newtheorem{assm}{Assumption}

\usepackage[hyperindex,breaklinks]{hyperref}
\hypersetup{
           breaklinks=true,   
           colorlinks=true,   
           pdfusetitle=true,  
        }
\usepackage{setspace}
\begin{document}

\ \vspace{20mm}\\

{\LARGE \flushleft  On the Principles of Deep Feedforward ReLU Networks}

\ \\
{\bf \large Changcun Huang}\\
{cchuang@mail.ustc.edu.cn}\\
{Shuitu Institute of Applied Mathematics, Chongqing 400700, P.R.C}\\
%


\thispagestyle{empty}
\markboth{}{NC instructions}
\ \vspace{-0mm}\\
%
\begin{center} {\bf Abstract} \end{center}
The architecture of deep feedforward neural networks is ubiquitous in deep learning, either as a whole system or as a subnetwork of other architectures, and thus its mechanism is a key ingredient of the black box of neural networks. On the basis of the simplest two-layer ReLU network, this paper systematically studies the mechanism of deep feedforward ReLU networks with multiple hidden layers and successfully explains the training solution obtained by the back-propagation algorithm. The concept of a path, especially in terms of the relationships between paths, plays a central role in uncovering the mystery of the black box. It is shown that a unit of a deep ReLU network can form a piecewise linear manifold to divide the input space, instead of a hyperplane of the two-layer case. How to efficiently use the hidden-layer units to produce both linear functions and partitions of the input space is also a central problem. The principles of a two-layer ReLU network can be generalized to the deeper case to a large extent, such as multiple strict partial orders and continuity restriction. The combination of the basic and simple principles proposed can yield complicated instantiations including the training solutions, and in this sense the black box of deep feedforward ReLU networks is revealed.


\ \\[-2mm]

\section{Introduction}
The principles of ReLU networks with one-hidden layer have been unravelled by the author's work \citet*{Huang2024}, and it's natural to ask whether they can be generalized to the deep-layer case. We expect that the clue from the simplest neural network could lead to the understanding of more complex architectures. This paper will prove the feasibility of this direction.

\subsection{Background}
A deep feedforward neural network or a multilayer perceptron (MLP) \citep*{Haykin2009,Du2022} with multiple hidden layers is ubiquitous, either as a whole model (e.g., \citet*{Jagtap2020}, \citet*{Raissi2019}) or as a component of other models (e.g., \citet*{Vaswani2017}, \citet*{Girin2022}), so its mechanism is fundamental in deep learning. The multilayer structure analogous to a MLP also exists in the brain's neural networks \citep*{Kandel2021} and thus the study of MLPs may be beneficial to neural science.

The black box of MLPs is also important for practical considerations. The safety \citep*{Bengio2025,Hendrycks2025} and energy consumption \citep*{Strubell2019, Argerich2024} of AI is a big problem especially after the advent of large language models (e.g., ChatGPT). The clarity of the mechanism of deep learning is a urgent demand in developing interpretable AI systems, controlling potential AI risks and designing economic AI models.

The main challenge is the complexity of network architectures, such as large depth, various connections and huge amount of parameters, seemingly beyond the capability of the usual scientific research methodology. Two main streams exist for this problem: one is by experimental observations, popular among computer scientists, such as \citet*{Erhan2010}, \citet*{Glorot2010}, \citet*{Zhang2016}, and \citet*{Ramanujan2020}; and the other is by pure mathematical deduction usually done by mathematicians (e.g.,\citet*{Yarotsky2017}, \citet*{Shen2021}, \citet*{DeVore2021}, \citet*{Daubechies2022}, \citet*{Guth2024}, and \citet*{Yang2025}). The former doesn't reach a unified theoretical framework, while the latter is nearly not relevant to applications. Thus, a theory for the black box of deep learning is still absent.

\subsection{Research Methodology}
This paper is on the basis of \citet*{Huang2024} for two-layer ReLU networks and can incorporate it into an integrated deduction system. However, due to the studies at different times, we here present the new concepts and results in a separate article. The research methodology is of theoretical physics and the ultimate goal is to use deductive theory to explain experiments---the training solution of a deep feedforward ReLU network obtained by the back-propagation algorithm \citep*{Rumelhart1986}.

The conclusions, in terms of a theorem, proposition or corollary, are not unintentionally listed dull facts, but are carefully chosen by the philosophy that they should be as simple and less as possible and simultaneously can explain as many experimental phenomena as possible.

A theorem in this paper is usually an abstract and general conclusion and may not be directly applied, while a corollary is more concrete and related to applications. A proposition is somewhat less important or general than a theorem and cannot be a consequence of a theorem. We often give a remark after a conclusion or concept to associate it with applications, to explain its meaning, or to remind the readers of its potential usefulness and generalization.

The conclusions are established not only for fitting experiments, but also for the completeness and neatness of a deduction system, for which the theoretical framework includes more contents than the experimental phenomena presently given.

To evaluate whether our theory is successful, all the criteria for theoretical physics in scientific philosophy can be employed, such as simplicity (or parsimony), consistency, and effectiveness in explaining and predicting  experimental or natural phenomena \citep*{Gauch2003, Bunge1973, Simon2026}.

\subsection{Paper Organization}
The major concern is to implement a desired partition $\mathcal{P}$ of input space $U=[0,1]^n$ as well as a desired piecewise linear function $g(\boldsymbol{x})$ over $\mathcal{P}$ via a deep feedforward ReLU network $\mathfrak{N}$, for which the principles used should be applicable to the training solution.

The paper is organized for that purpose: Section 2 investigates the partition of $U$; Section 3 gives the principles of function construction; Section 4 generalizes section 3 for explaining experiments; Section 5 proposes the mechanism of multiple outputs; Section 6 studies the univariate-function approximation and section 7 is for the multivariate case; Section 8 explains training solutions; Section 9 summarizes the principles of the black box of deep ReLU networks; Section 10 concludes this paper by a discussion. A more detailed abstract of a section will be given at the beginning of each section.

\section{Partition of Input Space}
The partition of $U$ via a two-layer ReLU network is through $n-1$-dimensional hyperplanes derived from its hidden-layer units. The case of network $\mathfrak{N}$ is different in that a unit of the hidden layers (except for the first one) can form a piecewise linear manifold to divide $U$; and the associated concepts and results are the main contents of this section. Section 2.1 is of some preliminaries. Section 2.2 proposes the basic principles of a piecewise linear manifold obtained by $\mathfrak{N}$. Section 2.3 summarizes the methods to produce piecewise linear manifolds. Section 2.4 generalizes the positive-output (or zero-output) region of a hyperplane to that of a piecewise linear manifold. Section 2.5 studies the properties of the regions of a partition.

\subsection{Preliminaries}
Let $\sigma(x)=\max\{x, 0\}$ be the activation function of a ReLU; when the variable $x$ becomes a vector $\boldsymbol{x}=[x_1, x_2, \dots, x_n]^T$, $\sigma(\boldsymbol{x}) = [\sigma(x_1), \sigma(x_2), \dots, \sigma(x_n)]^T$.
\begin{dfn}[Deep feedforward ReLU network]
A $\Phi$-layer \textsl{deep feedforward ReLU network}, denoted by product $\mathfrak{N}_m:=n\prod_{i=1}^{\Phi}m_{i}m'$, is a network with $n$-dimensional input and m-dimensional output, whose hidden layers and output layer are composed of ReLUs and linear units, respectively, which includes $\Phi$ layers (regardless of the input and output layers) for $\Phi \ge 2$ with the $i$th one having $m_{i}$ units, and which is fully connected between adjacent layers without other links. The symbol $'$ in the expression represents the different type of output-layer units. The index $i$ of a layer is called the \textsl{depth} of that layer and $\Phi$ is the depth of $\mathfrak{N}$.

Especially, the case of one-dimensional output can be written as
\begin{equation}
\mathfrak{N} = n\prod_{i=1}^{\Phi}m_{i}1'.
\end{equation}
Let $u_{ij}$ be the $j$th unit of the $i$th layer of $\mathfrak{N}$. The output of $u_{ij}$ can be expressed as $\sigma(\boldsymbol{w}_{ij}^T\boldsymbol{x}^{(i-1)} + b_{ij})$, where $\boldsymbol{x}^{(i-1)}$ is the input vector of the $i$th layer for $i\ge2$ and $\boldsymbol{x}^{(0)} = \boldsymbol{x} \in [0,1]^n$. We call $\boldsymbol{w}_{ij}$ the \textsl{input-weight vector} of $u_{ij}$ and its entries the \textsl{input weights}, and call $b_{ij}$ the bias of $u_{ij}$. All the input-weight vectors of the $i$th layer comprise the \textsl{weight matrix} of this layer, denoted by
\begin{equation}
W_i =
\begin{pmat}({})
\boldsymbol{w}_{i1}, \boldsymbol{w}_{i2}, \dots, \boldsymbol{w}_{im_{i}} \cr
\end{pmat},
\end{equation}
which is a matrix of size $m_{i-1} \times m_{i}$; similarly, $\boldsymbol{b}_i = [b_{i1}, b_{i2}, \dots, b_{im_i}]^T$ is the bias vector. Then the output of the $i$th layer can be represented as
\begin{equation}
\boldsymbol{x}^{(i)} = \sigma(W_i^T\boldsymbol{x}^{(i-1)} + \boldsymbol{b}_i).
\end{equation}
The $j$th row $\boldsymbol{v}_{j}$ of $W_{i+1}$, namely
\begin{equation}
\boldsymbol{v}_{j}(k) = W_{i+1}(j, k),
\end{equation}
or its transpose version $\boldsymbol{v}_{j}^T$, is called the \textsl{output-weight vector} of $u_{ij}$, whose entries are also called \textsl{output weights}. The input-weight vector of the linear unit of the output layer of $\mathfrak{N}$ is denoted by $\boldsymbol{w}_{\Phi+1}$.
\end{dfn}

\begin{figure}[!t]
\captionsetup{justification=centering}
\centering
\includegraphics[width=2.4in, trim = {2.5cm 2.5cm 4.5cm 3.1cm}, clip]{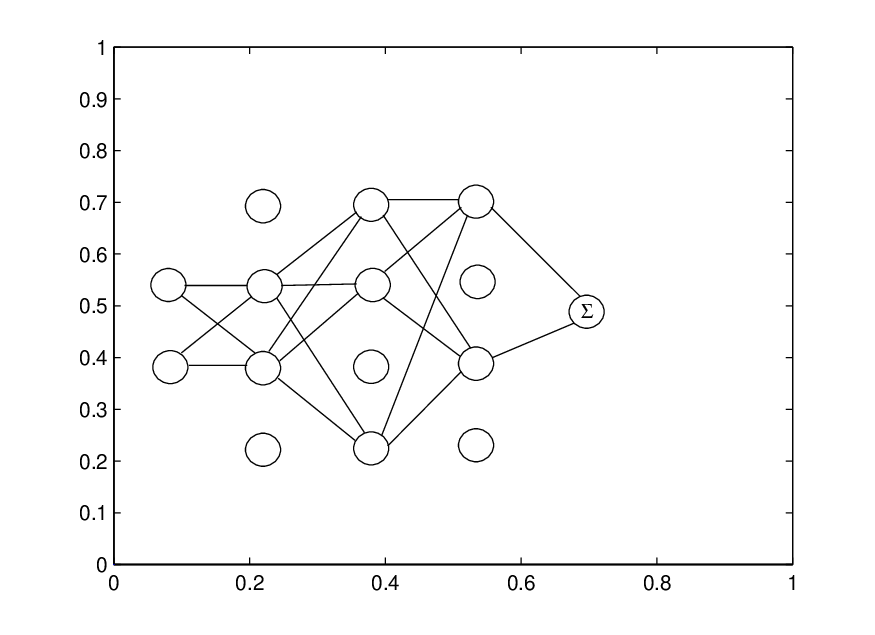}
\caption{A path of deep neural networks.}
\label{Fig.1}
\end{figure}

\begin{rmk}
This paper mainly investigates the single-output case of equation 2.1 whose results can be easily generalized to the multiple-output case (see section 5.1). The notation $\mathfrak{N}$ will be used throughout this paper and the expression of equation 2.1 will not be explicitly given unless otherwise stated.
\end{rmk}

A point $\boldsymbol{x} \in U=[0,1]^n$ activates a unit $\mathcal{U}$ of $\mathfrak{N}$ means that the output of $\mathcal{U}$ with respect to $\boldsymbol{x}$ is positive. Let $L$ be a hyperplane derived from $\mathcal{U}$ and when $\boldsymbol{x}$ activates $\mathcal{U}$ we also say $\boldsymbol{x}$ activates $L$. The closure operation for a region can be found in \citet*{Huang2024}; intuitively speaking, the closure of region $R$ is the set including both $R$ and its boundary.
\begin{dfn}[Region, path and partition]
Suppose that a set $\mathscr{R} \subset U$ divided by $\mathfrak{N}$ activates the same subset $\mathscr{P}$ of the units of $\mathfrak{N}$. The closure $\bar{\mathscr{R}}$ of $\mathscr{R}$ is called a \textsl{region} and $\mathscr{P}$ together with their links are called a \textsl{path} of $\mathfrak{N}$. The set $\bar{\mathscr{R}}-\mathscr{R}$ is the \textsl{boundary} of $\mathscr{R}$. The set of all the regions obtained by $\mathfrak{N}$ is called a \textsl{partition} of input space $U$.
\end{dfn}

\noindent
\textbf{Example}. Figure \ref{Fig.1} gives an example of a path, including a subset of the units as well as the links connecting them.

\begin{dfn}[Regions of a layer]
A region of the $\nu$th layer of network $\mathfrak{N}$ (or path $\mathcal{P}$) is the one formed by a subset of $\mathfrak{N}$ (or $\mathcal{P}$) up to the $\nu$th layer.
\end{dfn}

\begin{prp}[A foundation of Partitions]
Let $\mathcal{N} = m_{\nu}\cdot 1$ be an arbitrary two-layer subnetwork of network $\mathfrak{N}$. Denote by $\mathcal{U}$ the unit of the second layer of $\mathcal{N}$. Suppose that network $\mathfrak{N}$ up to the $\nu$th layer partitions $U$ into a set $R_{\nu} = \{r_j: 1\le j \le \zeta\}$ of regions, with $r_j$ corresponding to path $p_j$. To some fixed $r_j$, let
\begin{equation}
y = \sigma(\boldsymbol{w}^T\boldsymbol{x}^{(\nu)}+b)
\end{equation}
be the output of $\mathcal{U}$ in path $p_j$. Equation 2.5 yields a hyperplane $\boldsymbol{w}^T\boldsymbol{x}^{(\nu)}+b=0$ of the input space, denoted by $\mathcal{L}$, since $\boldsymbol{x}^{(\nu)}$ is a linear function of $\boldsymbol{x} \in U$. Suppose that $\mathcal{L}$ is $n-1$-dimensional. Then to region dividing, there are three possibilities of the relationship between $\mathcal{L}$ and $r_j$: (1) $\mathcal{L}$ divides $r_j$ into two parts; (2) all the points of $r_j$ activate $\mathcal{L}$; (3) $\mathcal{L}$ can not be activated by any point of $r_j$.
\end{prp}
\begin{proof}
The conclusion is obvious by the preceding definitions.
\end{proof}

\begin{dfn}[Unit classification for partitions]
Unit $\mathcal{U}$ of proposition 1 is called a local unit, global unit and inactivated unit of region $r_j$ for cases $(1)$, $(2)$ and $(3)$, respectively. We also say that $r_j$ partially activates, completely activates and inactivates $\mathcal{U}$ for the three cases, respectively.
\end{dfn}

\begin{cl}[Multiple hyperplanes of a unit]
Under case $(1)$, the hyperplanes generated by $\mathcal{U}$ for different regions of $R_{\nu}$ could be distinct, or $\mathcal{U}$ could yield more than one hyperplane dividing different regions.
\end{cl}
\begin{proof}
The reason is that each $r_j$ corresponds to a unique path $p_j$ manifested by distinct $\boldsymbol{x}^{(\nu)}$ of equation 2.5, resulting in possibly different hyperplanes of the input space.
\end{proof}

\begin{prp}[Recurrence formula of the number of regions]
Let $r_{ij}$ for $i=1,2,\dots,\Phi$ and $j=1,2,\dots,\zeta_i$ be the $j$th region of the $i$th layer of $\mathfrak{N}$, with $\zeta_i$ the number of the regions. Then the following recurrence formula
\begin{equation}
\zeta_{\nu} = \sum_{j=1}^{\zeta_{\nu-1}}\xi_{\nu j}
\end{equation}
holds for $\nu \ge 2$, where $\xi_{\nu j}$ is the number of the subregions of $r_{\nu-1,j}$ divided by the hyperplanes of the units of the subsequent $\nu$th layer.
\end{prp}
\begin{proof}
This conclusion is simply the consequence of proposition 1.
\end{proof}
\begin{cl}[Monotonic increase]
The number of regions of each layer of $\mathfrak{N}$ monotonically grows with respect to depth $i$ due to possible region subdivisions, namely, $\zeta_{i+1} \ge \zeta_{i}$ for $i = 1, 2, \dots, \Phi-1$.
\end{cl}
\begin{proof}
The conclusion is by equation 2.6 and proposition 1.
\end{proof}

\begin{cl}[Exponential growth]
In equation 2.6, suppose that $\xi_{\nu j} \ge \mathcal{C}$ for each $\nu$ or the number of the subregions of each $r_{ij}$ is not less than a constant $\mathcal{C}$, and that $\zeta_1 \ge \mathcal{C}$. Then the number of the regions grows exponentially with respect to depth $\nu$ in terms of $\zeta_{\nu} \ge \mathcal{C}^{\nu}$.
\end{cl}
\begin{proof}
The condition of this corollary and equation 2.6 give $\zeta_{\nu} \ge \mathcal{C}\zeta_{\nu-1}$ with $\zeta_1 \ge \mathcal{C}$, implying this conclusion.
\end{proof}

\subsection{Properties of Knots}

\begin{dfn}[Knot and adjacent regions]
To a set $R = \{r_1, r_2, \dots, r_{\zeta}\}$ of the regions of $U$ partitioned by network $\mathfrak{N}$, two regions $r_{\nu}$ and $r_{\mu}$ for $1 \le \nu, \mu \le \zeta$ are said to be \textsl{adjacent}, if the dimensionality $\dim(r_{\nu} \cap r_{\mu})=n-1$; we write $r_{\nu} \frown r_{\mu}$ to denote this adjacent relationship. A \textsl{knot} $\mathcal{K}$ is a part of an $n-1$-dimensional hyperplane that separates $r_{\nu}$ and $r_{\mu}$, or $\mathcal{K} = r_{\nu} \cap r_{\mu}$.
\end{dfn}

\begin{dfn}[Adjacent paths]
Two paths $p_1$ and $p_2$ of $\mathfrak{N}$ are \textsl{adjacent} if they differ from each other only in one unit $\mathcal{U}$, that is, $\mathcal{U} \in p_2$ (or $p_1$) but $\mathcal{U} \notin p_1$ (or $p_2$).
\end{dfn}

\begin{lem}
Let $p_1$ and $p_2$ be two adjacent paths of $\mathfrak{N}$ with a unique unit $\mathcal{U} \notin p_1$ but $\mathcal{U} \in p_2$, corresponding to regions $r_1$ and $r_2$, respectively. Suppose that $\mathcal{U}$ is in the $\nu$th layer of $p_2$ and yields a knot $\mathcal{K}$. Let $\mathscr{U} \in p_1 \cap p_2$ be a unit generating knots in both the two paths, denoted by $k_1$ and $k_2$, respectively; suppose that $\mathscr{U}$ is in a layer whose depth
\begin{equation}
\mu > \nu.
\end{equation}
Let $\boldsymbol{w}^T\boldsymbol{x} + b = 0$, $\boldsymbol{w}_1^T\boldsymbol{x} + b_1 = 0$ and $\boldsymbol{w}_2^T\boldsymbol{x} + b_2 = 0$ be the equations of $\mathcal{K}$, $k_1$ and $k_2$, respectively. Then we have
\begin{equation}
\boldsymbol{w}_2^T\boldsymbol{x} + b_2 = \boldsymbol{w}_1^T\boldsymbol{x} + b_1 + \lambda(\boldsymbol{w}^T\boldsymbol{x} + b),
\end{equation}
in which
\begin{equation}
\lambda = (W_{\nu+2}W_{\nu+3}\dots \boldsymbol{w}_{\mu})^T\boldsymbol{v}
\end{equation}
where $W_{k}$ for $k=\nu+2, \nu+3, \dots, \mu-1$ is the input matrix of the $k$th layer of $p_2$, $\boldsymbol{w}_{\mu}$ is the input-weight vector of $\mathscr{U}$ and $\boldsymbol{v}$ is the output-weight vector of $\mathcal{U}$ with size $m_{\nu+1} \times 1$.
\end{lem}
\begin{proof}
The output $\sigma(\boldsymbol{w}_2^T\boldsymbol{x} + b_2)$ of $\mathscr{U}$ in $p_2$ differing from that of $p_1$ is due to the introduction of $\mathcal{U}$ in a shallower layer; denote this output difference by $\Delta$. Under path $p_2$, because all of its units are activated by $r_2$, the activation functions can be regarded as linear type; thus, $\Delta$ is in terms of $\lambda\sigma(\boldsymbol{w}^T\boldsymbol{x}+b)=\lambda(\boldsymbol{w}^T\boldsymbol{x}+b)$ and the parameter $\lambda$ is obtained by regarding the output of $\mathcal{U}$ as an one-dimensional input to the subnetwork of $p_2$ from the $\nu+1$th layer to the $\mu$th layer.
\end{proof}

\begin{figure}[!t]
\captionsetup{justification=centering}
\centering
\subfloat[Examples of theorem 1 and corollary 4.]{\includegraphics[width=2.5in, trim = {2.8cm 1.5cm 3.0cm 1.0cm}, clip]{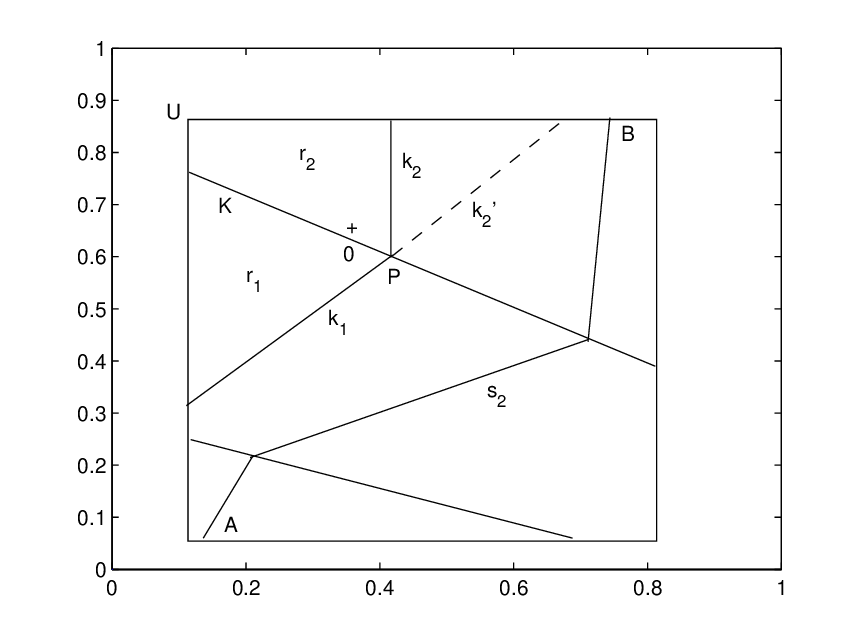}} \quad
\subfloat[Examples of theorem 2 and proposition 3.]{\includegraphics[width=2.5in, trim = {2.8cm 1.5cm 3.0cm 1.0cm}, clip]{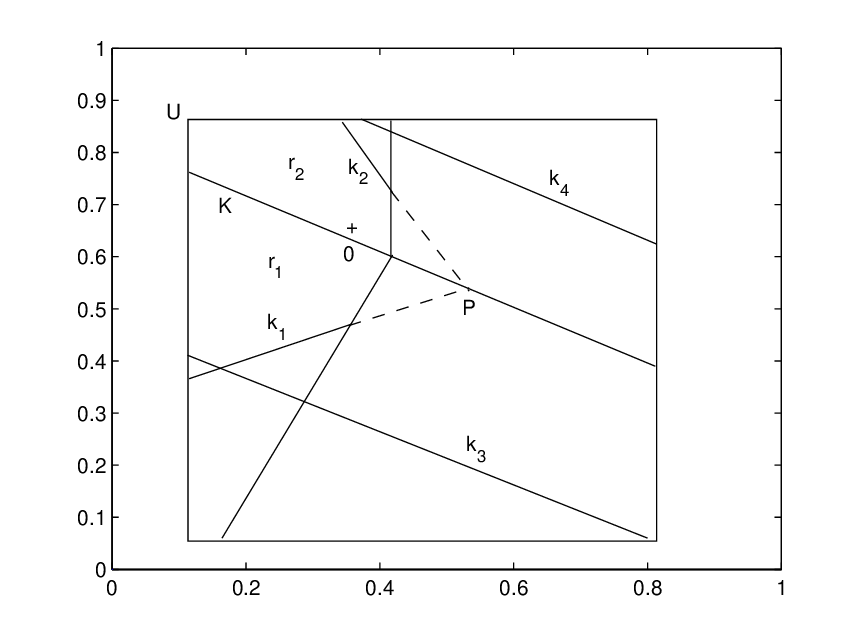}}
\caption{Knots from the same unit.}
\label{Fig.2}
\end{figure}

\begin{thm}[Knots from the same unit-\Romannum{1}]
Notations being from lemma 1, suppose that the input-dimensionality $n \ge 2$. Then if
\begin{equation}
k_1 \cap \mathcal{K} = \mathfrak{P} \ne \emptyset
\end{equation}
and $\lambda \ne 0$,
then
\begin{equation}
\mathcal{K} \cap k_1 \cap k_2 = \mathfrak{P},
\end{equation}
that is, $k_2$ passes through the intersection of $\mathcal{K}$ and $k_1$. And also, $k_2 \cap \mathcal{K} = \mathfrak{P}$ and $k_1 \cap k_2 = \mathfrak{P}$.
\end{thm}
\begin{proof}
Let $\boldsymbol{x}_0 \in \mathfrak{P}$ be a point of $\mathfrak{P}$ of equation 2.10. Under the notations of lemma 1, equation 2.10 means that
\begin{equation}
\begin{cases}
\boldsymbol{w}_1^T\boldsymbol{x}_0 + b_1 &= 0 \\
\boldsymbol{w}^T\boldsymbol{x}_0 + b &= 0
\end{cases}
\end{equation}
has a solution of $\boldsymbol{x}_0$ and $\boldsymbol{x}_0 \in \mathfrak{P}$. Equations 2.8 and 2.12 imply $\boldsymbol{w}_2^T\boldsymbol{x}_0 + b_2 = 0$, so $\mathfrak{P} \subset k_2$ and equation 2.11 follows.

Similarly, $\mathfrak{Q}=(k_2 \cap \mathcal{K}) \subset k_1$ is obtained by equation 2.8. Then to each $\boldsymbol{x} \in \mathfrak{Q}$, we have $\boldsymbol{x} \in \mathcal{K} \cap k_1$, implying $\mathfrak{Q} \subset \mathfrak{P}$. The previous result $\mathfrak{P} \subset k_2$ and the condition $\mathfrak{P} \in \mathcal{K}$ of equation 2.10 give $\mathfrak{P} \subset \mathfrak{Q}$. Thus, $\mathfrak{Q} = \mathfrak{P}$, that is, $k_2 \cap \mathcal{K} = \mathfrak{P}$. Equation 2.8 also yields $\mathfrak{R}=k_1 \cap k_2 \subset \mathcal{K}$; by equation 2.10 we have $\mathfrak{R} \subset \mathfrak{P}$; equations 2.10 and 2.8 also imply $\mathfrak{P} \subset \mathfrak{R}$; and hence $\mathfrak{R} = \mathfrak{P}$ or $k_1 \cap k_2=\mathfrak{P}$. Figure \ref{Fig.2}a shows an example of this theorem when the input is two-dimensional.
\end{proof}

The next two corollaries are the direct consequence of theorem 1 and the first corollary is a special case of the second one.

\begin{cl}[Piecewise linear curve of a unit]
When the input of network $\mathfrak{N}$ is two-dimensional, a knot is a line segment with a direction. Then theorem 1 means that introducing unit $\mathcal{U}$ in $p_1$ could change the direction of knot $k_1$ of $\mathscr{U}$. Recursive applications of theorem 1 could yield a continuous piecewise linear curve generated by the same unit $\mathscr{U}$---that is, a unit of $\mathfrak{N}$ can produce a continuous piecewise linear curve instead of a single line of the two-layer case.
\end{cl}
\noindent
\textbf{Example}. Figure \ref{Fig.2}a gives some examples of corollary 4. Knot $k_1$ changes its direction to become $k_2$ due to a new activated unit generating knot $K$ in a shallower layer; the dashed knot $k_2'$ is the one preserving the direction of $k_1$ for two-layer ReLU networks. When theorem 1 is repeatedly applied, a continuous piecewise linear curve such as $s_2$ from points $A$ to $B$ can be produced by a single unit.

\begin{cl}[Piecewise linear manifold of a unit]
To theorem 1, a knot $k_1$ could change its normal vector to become $k_2$ when intersecting the knot $\mathcal{K}$ produced by a unit $\mathcal{U}$ in a shallower layer; and if the similar operation is repeatedly done, a continuous piecewise linear manifold dividing $U$ could be formed by the single unit $\mathscr{U}$.
\end{cl}

\begin{rmk}
From corollaries 4 and 5, we can see an essential difference between a two-layer ReLU network $\mathcal{N}$ and deep network $\mathfrak{N}$. A unit of the hidden layer of $\mathcal{N}$ generates a hyperplane partitioning $U$, while a unit of the hidden layers of $\mathfrak{N}$ can yield a piecewise linear manifold instead. The latter is more flexible in forming a partition of $U$ that fits the geometric feature of data or functions.
\end{rmk}

\begin{thm}[Knots from the same unit-\Romannum{2}]
Under lemma 1, let $L$, $l_1$ and $l_2$ be the hyperplanes that the knots $K$, $k_1$ and $k_2$ lie on, respectively. If $l_1 \cap L = \mathfrak{P} \ne \emptyset$, we have $L \cap l_1 \cap l_2 = \mathfrak{P}$. Moreover, if $k_1 \cap \mathcal{K} = \emptyset$, then $k_2 \cap \mathcal{K} = \emptyset$.
\end{thm}
\begin{proof}
The first conclusion is analogous to theorem 1. If $k_1 \cap \mathcal{K} = \emptyset$, then $(l_1 \cap L=\mathfrak{P}) \cap \mathcal{K} = \emptyset$ and the second conclusion follows. In Figure \ref{Fig.2}b, knots $K$, $k_1$ and $k_2$ visually demonstrate this theorem.
\end{proof}

\begin{cl}[Multiple piecewise linear manifolds of a unit]
Notations from lemma 1, it is possible that the unit $\mathscr{U}$ generates two piecewise linear manifolds that have no common point.
\end{cl}
\begin{proof}
This corollary is by the second conclusion of theorem 2 as well as corollary 5. When the separated knots $k_1$ and $k_2$ of theorem 2 change their normal vectors independently by corollary 5, two piecewise linear manifolds can be formed.
\end{proof}

\begin{prp}[Knots from the same unit-\Romannum{3}]
Notations as in lemma 1, if $k_1$ is parallel to $\mathcal{K}$ or $k_1 \parallel \mathcal{K}$, then $k_2 \parallel \mathcal{K}$.
\end{prp}
\begin{proof}
In this case, equation 2.8 becomes $\boldsymbol{w}_2^T\boldsymbol{x} + b_2 = c\boldsymbol{w}^T\boldsymbol{x} + b_1 + \lambda(\boldsymbol{w}^T\boldsymbol{x} + b)$ with $c$ a constant, and the conclusion follows. In Figure \ref{Fig.2}b, the knots $k_3$ and $k_4$ with $k_3 \parallel \mathcal{K}$ and $k_4 \parallel \mathcal{K}$ give an example.
\end{proof}

\subsection{Principles of Knot Production}
The purpose of knot production is to control the partition of the input space to be a desired one, which is one of two main mechanisms of deep feedforward ReLU networks.

\begin{thm}[General principle of knot production]
Under some path $\mathcal{P}$ of network $\mathfrak{N}$, the knot production obeys the following two rules : (1) A unit $\mathcal{U}$ of the $\nu$th layer of $\mathcal{P}$ generates a knot when it is partially activated by a region $\mathscr{R}$ of the $\nu-1$th layer; (2) When $\mathscr{R}$ completely activates $\mathcal{U}$, even $\mathcal{U}$ leads to a hyperplane $\mathcal{L}$ dividing $U$, it cannot yield a knot through $\mathcal{L}$.
\end{thm}
\begin{proof}
The first conclusion is obvious. To the second one, the hyperplane $\mathcal{L}$ only temporarily exists under path $\mathscr{P}$ and would disappear when the path is changed, with no opportunity to divide $\mathscr{R}$.
\end{proof}

\begin{rmk}
This theorem tells us that not all of the units of a path can generate a knot.
\end{rmk}

To produce a certain knot, the knowledge of function implementation to be discussed in later sections is required and the associated principles of knot production will be introduced then. We now only give their short descriptions, with the details and proofs temporarily omitted.
\begin{itemize}
\item[(1)] Corollary 9 of section 2.5 generates a knot via the output weights of new added units in shallower layers.
\item[(2)] Corollary 10 of section 3.2 uses the input parameters of a unit to realize a knot.
\item[(3)] Theorem 14 of section 3.3 shows that a path can simultaneously implement a desired linear function and a required knot. Theorem 15 is for the case of multiple linear functions and one knot.
\item[(4)] Theorem 28 of section 5.2 incorporates all the principles of deep feedforward ReLU networks to generate multiple knots, especially the continuity-restriction principle of theorem 21. Theorem 29 further combines the mechanism of multiple outputs, thereby enabling multiple units to simultaneously implement more than one knot.
\end{itemize}

\subsection{Positive and Zero Parts of a Manifold}
In a two-layer ReLU network \citep*{Huang2024}, to a unit $\mathcal{U}$ and its corresponding $n-1$-dimensional hyperplane $\mathcal{L}$, we denoted by $\mathcal{L}^+$ the part of $U$ that can activate $\mathcal{U}$ and by $\mathcal{L}^0$ the other part, with $\mathcal{L}^+$ and $\mathcal{L}^0$ having clear geometric meanings. In the case of deep feedforward ReLU networks, a unit could yield a piecewise linear manifold $\mathcal{M}$ composed of more than one knot and we generalize the notations as $\mathcal{L}^+$ and $\mathcal{L}^0$ to $\mathcal{M}$.

\begin{dfn}[Activation of knots and manifolds]
Let $\mathscr{K}$ be a knot generated by a unit $\mathcal{U}$ of network $\mathfrak{N}$. Suppose that a region $\mathcal{R}$ corresponds to path $\mathscr{P}$ of $\mathfrak{N}$. We say $\mathcal{R} \subset \mathscr{K}^+$ or $\mathcal{R}$ activates $\mathscr{K}$, if $\mathscr{P}$ includes the unit $\mathcal{U}$; otherwise $\mathcal{R} \subset \mathscr{K}^0$. A point $\boldsymbol{x} \in \mathscr{K}^+$, if there exists a region $\mathcal{R}$ such that $\boldsymbol{x} \in \mathcal{R} \subset \mathscr{K}^+$; otherwise, $\boldsymbol{x} \in \mathscr{K}^0$. Suppose that $\mathcal{M}$ is a piecewise linear manifold produced by $\mathcal{U}$, $\mathcal{R} \subset \mathcal{M}^+$ if $\mathcal{R}$ activates a knot of $\mathcal{M}$, and $\mathcal{R} \subset \mathcal{M}^0$ if $\mathcal{R}$ cannot activate any knot of $\mathcal{M}$. When $\mathcal{R} \subset \mathcal{M}^+$ (or $\mathcal{M}^0$), we say that $\mathcal{R}$ activates (or inactivates) $\mathcal{M}$.
\end{dfn}

\begin{rmk}
Note that this definition can include the type of a two-layer ReLU network as a special case.
\end{rmk}

We want to know the geometric meanings of $\mathcal{M}^+$ and $\mathcal{M}^0$ in definition 7 as the two-layer case.
\begin{lem}
Let $r_1$ and $r_2$ be the regions of two adjacent paths $p_1$ and $p_2$ of network $\mathfrak{N}$, respectively. Suppose that $p_1$ and $p_2$ are different in unit $\mathcal{U}$ with $\mathcal{U} \in p_2$ but $\mathcal{U} \notin p_1$. Let $\mathscr{K}$ be a knot of $\mathfrak{N}$, generated by some unit $\mathscr{U}$ that is not the previous $\mathcal{U}$. If $r_1 \subset \mathscr{K}^+$ (or $\mathscr{K}^0$), then $r_2 \subset \mathscr{H}^+$ (or $\mathscr{K}^0$).
\end{lem}
\begin{proof}
Since $r_1 \subset \mathscr{K}^+$, we have $\mathscr{U} \in p_1$. Because $p_2$ has only one unit $\mathcal{U}$ differen from $p_1$, $\mathscr{U} \in p_2$ follows; by definition 7, $r_2 \subset \mathscr{H}^+$. The case of $\mathscr{K}^0$ is similar.
\end{proof}

\begin{thm}[Influence of a knot]
Denote by $r_i$'s for $i=1,2,\dots,\zeta$ some regions of $U$ partitioned by network $\mathfrak{N}$, and by $\mathscr{H}$ a knot produced by some unit $\mathcal{U}$ of $\mathfrak{N}$. Suppose that: (1) the mutual adjacent relationship
\begin{equation}
r_1 \frown r_2 \frown \dots \frown r_{\zeta}
\end{equation}
holds; (2) each knot $k_{\nu}=r_{\nu} \cap r_{\nu+1}$ for $\nu=1,2,\dots,\zeta-1$ is not generated by $\mathcal{U}$. Then if $r_1 \subset \mathscr{K}^+$ (or $\mathscr{K}^0$), we have $r_{\nu+1} \subset \mathscr{K}^+$ (or $\mathscr{K}^0$) for all $\nu$'s.
\end{thm}
\begin{proof}
This is a repeated application of lemma 2.
\end{proof}

\begin{cl}[Influence of a manifold]
Notations from definition 7 and theorem 4, to the regions $r_i$'s, if $r_1 \subset \mathcal{M}^+$ (or $\mathcal{M}^0$), then $r_{\nu+1} \subset \mathcal{M}^+$ (or $\mathcal{M}^0$) for all $\nu$'s.
\end{cl}
\begin{proof}
The proof is by theorem 4 and definition 7.
\end{proof}

\begin{dfn}[Positive and zero parts of a manifold]
Notations being from definition 7, to the input space $U$ and a partition $\mathcal{P}$ derived from network $\mathfrak{N}$, the set of the regions that can activate (or inactive) $\mathcal{M}$ is called the positive (or zero) part of $\mathcal{M}$, denoted by $\mathcal{M}^+$ (or $\mathcal{M}^0$).
\end{dfn}

\subsection{Principles of Region Production}
\begin{thm}[Influence of a new activated unit]
Under the notations of lemma 2, suppose that the depth of the layer of $\mathcal{U}$ satisfies $\nu \ne \Phi$ (i.e., $\mathcal{U}$ is not in the last hidden layer of $p_2$) and that the output of $\mathcal{U}$ is $\sigma(\boldsymbol{w}^T\boldsymbol{x} + b)$. Then the influence of the new activated $\mathcal{U}$ on the $j$th unit $u_{ij}$ of the $i$th layer of $p_1$, with $i \ge \nu+1$ (namely a layer deeper than that of $\mathcal{U}$), can be expressed as
\begin{equation}
\Delta_{ij} = \boldsymbol{\alpha}_{ij}^T\boldsymbol{v}\cdot(\boldsymbol{w}^T\boldsymbol{x} + b),
\end{equation}
which is a modification of the input of $u_{ij}$ in terms of
\begin{equation}
s_{ij}'(\boldsymbol{x}) = \Delta_{ij} + s_{ij}(\boldsymbol{x}) = \Delta_{ij} + \boldsymbol{w}_{ij}^T\boldsymbol{x}^{(i-1)} + b_{ij},
\end{equation}
with $s_{ij}(\boldsymbol{x})=\boldsymbol{w}_{ij}^T\boldsymbol{x}^{(i-1)} + b_{ij}$ the original input of $u_{ij}$, where
\begin{equation}
\boldsymbol{\alpha}_{ij} = W_{\nu+2}W_{\nu+3}\dots \boldsymbol{w}_{ij},
\end{equation}
in which $\boldsymbol{w}_{ij}$ the $j$th column of $W_{i}$ or the input-weight vector of $u_{ij}$, and where $\boldsymbol{v}$ is the output-weight vector of $\mathcal{U}$ with size $m_{\nu+1} \times 1$.
\end{thm}
\begin{proof}
The proof is similar to that of lemma 1.
\end{proof}

The corollary below is a recursive application of theorem 5.
\begin{cl}[Influence of multiple new activated units]
Based on the notations of theorem 5, suppose that $p_2$ is derived from $p_1$ by newly activating $\theta$ units with $\theta \ge 2$, which are denote by $u_{k}$'s for $1\le k \le \theta$. Then the influence of $u_{k}$'s on a unit $u_{ij}$ of $p_1$ in a layer deeper than those of all $u_{k}$'s can be expressed as
\begin{equation}
s_{ij}'(\boldsymbol{x}) = \sum_{\nu=1}^{\theta}\Delta_{ij}^{(k)} + \boldsymbol{w}_{ij}^T\boldsymbol{x}^{(i-1)} + b_{ij},
\end{equation}
where
\begin{equation}
\Delta_{ij}^{(k)} = \boldsymbol{\alpha}_{ij}^{(k)T}\boldsymbol{v}_{k}\cdot(\boldsymbol{w}_{k}^T\boldsymbol{x} + b_{k}),
\end{equation}
in which $\boldsymbol{w}_{k}^T\boldsymbol{x} + b_{k} = 0$ is the equation of the knot generated by $u_{k}$.
\end{cl}

\begin{cl}[Principles of knot production-\Romannum{1}]
Notations from corollary 7, suppose that $\theta \ge n+1$, that the rank of matrix
\begin{equation}
W =
\begin{pmat}({})
\boldsymbol{w}_1 & \boldsymbol{w}_2 & \cdots & \boldsymbol{w}_{\theta} \cr
b_1 & b_2 & \cdots & b_{\theta} \cr
\end{pmat}
\end{equation}
is $n+1$, that $\boldsymbol{\alpha}_{ij}^{(k)} \ne \boldsymbol{0}$, and that all the units of $p_1$ can be activated by the regions of the paths after introducing $u_{k}$'s. Then to unit $u_{ij}$, arbitrary knot can be formed by adjusting the output weights of $u_{k}$'s.
\end{cl}
\begin{proof}
Since the rank of $W$ is $n+1$, $\sum_{k}\beta_{k}(\boldsymbol{w}_{k}^T\boldsymbol{x} + b_{k})$ can yield any linear function (see \citet*{Huang2024}'s lemma 3) by adjusting $\beta_k$'s, such that the term $\boldsymbol{w}_{ij}^T\boldsymbol{x}^{(i-1)} + b_{ij}$ in equation 2.17 can be compensated, contributing to arbitrary knot via $u_{ij}$. Then we set
\begin{equation}
\boldsymbol{\alpha}_{ij}^{(k)T}\boldsymbol{v}_{k} = \beta_{k},
\end{equation}
for which a solution of $\boldsymbol{v}_{k}$ can be easily found.
\end{proof}

\begin{thm}[Properties of a new region-\Romannum{1}]
Notations being from theorem 5, suppose that $|\boldsymbol{\alpha}_{ij}^T\boldsymbol{v}|$ is bounded and that $\|\boldsymbol{w}\|_2=1$ in equation 2.14. If the volume $\mathscr{V}$ of $r_2$ is sufficiently small due to the parameter setting of the knot of $\mathcal{U}$, then $p_2$ is adjacent to $p_1$---that is, $r_2$ can activate all the units of $p_1$ in the layers deeper than the $\nu$th one.
\end{thm}
\begin{proof}
Denote by $\mathcal{K}$ the knot generated by $\mathcal{U}$ with equation $\boldsymbol{w}^T\boldsymbol{x} + b=0$. Then by theorem 5, the influence of $\mathcal{U}$ on each unit $u_{ij}$ for $i > \nu$ can be expressed as $\Delta_{ij}=\boldsymbol{\alpha}_{ij}^T\boldsymbol{v}(\boldsymbol{w}^T\boldsymbol{x} + b)=\lambda(\boldsymbol{w}^T\boldsymbol{x} + b)$. Region $r_2$ of $p_2$ is adjacent to $\mathcal{K}$, or $\mathcal{K}$ belongs to the boundary of $r_2$. Write $d_{m} = \max_{\boldsymbol{x}\in r_2}d(\boldsymbol{x},\mathcal{K})$, where $d(\boldsymbol{x},\mathcal{K})=|\boldsymbol{w}^T\boldsymbol{x} + b|/\|\boldsymbol{w}\|_2=|\boldsymbol{w}^T\boldsymbol{x} + b|$ is the distance from point $\boldsymbol{x} \in r_2$ to $\mathcal{K}$. As $d_m$ tends to zero, $\mathscr{V}$ could be sufficiently small; simultaneously, since $|\boldsymbol{\alpha}_{ij}^T\boldsymbol{v}|$ is bounded, $|\Delta_{ij}|$ can also be small enough such that $\Delta_{ij}$ cannot affect the activation of $u_{ij}$. The solution of this theorem can be expressed as
\begin{equation}
[\boldsymbol{w}^T, b]=\arg\bigcap_{i=\nu+1}^{\Phi}\bigcap_{j=1}^{m_i}\big(|\Delta_{ij}(\boldsymbol{w}, b)| < \varepsilon\big),
\end{equation}
where $\varepsilon$ is a threshold that can enable arbitrary $\boldsymbol{x} \in r_2$ to activate $u_{ij}$ for all $i$ and $j$.
\end{proof}

\begin{thm}[Properties of a new region-\Romannum{2}]
Notations as in theorem 5, the new formed region $r_2$ of path $p_2$ adjacent to $p_1$ satisfies
\begin{equation}
r_2 = r^{(\nu)} \cap \mathcal{K}^+ \cap \bigcap_{i=\nu+1}^{\Phi} \bigcap_{j=1}^{m_i}r_{ij}',
\end{equation}
where $r^{(\nu)}$ is the region of the $\nu$th layer of $p_1$, $\mathcal{K}$ is the knot formed by $\mathcal{U}$, and
\begin{equation}
r_{ij}'=\{\boldsymbol{x}: s_{ij}'(\boldsymbol{x}) > 0, \ \nu+1\le i\le \Phi, \ 1 \le j \le m_i\}
\end{equation}
with $s_{ij}'(\boldsymbol{x})$ from equation 2.17. Write $R_{\nu}=\{r_{ij}': \nu+1 \le i \le \Phi, 1 \le j \le m_i\}$. Then the cardinality of $R_{\nu}$ is equal to $|R_{\nu}| = \sum_{i=\nu+1}^{\Phi}m_i$; and if $\mu < \tau$ for $1\le \mu, \tau \le \Phi-1$, we have $|R_{\mu}| > |R_{\tau}|$, which means that if $\mathcal{U}$ is introduced in a shallower layer, more units would be involved in forming $r_2$.
\end{thm}
\begin{proof}
To path $p_1$, the corresponding region $r_1$ can be regarded as the intersection of the regions activating the units $u_{ij}$ of $p_1$---that is, $r_1 = \bigcap_{i=1}^{\Phi} \bigcap_{j=1}^{m_i}r_{ij}$, where $r_{ij}=\{\boldsymbol{x}: s_{ij}(\boldsymbol{x}) > 0\}$ with $s_{ij}(\boldsymbol{x})=\boldsymbol{w}_{ij}^T\boldsymbol{x}^{(i-1)} + b_{ij}$; and equation 2.22 is a modification of this formula after  introducing the new unit $\mathcal{U}$.
\end{proof}

\begin{rmk-1}
This theorem is related to the volume of region $r_2$. Since the term $\bigcap_{i}\bigcap_{j}r_{ij}'$ of the right side of equation 2.22 is a set-intersection operation, if any one of $r_{ij}'$'s is small, $r_2$ would be restricted by it; and this possibility may increase when $\mathcal{U}$ is in a shallower layer, because more regions would be involved in this operation.
\end{rmk-1}

\begin{rmk-1}
In the sense of theorem 6, remark 1 and corollary 2, region $r_2$ formed by $p_2$ tends to be smaller, no matter $\mathcal{U}$ is placed in a shallower or deeper layer. Corollary 2 demonstrates that the deeper the layer of $\mathcal{U}$ is, the smaller $r_2$ may become, due to region subdivisions via the succeeding layers; theorem 6 and remark 1 indicate that a shallower layer of $\mathcal{U}$ may contribute to smaller $r_2$. In either case, the effect is the same.
\end{rmk-1}

\section{Function Implementation}
The relationship between function construction and region dividing in network $\mathfrak{N}$ is not as clear as that of a two-layer ReLU work, in the sense that some parameters are shared by the two purposes. The concept of a path plays a central role in solving this problem. Section 3.1 proves the continuous property of a piecewise linear function $g(\boldsymbol{x})$ output by $\mathfrak{N}$. Section 3.2 constructs the first linear function of $g(\boldsymbol{x})$ through an initial path. Section 3.3 realizes the linear functions on adjacent regions by adjacent paths. Section 3.4 investigates a coefficient vector related to the solution existence of both functions and knots. Section 3.5 proposes the continuity-restriction principle for the linear functions that cannot be directly constructed. Section 3.6 studies a local property of solutions that is useful in parameter setting. Section 3.7 implements a desired spline over a single strict partial order of knots.

\subsection{Function Space}

\begin{dfn}[Function space] Write
\begin{equation}
\mathfrak{C}(R):=\{s: s(\boldsymbol{x})=s_i(\boldsymbol{x})\ \text{for} \ x \in r_i, s_i \smile \mathscr{N}_i \ \text{if} \ \mathscr{N}_i \ne \emptyset , \ i = 1, 2, \dots, \zeta\},
\end{equation}
where $R=\{r_1, r_2, \dots, r_{\zeta}\}$ is the set of regions of $U$ partitioned by network $\mathfrak{N}$, $\mathscr{N}_i$ is the set of the linear functions on the regions that are adjacent to $r_i$, and $s_i \smile \mathscr{N}_i$ means that $s_i$ is continuous with each element of $\mathscr{N}_i$,
\end{dfn}

\begin{lem}
The function $g(\boldsymbol{x})$ output by network $\mathfrak{N}$ is continuous at the knots derived from local units.
\end{lem}
\begin{proof}
Each local unit $\mathcal{U}$ yields a knot $\mathcal{K}$ of $U$. By the definition of the continuity of a function with respect to $\mathcal{K}$ (\citet*{Huang2024}'s definition 6) as well as the output property of a ReLU, $g(\boldsymbol{x})$ is continuous at $\mathcal{K}$. Since $\mathcal{U}$ is arbitrarily selected, the conclusion follows.
\end{proof}

If we say a hyperplane (or knot) $\mathscr{L}$ is adjacent of a region $\mathcal{R}$ of $U$ derived from network $\mathfrak{N}$, it means that $\dim(\mathscr{L} \cap \mathcal{R})=n-1$.
\begin{thm}[Property of adjacent paths]
Let $p_1$ and $p_2$ be two adjacent paths of network $\mathfrak{N}$ with a unit $\mathcal{U} \in p_2$ but $\mathcal{U} \notin p_1$, whose regions are $r_1$ and $r_2$, respectively. Suppose that a linear function $s_1(\boldsymbol{x})$ on $r_1$ has been realized by $p_1$. Then we have: the corresponding hyperplane $\mathcal{L}$ of $\mathcal{U}$ is adjacent to $r_1$; $r_2$ is adjacent to $r_1$ (or $r_2 \frown r_1$); a piecewise linear function $\mathcal{S}(\boldsymbol{x}) = s_{i}(x)$ on $r_{i}$ for $i=1, 2$ can be realized by $p_1$ and $p_2$, which is continuous at the knot $\mathcal{K} = \mathcal{L}\cap r_1$ and satisfies
\begin{equation}
s_2(\boldsymbol{x}) = s_1(\boldsymbol{x}) + \lambda\sigma(\boldsymbol{w}^T\boldsymbol{x} + b)
\end{equation}
for $\boldsymbol{x} \in r_2$, where $\boldsymbol{w}^T\boldsymbol{x} + b = 0$ is the equation of $\mathcal{L}$ and
\begin{equation}
\lambda = (W_{\nu+2}W_{\nu+3}\dots \boldsymbol{w}_{\Phi+1})^T\boldsymbol{v},
\end{equation}
similar to equation 2.9 of lemma 1, where $\boldsymbol{v}$ is the output-weight vector of $\mathcal{U}$.
\end{thm}
\begin{proof}
After the $\nu$th layer, $p_2$ and $p_1$ subdivide different regions, despite the units used being the same. The subdivision stops at the output layer with $r_2$ and $r_1$ formed. By the condition of this theorem, $r_1 \subset \mathcal{L}^0$, $r_2 \subset \mathcal{L}^+$ and we can find a point $p \in \mathcal{L}$ such that: let $C_p$ be an $n$-sphere whose centre is $p$ and radius $d$; then to arbitrary small $d$, $C_p \cap \mathcal{L}^+ \subset r_2$ and $C_p \cap \mathcal{L}^0 \subset r_1$. This implies that there's no region between $r_1$ and $r_2$ separating them; thus, $r_1$ is adjacent to $r_2$ and both $r_1$ and $r_2$ are adjacent to $\mathcal{L}$. Since $p_2$ differs from $p_1$ only in unit $\mathcal{U}$, the output function of $p_2$ must be in the form of equation 3.2 and the parameter $\lambda$ of equation 3.3 can be obtained by the method of equation 2.9. Finally, by lemma 3, $\mathcal{S}(\boldsymbol{x})$ is continuous at $\mathcal{K}$.
\end{proof}

The converse of theorem 8 is also true, which is the following theorem.
\begin{thm}[Property of adjacent regions]
Let $R$ be the set of the regions of $U$ derived from network $\mathfrak{N}$. If two regions $r_1$ and $r_2$ of $R$ are adjacent and separated by knot $\mathcal{K}$, then regardless of the case of different units producing the same knot, their corresponding paths are adjacent and only different in a unit $\mathcal{U}$ that generates $\mathcal{K}$.
\end{thm}
\begin{proof}
Because $r_1$ and $r_2$ are both adjacent to $\mathcal{K}$ and separated by it, their corresponding paths $p_1$ and $p_2$, respectively, would
be different in a unit $\mathcal{U}$ producing $\mathcal{K}$. Without loss of generality, suppose that $\mathcal{U} \in p_2$ but $\mathcal{U}
\notin p_1$. If besides $\mathcal{U}$, there exists another unit $\mathscr{U}$ of $p_2$ inactivated by $p_1$, it would yield an additional knot $\mathscr{K}$ that separates $r_1$ and $r_2$, and this is impossible.
\end{proof}

\begin{thm}[Continuity property]
The function $g(\boldsymbol{x})$ output by network $\mathfrak{N}$ satisfies
\begin{equation}
g(\boldsymbol{x}) \in \mathfrak{C}(R)
\end{equation}
of equation 3.1, which means that $g(\boldsymbol{x})$ is a continuous piecewise linear function on $R$.
\end{thm}
\begin{proof}
The proof is by theorems 9 and 8. Since the regions of $R$ are separated by knots and adjacent to their neighbors, by theorem 9, the corresponding paths are adjacent. Then according to theorem 8, the function $g(\boldsymbol{x})$ produced by adjacent paths should be continuous.
\end{proof}

\subsection{Initial-Path Function}
\begin{dfn}[Initial path]
Let $g(\boldsymbol{x})$ be a piecewise linear function output by network $\mathfrak{N}$. The initial path $p_0$ of $\mathfrak{N}$ is the one producing the first linear function (arbitrarily selected) of $g(\boldsymbol{x})$, based on which other linear functions can be formed.
\end{dfn}

\begin{rmk}
Regardless of two-sided solutions to be investigated in section 4, among all the paths of $\mathfrak{N}$, $p_0$ could be the one with the smallest number of units (see the examples of sections 6 and 7).
\end{rmk}

\begin{thm}[Region transfer through layers]
Let $\mathcal{N}_2=m_{\nu} \cdot m_{\nu+1}$ be a two-layer subnetwork of $\mathfrak{N}$. Suppose that a region $r_{\nu}$ of $U$ has been transmitted to the output of the $\nu$th layer (or the first layer of $\mathcal{N}_2$) in terms of
\begin{equation}
\boldsymbol{x}^{(\nu)} =
\begin{pmat}({})
\boldsymbol{x}' \cr
W^{(\nu)}\boldsymbol{x} + \boldsymbol{b}^{(\nu)} \cr
\end{pmat}
\end{equation}
where $\boldsymbol{x}'$ is an affine transformation of $\boldsymbol{x} \in r_{\nu}$ or
\begin{equation}
\boldsymbol{x}' = W\boldsymbol{x} + \boldsymbol{b},
\end{equation}
with $W$ nonsingular. Suppose that a subregion $r_{\nu+1} \subseteq r_{\nu}$ activates $\theta$ units of the $\nu+1$th layer of $\mathfrak{N}$ with $n\le \theta \le m_{\nu+1}$. Then $r_{\nu+1}$ could be transmitted to the output of the $\nu+1$th layer (or the second layer of $\mathcal{N}_2$) through
\begin{equation}
\boldsymbol{x}^{(\nu+1)} =
\begin{pmat}({})
\boldsymbol{x}'' \cr
W^{(\nu+1)}\boldsymbol{x} + \boldsymbol{b}^{(\nu+1)} \cr
\end{pmat},
\end{equation}
where $\boldsymbol{x}''$ is an affine transformation of $\boldsymbol{x} \in r_{\nu+1}$.
\end{thm}
\begin{proof}
Let $u_i$ be one of the units of the $\nu+1$th layer activated by $r_{\nu+1}$ and
\begin{equation}
\boldsymbol{w}_i^T\boldsymbol{x}^{(\nu)} + b_i =0
\end{equation}
is its corresponding $m_{\nu}-1$-dimensional hyperplane, which can be regarded as an $n-1$-dimensional hyperplane $\mathscr{L}$ because of equation 3.5. First, we want to realize an arbitrary $\mathscr{L}$ with equation $\boldsymbol{w}^T\boldsymbol{x}+b=0$ through setting the parameters $\boldsymbol{w}_i$ and $b_i$ of equation 3.8.

Let $\boldsymbol{v}$ be an $n \times 1$ vector whose entries belong to $\boldsymbol{w}_i$, corresponding to the units (whose set is denoted by $S$) outputting $\boldsymbol{x}'$ of equation 3.5. Only using the units of $S$, equations 3.8 and 3.6 give a hyperplane of the input space with equation
\begin{equation}
\boldsymbol{v}^T\boldsymbol{x}' + b_i = (W^T\boldsymbol{v})^T\boldsymbol{x} + b_i + \boldsymbol{v}^T\boldsymbol{b} = 0,
\end{equation}
in which $W^T\boldsymbol{v}$ can produce an arbitrary vector $\boldsymbol{w}$ through setting $\boldsymbol{v}$ due to the nonsingular property of $W$, and $b_i + \boldsymbol{v}^T\boldsymbol{b}$ can be an arbitrary value via adjusting $b_i$. Denote by $\boldsymbol{v}'$ the $(m_{\nu}-n) \times 1$ vector that includes the entries of $\boldsymbol{w_i}$ except for those of $\boldsymbol{v}$; then $\boldsymbol{w}_i = [\boldsymbol{v}^T, \boldsymbol{v}'^T]^T$. To yield a certain $\mathscr{L}$ with equation $\boldsymbol{w}^T\boldsymbol{x}+b=0$ via equation 3.8, $W^T\boldsymbol{v}$ of equation 3.9 can compensate the influence of $\boldsymbol{v}'$ for the production of $\boldsymbol{w}$, while $b$ is obtained by adjusting $b_i$ of equation 3.9; that is, the two equations
\begin{equation}
W^T\boldsymbol{v} + W^{(\nu)T}\boldsymbol{v}' = \boldsymbol{w}
\end{equation}
and
\begin{equation}
b_i + \boldsymbol{v}^T\boldsymbol{b} + \boldsymbol{v}'^T\boldsymbol{b}^{(\nu)} = b,
\end{equation}
has a solution of $\boldsymbol{w}_i = [\boldsymbol{v}^T, \boldsymbol{v}'^T]^T$ and $b_i$---in fact, $\boldsymbol{v}'$ of equation 3.10 can be arbitrarily set, after which a solution of $\boldsymbol{v}$ can be obtained due to the nonsingular $W$; then adjust $b_i$ to realize $b$ in equation 3.11.

We first construct $n$ hyperplanes $l_1, l_2, \dots, l_n$ in the input space by \citet*{Huang2020}'s theorem 4, satisfying $r_{\nu+1} \subset \bigcap_{i=1}^nl_i^+$. Then realize $l_i$'s through the units $u_i$'s of the $\nu+1$th layer by the above method; the outputs of $u_i$'s comprise the entries of $\boldsymbol{x}''$, an affine transformation of $\boldsymbol{x}$ due to the construction method of $l_i$'s. The parameters of the remaining $m_{\nu+1}-n$ units can be arbitrarily set with a constraint that they are activated by any $\boldsymbol{x} \in r_{\nu+1}$.
\end{proof}

\begin{cl}[Principles of knot production-\Romannum{2}]
Suppose that the $\nu$th layer of network $\mathfrak{N}$ satisfies the condition of equation 3.5. Then to an arbitrary unit $u_i$ of the $\nu+1$th layer of $\mathfrak{N}$, any knot with equation $\boldsymbol{w}^T\boldsymbol{x}+b=0$ can be realized by setting the input parameters of $u_i$.
\end{cl}
\begin{proof}
The construction method is in the proof of theorem 11.
\end{proof}

\begin{rmk}
Compared with corollary 8, in this corollary a unit directly uses its input parameters to produce a knot rather than resorts to the output weights of other units.
\end{rmk}

\begin{thm}[Function construction via initial path]
Under initial path $p_0$ of network $\mathfrak{N}$ with $n$-dimensional input, suppose that: (1) the number of the units of the last hidden layer (or the $\Phi$th layer) satisfies $m_{\Phi} \ge n+1$; (2) each of the remaining hidden layers has at least $n$ units (i.e., $m_i \ge n$ for $i=1,2,\dots, \Phi-1$); (3) the output of each hidden layer can be represented in the form of equation 3.5. Then any linear function on the region $r_0$ of $p_0$ can be realized by $p_0$.
\end{thm}
\begin{proof}
By the condition of this theorem, the output of the $\Phi-1$th layer (the one previous the last hidden one) can be written as
\begin{equation}
\boldsymbol{x}^{(\Phi-1)} =
\begin{pmat}({})
\boldsymbol{x}' \cr
W^{(\Phi-1)}\boldsymbol{x} + \boldsymbol{b}^{(\Phi-1)} \cr
\end{pmat}
\end{equation}
where
\begin{equation}
\boldsymbol{x}' = A\boldsymbol{x} + \boldsymbol{b} = T(\boldsymbol{x})
\end{equation}
is an affine transformation of $\boldsymbol{x} \in r_{\Phi-1}$, where $r_{\Phi-1}$ is the region of the $\Phi-1$th layer of $p_0$, which should activate at least $n+1$ units of the $\Phi$th layer to form an arbitrary linear function on $r_0$. When only considering $\boldsymbol{x}'$ of equation 3.12, the associated parameters of the activated units of the $\Phi$th layer should form a matrix (called ``linear-output matrix'' in \citet*{Huang2024})
\begin{equation}
W =
\begin{pmat}({})
\boldsymbol{w}_1 & \boldsymbol{w}_2 & \cdots & \boldsymbol{w}_{m_{\Phi}} \cr
b_1 & b_2 & \cdots & b_{m_{\Phi}} \cr
\end{pmat}
\end{equation}
whose rank is $n+1$; and the $i$th volume of $W$ corresponds to the $n-1$-dimensional hyperplane $\boldsymbol{w}_i^T\boldsymbol{x}+b_i=0$ for $i=1, 2, \dots, m_{\Phi}$ of the $i$th unit $u_i$ of the $\Phi$th layer, denoted by $l_i$. Let $r_0'$ be the affine transformation $T$ of $r_0$ introduced in equation 3.13. We can use the method of \citet*{Huang2024}'s lemma 4 to construct $l_i$'s, such that $r_0' \subset \bigcap_{i=1}^{m_{\Phi}}l_i^+$ and the rank of matrix $W$ is $n+1$.

The input parameters $\boldsymbol{w}_i'$ and $b_i'$ of $u_i$ are derived from $\boldsymbol{w}_i$ and $b_i$ via the affine transformation $T$ of equation 3.13, that is,
\begin{equation}
\boldsymbol{w}_i'^T\boldsymbol{x}'+ b_i'=\boldsymbol{w}_i^T\boldsymbol{x}+b_i;
\end{equation}
thus,
\begin{equation}
\boldsymbol{w}_i' = {A^{-1T}}\boldsymbol{w}_i \ \text{and} \ b_i'=b_i-\boldsymbol{w}_i'^T\boldsymbol{b}.
\end{equation}
Through this parameter setting, the output of $u_i$ is $\sigma(\boldsymbol{w}_i^T\boldsymbol{x}+b_i)$ and the output weights of $u_i$'s can be set based on matrix $W$ of equation 3.14 to produce an arbitrary linear function (\citet*{Huang2024}'s lemma 3).

Next, we turn to $W^{(\Phi-1)}\boldsymbol{x} + \boldsymbol{b}^{(\Phi-1)}$ of equation 3.12 and its influence on the output linear function can be compensated by setting the parameters of the output layer similarly to the proof of theorem 11, due to the fact that arbitrary linear function can be implemented.
\end{proof}

\begin{rmk}
The conditions of this theorem is only necessary for generating an arbitrary linear function rather than a certain one; so even they are not satisfied, some linear functions can still be implemented.
\end{rmk}

\begin{cl}[Construction of initial-path functions]
Given a region $r_0$ of $U$, a path $p_0$ of $\mathfrak{N}$ can be constructed to realize an arbitrary linear function on it, with the last hidden layer of $p_0$ having at least $n+1$ units and each of the other hidden layers having at least $n$ units.
\end{cl}
\begin{proof}
We here only present one solution: region $r_0$ is formed in the first layer and the succeeding layers only transmit it to the last hidden layer without further dividing it. To construct the first layer, select one of the boundary of $r_0$ as the hyperplane of a unit; other $n-1$ hyperplanes are constructed in the input space via \citet*{Huang2024}'s lemma 4, whose parameters form other $n-1$ units of the first layer. Those $n$ units generates $\boldsymbol{x}'$ of $\boldsymbol{x}^{(1)}$ in equation 3.5; the remaining units can be arbitrarily set as long as they are activated by $r_0$. Once $\boldsymbol{x}^{(1)}$ has been established, the succeeding $\boldsymbol{x}^{(\nu)}$'s can be recursively constructed by theorem 11 and the desired linear function is implemented in the output layer by theorem 12.
\end{proof}

\subsection{Function Construction via Adjacent Paths}

\begin{thm}[Function construction via adjacent paths]
Denote by $p_1$ and $p_2$ two adjacent paths of network $\mathfrak{N}$ with a unit $\mathcal{U} \in p_2$ and $\mathcal{U} \notin p_1$, whose corresponding regions are $r_1$ and $r_2$, respectively. Suppose that $\mathcal{U}$ is in the $\nu$th layer of $p_2$ and its output-weight vector is $\boldsymbol{v}=[v_1, v_2, \dots, v_{m_{\nu+1}}]^T$. By theorem 8, the linear functions $s_1(\boldsymbol{x})$ on $r_1$ and $s_2(\boldsymbol{x})$ on $r_2$ satisfies $s_2(\boldsymbol{x})=s_1(\boldsymbol{x})+\lambda\sigma(\boldsymbol{w}^T\boldsymbol{x}+b)$, where $\boldsymbol{w}^T\boldsymbol{x}+b=0$ is the equation of knot $\mathcal{K}=r_1 \cap r_2$ and
\begin{equation}
\lambda = \boldsymbol{\alpha}^T\boldsymbol{v}
\end{equation}
with
\begin{equation}
\boldsymbol{\alpha} = W_{\nu+2}W_{\nu+3}\dots \boldsymbol{w}_{\Phi+1},
\end{equation}
where $W_1, W_2,\dots, W_{\Phi}$ and $\boldsymbol{w}_{\Phi+1}$ are the input-weight matrices of the layers of path $p_1$. Then given a $s_1(\boldsymbol{x})$, under the condition that $p_2$ is adjacent to $p_1$ and $\boldsymbol{\alpha} \ne \boldsymbol{0}$, any $s_2(\boldsymbol{x})$ continuous with $s_1(\boldsymbol{x})$ can be implemented through adjusting $\boldsymbol{v}$.
\end{thm}
\begin{proof}
The goal is to set a desired $\lambda$ by adjusting $\boldsymbol{v}$ of equation 3.17 and a solution exists whenever the condition of this theorem is satisfied.
\end{proof}

\begin{thm}[Principles of knot production-\Romannum{3}]
Under theorem 13, suppose that unit $\mathscr{U}$ of path $p_1$ is in a layer deeper than the layer that contains $\mathcal{U}$ of $p_2$, generating knot $\mathscr{K}$, and that $\mathscr{K}$ intersects $\mathcal{K}$ or $\mathscr{K} \cap \mathcal{K} \ne \emptyset$. Let $\boldsymbol{w}_1^T\boldsymbol{x}+b_1=0$ be the equation of $\mathscr{K}$ and $\boldsymbol{w}_1'^T\boldsymbol{x}+b_1'=0$ be its modified version $\mathscr{K}'$ in path $p_2$ due to the introduction of $\mathcal{U}$. Suppose that $\mathscr{K}'$ can be represented in this form
\begin{equation}
\boldsymbol{w}_1'^T\boldsymbol{x}+b_1'=\boldsymbol{w}_1^T\boldsymbol{x}+b_1 + \beta(\boldsymbol{w}^T\boldsymbol{x}+b),
\end{equation}
where $\boldsymbol{w}^T\boldsymbol{x}+b=0$ is the equation of knot $\mathcal{K}$. Then by adjusting the output-weight vector $\boldsymbol{v}$ of $\mathcal{U}$, $\mathscr{K}'$ and $s_2(\boldsymbol{x})$ can be simultaneously realized by $p_2$, provided that
\begin{equation}
\begin{aligned}
\begin{cases}
\boldsymbol{\alpha}^T\boldsymbol{v} &= \lambda \\
\boldsymbol{\alpha}_{u}^T\boldsymbol{v} &= \beta\
\end{cases}
\end{aligned}
\end{equation}
has a solution of $\boldsymbol{v}$, where $\boldsymbol{\alpha}_{u}^T\boldsymbol{v} = \beta$ is from equation 2.20 of corollary 9.
\end{thm}
\begin{proof}
The proof is by theorem 13 and corollary 9.
\end{proof}

\begin{dfn}[Influence coefficient vector]
In theorem 14, $\boldsymbol{\alpha}$ is called the \textsl{influence coefficient vector} of path $p_2$ for linear function $s_2(\boldsymbol{x})$, and $\boldsymbol{\alpha}_u$ called the \textsl{influence coefficient vector} of path $p_2$ for unit $\mathscr{U}$.
\end{dfn}

\begin{thm}[Principles of knot production-\Romannum{4}]
Given a path $p_0$ of network $\mathfrak{N}$, suppose that paths $p_{i}$'s for $i=1, 2, \dots, \theta$ with $\theta \ge n+1$ are obtained by adding units one by one on the basis of $p_0$, with $p_{i}$ adjacent to $p_{i-1}$ or $u_i \in p_{i}$ but $u_i \notin p_{i-1}$, where $u_i$ is a unit. Let $r_j$ for $j=0,1,\dots,\theta$ be the region of $p_j$ and $s_j(\boldsymbol{x})$ a linear function on $r_j$ implemented by $p_j$; the linear functions satisfy $s_i(\boldsymbol{x}) = s_{i-1}(\boldsymbol{x})+\lambda_{i}\sigma(\boldsymbol{w}_i^T\boldsymbol{x}+b_i)$. Denote by $\mathscr{U}$ a unit of $p_0$ in a layer deeper than the layers of all $u_i$'s. Let $\boldsymbol{\alpha}_i$ be the influence coefficient vector of $p_i$ for function $s_{i}(\boldsymbol{x})$ derived from equation 3.17 and $\boldsymbol{\alpha}_{iu}$ be the influence coefficient vector of $p_i$ for $\mathscr{U}$ analogously to $\boldsymbol{\alpha}_{u}$ of equation 3.20. Denote by column vector $\boldsymbol{v}_i$ the output-weight vector of $u_i$.

Under the constraint that $p_i$ is adjacent to $p_{i-1}$ as well as the condition of corollary 8, all $s_i(\boldsymbol{x})$'s and an arbitrary knot can be simultaneously realized by unit $\mathscr{U}$, provided that each
\begin{equation}
\begin{cases}
\begin{aligned}
{\boldsymbol{\alpha}_i}^T\boldsymbol{v}_i &= \lambda_{i} \\
\ \boldsymbol{\alpha}_{iu}^T\boldsymbol{v}_{i} &= \beta_{i}
\end{aligned}
\end{cases}
\end{equation}
has a solution of $\boldsymbol{v}_i$, where $\boldsymbol{\alpha}_{iu}^T\boldsymbol{v}_{i} = \beta_{i}$ is from equation 2.20 of corollary 9.
\end{thm}
\begin{proof}
The conclusion is also by theorem 13 and corollary 9.
\end{proof}

\begin{thm}[Construction of adjacent paths and desired functions]
Let $p_1$ be a path of network $\mathfrak{N}$. By theorem 8, if another path $p_2$ is adjacent to $p_1$, a linear function $s_2(\boldsymbol{x})$ on $r_2$ of $p_2$ that is continuous with the one $s_1(\boldsymbol{x})$ on $r_1$ of $p_1$ satisfies $s_2(\boldsymbol{x})=s_1(\boldsymbol{x})+\lambda\sigma(\boldsymbol{w}^T\boldsymbol{x}+b)$. Given $p_1$ and $s_1(\boldsymbol{x})$, if $|\lambda|$ is sufficiently small and the influence coefficient vector $\boldsymbol{\alpha}\ne \boldsymbol{0}$, a path $p_2$ adjacent to $p_1$ as well as an arbitrary $s_2(\boldsymbol{x})$ continuous with $s_1(\boldsymbol{x})$ can be simultaneously constructed.
\end{thm}
\begin{proof}
We write equation 3.17 as
\begin{equation}
\alpha_1v_1 + \alpha_2v_2 + \dots + \alpha_kv_{k} = \lambda
\end{equation}
with $k=m_{\nu+1}$, where $\alpha_i$'s and $v_i$'s for $i=1,2,\dots,k$ are the entries of $\boldsymbol{\alpha}$ and $\boldsymbol{v}$, respectively. By equation 3.18, $\alpha_i$'s are constant determined by the parameters of $p_1$ after the $\nu$ the layer.

To make $p_2$ adjacent to $p_1$, the introduction of $\mathcal{U}$ should not influence the activation of the units of $p_1$ after the $\nu$ the layer. By equations 2.14 and 2.15, the influence of $\mathcal{U}$ on a unit $u_{ij}$ of $p_1$ in a layer deeper than the $\nu$th one is $\Delta_{ij} = \lambda(\boldsymbol{w}^T\boldsymbol{x} + b)$; then a solution is that $|\lambda|$ is sufficiently small, such that the disturbance $\Delta_{ij}$ cannot affect the activation of $u_{ij}$. This is a solution of path $p_2$.

To generate a desired $s_2(\boldsymbol{x})$, $\lambda$ should be freely adjusted. We should find a solution of both $\lambda$ for $s_2(\boldsymbol{x})$ and $\boldsymbol{v}$ for $p_2$ simultaneously. Equation 3.22 can be written as
\begin{equation}
v_k = \frac{\lambda - (\alpha_1v_1 + \alpha_2v_2 + \dots + \alpha_{k-1}v_{k-1})}{\alpha_k}
\end{equation}
if $\alpha_k \ne 0$ (which can be selected), through which we see that if $|\lambda|$ and $|v_{j}|$'s for $j=1,2,\dots,k-1$ are small enough, $|v_k|$ can also be arbitrarily small. Thus, within a sufficiently small range, $\lambda$ can be freely adjusted to yield a desired $s_2(\boldsymbol{x})$, without influencing the existence of path $p_2$.
\end{proof}

\begin{dfn}[Connected regions]
Let $R =\{r_1, r_2, \dots, r_{\zeta}\}$ be a partition of $U$ via network $\mathfrak{N}$. Two regions $r_i$ and $r_j$ of $R$ are said to be connected, if they can be linked by a series of adjacent regions in terms of
\begin{equation}
r_i \frown r_{\nu_1} \frown r_{\nu_2} \frown \dots \frown r_j,
\end{equation}
in which all the regions belong to $R$.
\end{dfn}

\begin{thm}[Recursive application of adjacent paths]
To a piecewise linear function $g(\boldsymbol{x})$ output by network $\mathfrak{N}$, each of its linear functions can be obtained by the principle of theorem 13 on the basis of an initial path.
\end{thm}
\begin{proof}
Using the notations of definition 12, each pair of adjacent regions corresponds to two adjacent paths, for which the principle of theorem 13 can be applied to generate the associated linear functions of $g(\boldsymbol{x})$. Since arbitrary two regions of $R$ are connected, each linear function of $g(\boldsymbol{x})$ can be directly or indirectly obtained by theorem 13 on the basis of an initial path.
\end{proof}

\begin{prp}[Function construction via the last hidden layer]
Under the notations of theorem 13, if the new unit $\mathcal{U}$ of path $p_2$ is added in the last hidden layer of $p_1$, the production of $s_2(\boldsymbol{x})$ is by the principle of two-layer ReLU networks; that is, the single output weight of $\mathcal{U}$ can uniquely determine a $s_2(\boldsymbol{x})$ continuous with $s_1(\boldsymbol{x})$.
\end{prp}
\begin{proof}
The conclusion is obvious by the principle of two-layer ReLU networks \citep*{Huang2024}.
\end{proof}

\begin{rmk}
This proposition differs from theorem 13 in that it doesn't need to fulfil the adjacent-path condition, for which a solution is easier to be constructed.
\end{rmk}

\subsection{Properties of Influence Coefficient Vectors}
By theorems from 13 to 15 we know that the influence coefficient vector $\boldsymbol{\alpha} \ne \boldsymbol{0}$ or $\boldsymbol{\alpha}_u \ne \boldsymbol{0}$ is a necessary condition for the solution existence of both knots and functions, for which we should exclusively study the properties of influence coefficient vectors. The results are optimistic, in the sense that it is not ``easy'' to encounter $\boldsymbol{\alpha}=\boldsymbol{0}$ or $\boldsymbol{\alpha}_u=\boldsymbol{0}$; especially when manually set, a nonzero influence coefficient vector is not difficult to be constructed.

\begin{thm}[Properties of influence coefficient vectors-\Romannum{1}]
Notations from theorem 13, if not all the values of $s_1(\boldsymbol{x})$ for $\boldsymbol{x} \in r_1$ are zero or $s_1(\boldsymbol{x}) \equiv 0$ doesn't holds, and if
\begin{equation}
\boldsymbol{x}^{(\nu+1)}=(W_{1}W_{2}\dots W_{\nu+1})^T\boldsymbol{x},
\end{equation}
is not a zero vector, then $\boldsymbol{\alpha} \ne \boldsymbol{0}$.
\end{thm}
\begin{proof}
The linear function $s_1(\boldsymbol{x})$ can be expressed as $s_1(\boldsymbol{x})=(W_{1}W_{2}\dots \boldsymbol{w}_{\Phi+1})^T\boldsymbol{x}$; this expression can also be written as
\begin{equation}
s_1(\boldsymbol{x})=\boldsymbol{\alpha}^T\boldsymbol{x}^{(\nu+1)},
\end{equation}
through which we can see that this theorem holds.
\end{proof}

\begin{cl}[Nonzero $\boldsymbol{\alpha}$ related to region transfer]
Under the notations of theorem 18, if $\boldsymbol{x}^{(\nu+1)}$ can be represented as equation 3.5 of theorem 11 and if $s_1(\boldsymbol{x}) \equiv 0$ doesn't holds, then $\boldsymbol{\alpha} \ne \boldsymbol{0}$.
\end{cl}
\begin{proof}
In this case, by the meaning of equation 3.5, $\boldsymbol{x}^{(\nu+1)} \ne \boldsymbol{0}$ and the conclusion follows from theorem 18.
\end{proof}

\begin{thm}[Properties of influence coefficient vectors-\Romannum{2}]
In theorem 5 of section 2.5, if the output $\boldsymbol{x}^{(\nu+1)}$ of the $\nu+1$th layer of $p_1$ as equation 3.25 is not a zero vector, then the influence coefficient vector $\boldsymbol{\alpha}_{ij}$, introduced for the influence of a new activated unit $\mathcal{U}$ on a unit $u_{ij}$ in a layer deeper than that of $\mathcal{U}$, satisfies $\boldsymbol{\alpha}_{ij} \ne \boldsymbol{0}$.
\end{thm}
\begin{proof}
Since $u_{ij}$ of path $p_1$ is activated, its output function $s_{ij}'(\boldsymbol{x})$ (see equation 2.15) is not always zero; and by theorem 18 this conclusion holds.
\end{proof}

\begin{thm}[Properties of influence coefficient vectors-\Romannum{3}]
Under the notations of theorem 18, suppose that $s_1(\boldsymbol{x})=0$ for all $\boldsymbol{x} \in r_1$ (or $s_1(\boldsymbol{x}) \equiv 0$) and that $\boldsymbol{x}^{(\nu+1)}$ of equation 3.25 can be expressed in the form of equation 3.5, that is,
\begin{equation}
\boldsymbol{x}^{(\nu+1)} =
\begin{pmat}({})
\boldsymbol{x}' \cr
W^{(\nu+1)}\boldsymbol{x} + \boldsymbol{b}^{(\nu+1)} \cr
\end{pmat},
\end{equation}
with $\boldsymbol{x}'$ an affine transformation of $\boldsymbol{x} \in U$. Let $\sigma(\boldsymbol{w}_i^T\boldsymbol{x}+b_i)$ for $i=1,2,\dots,m_{\nu+1}$ be the output of the $i$th unit of the $\nu+1$th layer of path $p_1$, corresponding to the entries of $\boldsymbol{x}^{(\nu+1)}$; and the parameters $\boldsymbol{w}_i$'s and $b_i$'s comprise a linear-output matrix $W$ as equation 2.19. If the length $m_{\nu+1}$ of vector $\boldsymbol{x}^{(\nu+1)}$ satisfies $m_{\nu+1}>n+1$ and the rank of $W$ is $n+1$, then the influence coefficient vector $\boldsymbol{\alpha}$ has a solution of nonzero vector for $s_1(\boldsymbol{x}) \equiv 0$; otherwise if $m_{\nu} \le n+1$ and the rank of $W$ is $m_{\nu+1}$, we have $\boldsymbol{\alpha}=\boldsymbol{0}$ for $s_1(\boldsymbol{x}) \equiv 0$.
\end{thm}
\begin{proof}
By equation 3.26, the condition $s_1(\boldsymbol{x}) \equiv 0$ can be written as
\begin{equation}
\boldsymbol{\alpha}^T\boldsymbol{x}^{(\nu+1)} = 0.
\end{equation}
Equation 3.28 can be reduced to a system of linear equations through the expression $\sigma(\boldsymbol{w}_i^T\boldsymbol{x}+b_i)$ of each entry of $\boldsymbol{x}^{(\nu+1)}$, with the entries of $\boldsymbol{\alpha}$ as variables (\citet*{Huang2024}'s lemma 3). To the first conclusion, when the condition is satisfied, equation 3.28 has infinitely many solutions of $\boldsymbol{\alpha}$ including a nonzero-vector one. To the second conclusion, there exists a unique solution $\boldsymbol{\alpha}=\boldsymbol{0}$ under that condition.
\end{proof}

\subsection{Principle of Continuity Restriction}
\begin{thm}[Principle of continuity restriction]
Let $\mathcal{R}$ be a region of $U=[0,1]^n$ for $n \ge 2$ obtained by network $\mathfrak{N}$. Denote by $k_1$ and $k_2$ two knots derived from $\mathcal{R}$. Suppose that $l_1 \cap l_2 \ne \emptyset$, where $l_1$ and $l_2$ are the $n-1$-dimensional hyperplanes that $k_1$ and $k_2$ lie on, respectively. If the linear functions on $k_1$ and $k_2$ have been realized by $\mathfrak{N}$ through function construction for the adjacent regions of $\mathcal{R}$, then the one on $\mathcal{R}$ is also implemented.
\end{thm}
\begin{proof}
To this problem, there's no substantial difference between deep and two-layer ReLU networks (\citet*{Huang2024}'s theorem 9). A key fact is that a piecewise linear function output by $\mathfrak{N}$ is continuous, which was proved in theorem 10.
\end{proof}

\begin{cl}[Function construction from neighborhoods]
Notations being from theorem 21, let $r_1$ and $r_2$ be two regions both adjacent to $\mathcal{R}$. Suppose that the hyperplanes of $r_1 \cap \mathcal{R}$ and $r_2 \cap \mathcal{R}$ are not parallel. Then if the linear functions on $r_1$ and $r_2$ have been implemented by $\mathfrak{N}$, the one on $\mathcal{R}$ is also simultaneously realized.
\end{cl}
\begin{proof}
This corollary is an alternative description of theorem 21, more convenient for applications.
\end{proof}

\subsection{Multiple Expressions of Linear Functions}
\begin{thm}[Multiple expressions of linear functions]
Suppose that a region $\mathcal{R}$ formed by a path $\mathcal{P}$ of network $\mathfrak{N}$ has $\phi$ adjacent regions, denoted by $r_i$'s for $i=1,2,\dots,\phi$. Each $r_i$ corresponds to a knot $k_i$ separating $r_i$ and $\mathcal{R}$ (or $k_i=r_i \cap \mathcal{R}$); and $k_i$ is generated by unit $u_i$, corresponding to a path $p_i$ of $\mathfrak{N}$ that is adjacent to $\mathcal{P}$; the region of $p_i$ is $r_i$. Let $g(\boldsymbol{x})$ be the piecewise linear function output by $\mathfrak{N}$ and $s_i(\boldsymbol{x})$ the linear function of $g(\boldsymbol{x})$ on $r_i$. Then the linear function $s(\boldsymbol{x})$ of $g(\boldsymbol{x})$ on $\mathcal{R}$ has $\phi$ independent expressions, each of which is either
\begin{equation}
s(\boldsymbol{x}) = s_i(\boldsymbol{x}) + \lambda_i\sigma(\boldsymbol{w}_i^T\boldsymbol{x} + b_i)
\end{equation}
when $\mathcal{R} \subset k_i^+$ and $r_i \subset k_i^0$
or
\begin{equation}
s(\boldsymbol{x}) = s_i(\boldsymbol{x}) - \lambda_i\sigma(\boldsymbol{w}_i^T\boldsymbol{x} + b_i)
\end{equation}
if $\mathcal{R} \subset k_i^0$ and $r_i \subset k_i^+$, where $\boldsymbol{w}_i^T\boldsymbol{x} + b_i = 0$ is the equation of $k_i$ and $\lambda_i$ is a constant as in equation 3.2.
\end{thm}
\begin{proof}
To equation 3.29, the influence of $p_i$ on $s(\boldsymbol{x})$ is embedded in expression $s_i(\boldsymbol{x})$; and $\lambda_i\sigma(\boldsymbol{w}_i^T\boldsymbol{x} + b_i)$ is the effect of new activated unit $u_i$ generating $k_i$. To each of the remaining knots of $\mathfrak{N}$, say, $\mathscr{K}$, either $(\mathcal{R} \cup r_i) \subset \mathscr{K}^+$ or $(\mathcal{R} \cup r_i) \subset \mathscr{K}^0$ holds, because $\mathcal{P}$ and $p_i$ differ from each other only at $u_i$, similarly to lemma 2. If $r_i \subset \mathscr{K}^+$, the influence of $\mathscr{K}^+$ is embedded in $s_i(\boldsymbol{x})$ of equation 3.29 and thus $s(\boldsymbol{x})$ also contains this influence. Equation 3.30 can be written as $s_i(\boldsymbol{x}) = s(\boldsymbol{x}) + \lambda_i\sigma(\boldsymbol{w}_i^T\boldsymbol{x} + b_i)$, analogous to equation 3.29, so the proof is similar.
\end{proof}

\begin{rmk}
This theorem is useful in determining the solution of $\mathfrak{N}$ through local information and will be used to interpret the training solution in later section 8.
\end{rmk}

\subsection{Splines over Single Strict Partial Order}

\begin{dfn}[Strict partial order of knots]
Let
\begin{equation}
K = \{k_1, k_2, \dots, k_{\zeta}\}
\end{equation}
be a set of knots generated by network $\mathfrak{N}$, with each $k_i$ for $i=1,2,\dots, \zeta$ produced by unit $u_i$ of $\mathfrak{N}$ in path $p_i$. Suppose that there exists a set
\begin{equation}
\begin{aligned}
R = \{r_{1}, r_{\nu}: r_{1} \subset k_{1}^+, &\dim(r_1\cap k_1)=n-1, \bigcup_{j = 1}^{\nu-1}r_{j} \subseteq k_{\nu}^0, r_{\nu} \subset \bigcap_{\mu = 1}^{\nu}k_{\mu}^+, \\&k_{\nu} = r_{\nu} \cap r_{\nu-1}, \nu=2,3, \dots, \zeta\},
\end{aligned}
\end{equation}
of regions of $U=[0,1]^n$ divided by $\mathfrak{N}$. If $1\le \nu_1 < \nu_2 \le \zeta$, we define that
$k_{\nu_1} \prec k_{\nu_2}$. Write
\begin{equation}
\mathscr{O}(K) := k_1 \prec k_2 \dots \prec k_{\zeta}.
\end{equation}
The elements of $R$ are called the ordered regions of $\mathscr{O}(K)$. The region $r_0=U \cap k_1^0$ is called the \textsl{initial region} of $\mathscr{O}(K)$.
\end{dfn}

\begin{rmk}
The definition is borrowed from \citet*{Huang2024}'s definitions 3 and 15. The relation ``$\prec$'' is obviously a strict partial order (see the proof of \citet*{Huang2024}'s proposition 2).
\end{rmk}

\begin{dfn}[Continuous linear spline]
We use the notation
\begin{equation}
\mathfrak{S}_n(R, K):= \{s(\boldsymbol{x}): s(\boldsymbol{x})=s_i(\boldsymbol{x}) \  \text{for} \ \boldsymbol{x}\in r_i\}
\end{equation}
to denote the set of continuous piecewise linear functions defined on the regions of $R$ with the knots of $K$ introduced in definition 13, where the linear functions $s_i(\boldsymbol{x})$'s of $s(\boldsymbol{x})$ are subject to
\begin{equation}
s_{\nu}(\boldsymbol{x}) = s_{\nu-1}(\boldsymbol{x}) + \lambda_{\nu}\sigma(\boldsymbol{w}_{\nu}^T\boldsymbol{x} + b_{\nu})
\end{equation}
for $\nu=2,3,\dots,\zeta$, where $\sigma(\boldsymbol{w}_{\nu}^T\boldsymbol{x} + b_{\nu})$ is the output of unit $u_i$; each element of $\mathfrak{S}_n(R, K)$ is also called a \textsl{continuous linear spline} or shortly a \textsl{spline} in this paper. The linear function on the initial region $r_0$ is called the \textsl{initial linear function} of $s(\boldsymbol{x})$, denoted by $s_0(\boldsymbol{x})$.
\end{dfn}

\begin{thm}[Splines over a single strict partial order]
Any continuous linear spline $s(\boldsymbol{x}) \in \mathfrak{S}_n(R, K)$ of equation 3.34 can be realized by network $\mathfrak{N}$, provided that: (a) the conditions of theorem 12 for constructing the function over an initial path are satisfied; (b) the conditions of theorem 13 for adjacent paths and influence coefficient vectors are satisfied by each pair of paths $p_{\nu}$ and $p_{\nu-1}$ for $2\le \nu \le \zeta$, where $p_{i}$ is the path of region $r_{i}$ for $i=1,2,\dots, \zeta$.
\end{thm}
\begin{proof}
The initial linear function $s_0(\boldsymbol{x})$ can be realized by an initial path as in corollary 11. To each $s_{\nu}(\boldsymbol{x})$, by theorem 13, if the adjacent-path condition is satisfied by $p_{\nu}$ and $p_{\nu-1}$ and influence coefficient vector $\boldsymbol{\alpha}_{\nu} \ne \boldsymbol{0}$, it can be constructed.
\end{proof}

\section{Two-Sided Solutions}
As the case of two-layer ReLU networks, one-sided solutions of network $\mathfrak{N}$ are already enough for function construction and the introduction of two-sided ones is for enlarging the solution space to explain experiments. A two-sided solution of $\mathfrak{N}$ is much more complex than that of a two-layer ReLU network (\citet*{Huang2024}'s section 5) due to the disturbance of multiple hidden layers. We reduce this complexity to three typical cases, which are the contents of the first three subsections, and a general two-sided solution can be explained by their generalization or combination.

\subsection{First Case of Single Negative Unit}

The following assumption is introduced for simplifying the description and is tacitly assumed to be true in section 4. In other sections, we will explicitly give it when required.
\begin{assm}[Adjacent-path assumption]
Notations being from theorem 23 and definitions 13 and 14, support that to each $\nu$, the adjacent relationship between paths $p_{\nu}$ and $p_{\nu-1}$ is satisfied, that the influence coefficient vector $\boldsymbol{\alpha}_{\nu} \ne \boldsymbol{0}$, and that the knots of the ordered regions $r_i$'s besides $k_i$'s are preserved. These conditions are collectively called ``\textsl{adjacent-path assumption}''.
\end{assm}

\begin{rmk}
To fulfil this assumption, for instance, theorem 6 or 16 can ensure adjacent paths and the methods of section 2.3 for knot production can preserve the knots of regions; technical details can be found in the examples of section 8.3. The nonzero influence coefficient vector was studied in section 3.4.
\end{rmk}

\begin{dfn}[Negative form of a knot (unit)]
Give a knot $\mathcal{K}$ of network $\mathfrak{N}$ with equation $\boldsymbol{w}^T\boldsymbol{x}+b=0$, its negative form $-\mathcal{K}$ means that its equation is modified to be $-\boldsymbol{w}^T\boldsymbol{x}-b=0$. Suppose that unit $\mathcal{U}$ of $\mathfrak{N}$ generates $\mathcal{K}$ and then the negative form $-\mathcal{U}$ of $\mathcal{U}$ corresponds to $-\mathcal{K}$ with output $\sigma(-\boldsymbol{w}^T\boldsymbol{x}-b)$.
\end{dfn}

\begin{lem}
Under a path $\mathcal{P}$ of network $\mathfrak{N}$, suppose that knot $\mathcal{K}$ is produced by the $j$th unit $u_{ij}$ of the $i$th layer, whose equation is $\boldsymbol{w}^T\boldsymbol{x}+b=0$ or $\boldsymbol{w}_{ij}^T\boldsymbol{x}^{(i-1)}+b_{ij}=0$, where $\boldsymbol{x}^{(i-1)}$ is the input of the $i$th layer of $\mathcal{P}$ and $\boldsymbol{w}_{ij}$ and $b_{ij}$ are the input parameters of $u_{ij}$. Then $-\mathcal{K}$ with equation $-\boldsymbol{w}^T\boldsymbol{x}-b=0$ can be obtained by changing $\boldsymbol{w}_{ij}$ and $b_{ij}$ into $-\boldsymbol{w}_{ij}$ and $-b_{ij}$, respectively, if the parameters of the units of $\mathcal{P}$ in the layers shallower than the $i$th one remain invariant.
\end{lem}
\begin{proof}
We know that $\boldsymbol{w}^T\boldsymbol{x}+b=0$ is derived from $\boldsymbol{w}_{ij}^T\boldsymbol{x}^{(i-1)}+b_{ij}=0$ by expanding $\boldsymbol{x}^{(i-1)}$ as a linear combination of the dimensions of the input space, or we can write $\boldsymbol{w}^T\boldsymbol{x}+b=0 = \boldsymbol{w}_{ij}^T\boldsymbol{x}^{(i-1)}+b_{ij}=0$. If the parameters of the shallower layers are fixed, $\boldsymbol{x}^{(i-1)}$ would be preserved; then if changing $\boldsymbol{w}_{ij}$ and $b_{ij}$ into $-\boldsymbol{w}_{ij}$ and $-b_{ij}$, respectively,  we have  $-\boldsymbol{w}^T\boldsymbol{x}-b=0 = (-\boldsymbol{w}_{ij})^T\boldsymbol{x}^{(i-1)}-b_{ij}=0$.
\end{proof}

\begin{prp}[Influence of negative knots]
Given a strict partial order $\mathscr{O}(K)$ of equation 3.33, suppose that the $i$th knot $k_i$ generated by unit $u_i$ is changed into its negative form $-k_i$, with the modified unit denoted by $-u_i$. Then the unit $-u_i$ must be in the paths $p_j$'s for $j < i$, with $p_j$ corresponding to knot $k_j$.
\end{prp}
\begin{proof}
The reason is by theorem 4. After this negative-knot operation, we have $r_{i-1} \subset -k_i^+$. Since $r_{i-2} \cap r_{i-1} = k_{i-1} \ne k_i$, $r_{i-2} \subset -k_i^+$ also holds. This process can be repeated until $r_1 \subset -k_i^+$.
\end{proof}

\begin{thm}[Principles of two-sided solutions-\Romannum{1}]
In theorem 23, let $u_i$ be the unit generating knot $k_i$ and $u_j$ for $j=1,2 \dots,i-1$ be the one that yields $k_j$ with $k_j \prec k_i$. Suppose that $u_j$ is in the $d_{j}$th layer of network $\mathfrak{N}$, and that $u_i$ is in the $\nu$th layer, satisfying $\nu \le d_j$ for all $j$. Then, if $k_i$ generated by $u_i$ is changed into its negative form $-k_i$, a solution of arbitrary $s(\boldsymbol{x}) \in \mathfrak{S}_n(R, K')$ via network $\mathfrak{N}$ can still be founded, where $K'$ is the modified version of $K$ by substituting $k_i$ with $-k_i$ (with this notation applicable to other cases of this section).
\end{thm}
\begin{proof}
After the negative-form operation for $k_i$ or $u_i$, the original $s_{i}(\boldsymbol{x}) = s_{i-1}(\boldsymbol{x})+\lambda_{i}\sigma(\boldsymbol{w}_{i}^T\boldsymbol{x}+b_{i})$ for $s_{i}(\boldsymbol{x})$ should become
\begin{equation}
s_{i}(\boldsymbol{x}) = s_{i-1}(\boldsymbol{x}) - \lambda_{i}\sigma(-\boldsymbol{w}_{i}^T\boldsymbol{x} - b_{i}),
\end{equation}
where
\begin{equation}
\lambda_i = \boldsymbol{\alpha}_i^T\boldsymbol{v}_i
\end{equation}
is substituted by
\begin{equation}
-\lambda_i = {\boldsymbol{\alpha}_i'}^T\boldsymbol{v}_i
\end{equation}
for which $\boldsymbol{v}_i$ should be reset. Write
\begin{equation}
\boldsymbol{\alpha}_i = W_{\nu+2}W_{\nu+3}\dots \boldsymbol{w}_{\Phi+1}
\end{equation}
by equation 3.18.

We first reset the parameters of $u_{j}$'s for $j=1,2 \dots,i-1$, because according to proposition 5 each path $p_j$ corresponding to $u_j$ would contain $-u_i$ after the operation. Suppose that $p_0$ is the initial path in theorem 23 and implements $s_0(\boldsymbol{x})$ via theorem 12. Let the output of $u_i$ be
\begin{equation}
\sigma(\boldsymbol{w}_{\nu i}^T\boldsymbol{x}^{(\nu-1)} + b_{\nu i}) = \sigma(\boldsymbol{w}_{i}^T\boldsymbol{x} + b_{i}),
\end{equation}
where $\boldsymbol{w}_{\nu i}$ and $b_{\nu i}$ are the input parameters of $u_i$ in the $\nu$th layer.
After $-u_i$ being constructed by lemma 4, $-k_i$ is formed and the $\nu$th layer of $p_0$ also adds a unit $-u_{i}$ (similarly to lemma 4) whose output is
\begin{equation}
\sigma(-\boldsymbol{w}_{\nu i}^T\boldsymbol{x}^{(\nu-1)} - b_{\nu i});
\end{equation}
the altered path is denoted by $p_0'$.

Suppose that the output-weight vector $\boldsymbol{v}_i$ of $-u_i$ has been already set and we first fix it as a constant vector in the following discussion. A key point is that the influence of $-u_i$ on $p_0$ can be restrict in the $\nu$th layer, such that the parameters and the outputs of the succeeding layers can remain invariant, resulting in the original linear function $s_0(\boldsymbol{x})$ via $p_0'$. The method is by corollary 10 and the proof of theorem 11. Let $\mathcal{U}$ be a unit of the $\nu+1$ layer of $p_0$ whose output is $\sigma(\boldsymbol{w}^T\boldsymbol{x}+b)$; write $s(\boldsymbol{x})=\boldsymbol{w}^T\boldsymbol{x}+b$. By theorem 12, arbitrary $s(\boldsymbol{x})$ can be realized by the input parameters of $\mathcal{U}$. Thus, the influence of $-u_i$ on $\mathcal{U}$ can be compensated by $s(\boldsymbol{x})$, such that the output of $\mathcal{U}$ can remain the same as that of the original path $p_0$. Because $\mathcal{U}$ is arbitrarily selected in the $\nu+1$th layer, the outputs of all the units of the $\nu+1$th layer of $p_0'$ could be equal to those of $p_0$, respectively; consequently, the parameters of $p_0$ after the $\nu+1$th layer can be preserved to be those of $p_0'$, without influencing the original result of $p_0$---producing $s_0(\boldsymbol{x})$.

Path $p_1$ is derived from $p_0$ by adding unit $u_1$ in a layer whose depth $d_1 \ge \nu$ and there are three cases. (1) $d_1>\nu+1$: After $p_0$ becoming $p_0'$, we should construct $p_1'$ on the basis of $p_0'$. Since the outputs of the layers of $p'_0$ whose depths are greater than $\nu$ are equal to those of $p_0$, respectively, $u_1$ could be introduced in $p_0'$ in the original position of $p_0$, whose parameters are also the same as those in $p_0$. (2) $d_1=\nu+1$: The input parameters of $u_1$ are reset to compensate the influence of $-u_i$ as discussed above and its output parameters are unchanged. (3) $d_1=\nu$: all the parameters of $u_1$ are invariant since $-u_i$ doesn't affect its input and the impact on its output has been compensated by the previous operations.

The remaining paths $p_2, p_3,\dots,p_{i-1}$ can be recursively dealt with similarly to $p_1$, yielding new paths $p_2', p_3',\dots,p_{i-1}'$.

Now, we turn to the process from $p_{i-1}'$ to $p_i'$ that removes $-u_{i}$ from $p_{i-1}'$, resulting in function $s_{i}(\boldsymbol{x})$, for which we should reset the output-weight vector $\boldsymbol{v}_i$ of $-u_{i}$. By the above parameter settings of the new paths as well as equation 4.4, we know that
\begin{equation}
\boldsymbol{\alpha}'_i = \boldsymbol{\alpha}_i;
\end{equation}
equations 4.7, 4.2 and 4.3 imply
\begin{equation}
\boldsymbol{v}'_i = -\boldsymbol{v}_i.
\end{equation}
Thus, we can initially set the output-weight vector of $-u_i$ to be $-\boldsymbol{v}_i$ and the previous assumption about the fixed output-weight vector is resolved.

We then process $u_{k}$ for $k=i+1, i+2, \dots, \zeta$ of path $p_k$. In the new path $p_{k}'$, each $u_k$ can be placed in the original position of $p_{k-1}$, whose parameter setting cannot influence the preceding accomplished results. The input parameters of $u_k$ should be set to generate knot $k_j$ by corollary 10; the output parameters of $u_k$ yield the linear function $s_k(\boldsymbol{x})$ on $r_k$ via theorem 13; under adjacent-path assumption and the condition of theorem 23, a solution can be founded. This completes the construction of this type of two-sided solution. We summarize the above process by the following steps.
\begin{itemize}
\item[(1)] Let $\boldsymbol{w}_{\nu i}$ and $b_{\nu i}$ be the input parameters of $u_{i}$ and $\boldsymbol{v}_i$ its output-weight vector. Change $\boldsymbol{w}_{\nu i}$, $b_{\nu i}$ and $\boldsymbol{v}_i$ into $-\boldsymbol{w}_{\nu i}$, $-b_{\nu i}$ and $-\boldsymbol{v}_i$, respectively.
\item[(2)] Reset the parameters of the $\nu$th layer of the initial path $p_0$ to make the outputs of the $\nu+1$th layer invariant with the previous negative-form operation. The modified path is denoted by $p_0'$.
\item[(3)] Add units $u_1, u_2, \dots, u_{i-1}$ one by one in their original positions on the basis of $p_0'$ to produce the paths $p_1', p_2', \dots, p_{i-1}'$, which correspond to $p_1, p_2, \dots, p_{i-1}$, respectively. The parameters of $u_j$ for $1\le j \le i-1$ are the same as those in the original paths, except for the case that when $d_j=\nu+1$ the input-weight vector of $u_j$ should be reset.
\item[(4)] On the basis of $p_i'$, introduce $u_k$'s for $k=i+1, i+2, \dots, \zeta$ in their original positions and reset their parameters via corollary 10 and theorem 13.
\end{itemize}
\end{proof}

\subsection{Second Case of of Single Negative Unit}
\begin{prp}[A special two-sided solution]
Under the notations of theorem 24, suppose that to $u_i$ we have $i=2$ and among $u_k$'s for $1\le k\le \zeta$ only $u_1$ is in a layer shallower than the layer of $u_2$. After $u_2$ being modified to $-u_2$, to find a solution of any $s(\boldsymbol{x}) \in \mathfrak{S}_n(R, K')$, first set the parameters of $u_2$ by theorem 24. The input parameters of $u_{1}$ could be invariant, while its output-weight vector $\boldsymbol{v}_1$ should be updated for $s_{1}(\boldsymbol{x})=s_{0}(\boldsymbol{x})+\lambda_{1}\sigma(\boldsymbol{w}_{1}^T\boldsymbol{x}+b_{1})$ and simultaneously this update should not affect the the knot of $-u_2$ whose depth is deeper than that of $u_1$, for which a solution exists, if
\begin{equation}
\begin{cases}
\ {\boldsymbol{\alpha}_1'}^T\boldsymbol{v}_1 = \lambda_{1}
\\
\ {\boldsymbol{\alpha}_{12}'}^T\boldsymbol{v}_1 = 0
\end{cases}
\end{equation}
has a solution of $\boldsymbol{v}_1$, where $\boldsymbol{\alpha}_1'$ is the updated version of $\boldsymbol{\alpha}_1$ from equation 3.17 after the negative-knot operation and $\boldsymbol{\alpha}_{12}'$ from
\begin{equation}
\Delta_{2} = {\boldsymbol{\alpha}_{12}'}^T\boldsymbol{v}_1(\boldsymbol{w}_{1}^T\boldsymbol{x} + b_1)=0
\end{equation}
by equation 2.14, which is the influence of $u_1$ on $-u_2$; the first formula of equation 4.9 ensures the generation of $s_{1}(\boldsymbol{x})$ and the second one means that the update of $\boldsymbol{v}_1$ doesn't influence the knot of $-u_2$.
\end{prp}
\begin{proof}
To produce $\lambda_{1}$ for $s_{1}(\boldsymbol{x})$, because $u_{1}$ is in a layer shallower than that of $u_{2}$ and the parameters of $u_2$ have been updated by the operations of theorem 24, the output-vector $\boldsymbol{v}_1$ of $u_1$ should be reset---this is the first formula of equation 4.9. The updated $\boldsymbol{v}_1$ may in turn affect the knot $-k_2$ and the second formula of equation 4.9 solves this problem by removing the disturbance.
\end{proof}

\begin{thm}[Principles of two-sided solutions-\Romannum{2}]
Notations as in theorem 24, suppose that $i=\theta+1$ (i.e., $u_{i}=u_{\theta+1}$) with $\theta \ge n+1$ and that $u_j$'s for $1 \le j \le \theta=i-1$ are all in the layers shallower than the one of $u_i$, and that the input parameters of $u_j$'s form a matrix $W$ of equation 2.19 whose rank is $n+1$. After $u_{i}$ becoming $-u_i$, for a solution of any $s(\boldsymbol{x}) \in \mathfrak{S}_n(R, K')$, first, in the initial path $p_0$, the input parameters of $u_i$ are set to yield a global unit of region $r_0$ and its output-weight vector is changed to be $-\boldsymbol{v}_i$; then step (2) of theorem 24 is operated. Second, use theorem 15 to generate the knot $-k_i$ of $-u_i$ as well as the linear functions $s_{j}(\boldsymbol{x})$'s, through a solution of each $\boldsymbol{v}_j$ of
\begin{equation}
\begin{cases}
\begin{aligned}
{\boldsymbol{\alpha}_j'}^T\boldsymbol{v}_j &= \lambda_{j} \\
\boldsymbol{\alpha}_{ji}'^T\boldsymbol{v}_{j} &= \beta_{j}
\end{aligned}
\end{cases}
\end{equation}
similar to equation 4.9, with $\beta_{j}$'s for all $j$ subject to
\begin{equation}
\sum_j\beta_{j}(\boldsymbol{w}_j^T\boldsymbol{x}+b_j) + \boldsymbol{w}^T\boldsymbol{x}+b = -\boldsymbol{w}_i^T\boldsymbol{x}-b_i,
\end{equation}
where $\boldsymbol{w}_j^T\boldsymbol{x}+b_j$ is from the equation $\boldsymbol{w}_j^T\boldsymbol{x}+b_j=0$ of the knot $k_j$ of $u_{j}$,  $\boldsymbol{w}^T\boldsymbol{x}+b$ from the knot of $-u_i$ in path $p_0'$, and $-\boldsymbol{w}_i^T\boldsymbol{x}-b_i$ from $-k_i$. Third, units $u_k$'s for $k\ge i+1$ are processed by step (4) of theorem 24. By the above four steps, any $s(\boldsymbol{x}) \in \mathfrak{S}_n(R, K')$ can be constructed by network $\mathfrak{N}$.
\end{thm}
\begin{proof}
By theorem 15, $\sum_j\beta_{j}(\boldsymbol{w}_j^T\boldsymbol{x}+b_j)$ can realize arbitrary linear function, such that $ \boldsymbol{w}^T\boldsymbol{x}+b = 0$ of equation 4.12 can be compensated to yield the equation of $-k_i$.
\end{proof}

\begin{rmk}
Proposition 6 is a special case of this theorem.
\end{rmk}

\subsection{Third Case of of Single Negative Unit}
\begin{thm}[Principles of two-sided solutions-\Romannum{3}]
Notations being from theorem 24, suppose that $i=\theta+2$ (or $u_{i}=u_{\theta+2}$) with $\theta \ge n+1$ and units $u_{j}$ for $1 \le j \le \theta=i-2$ are in the layers shallower than the one of $u_{i}$, and that the input parameters of $u_{j}$'s form a matrix $W$ of equation 2.19 with rank $n+1$. Unit $u_{i-1}$ is in a layer deeper than the one of $u_i$ and this is the difference from theorem 25. For $u_{i}$ becoming $-u_i$, to find a solution of an arbitrary $s(\boldsymbol{x}) \in \mathfrak{S}_n(R, K')$, first, $u_i$  and $u_{j}$'s are processed by theorem 25. Second, the parameters of $u_{i-1}$ remain the same as the original ones. Third, the knot $k_{i-1}$ of $u_{i-1}$ is produced by corollary 9, namely by adjusting the output parameters of $u_j$'s, after which the linear function $s_{i-1}(\boldsymbol{x})$ can be simultaneously realized. To satisfy both the first and third steps, equation 4.11 should be modified to
\begin{equation}
\begin{cases}
\begin{aligned}
{\boldsymbol{\alpha}_j'}^T\boldsymbol{v}_j &= \lambda_{j} \\
\boldsymbol{\alpha}_{ji}'^T\boldsymbol{v}_{j} &= \beta_{j} \\
\boldsymbol{\alpha}_{j,i-1}'^T\boldsymbol{v}_{j} &= \gamma_{j}
\end{aligned}
\end{cases},
\end{equation}
in which the third formula is for the production of knot $k_{i-1}$. Fourth, units $u_{k}$'s for $k\ge i+1$ are processed by step (4) of theorem 24. Through the four steps, any $s(\boldsymbol{x}) \in \mathfrak{S}_n(R, K')$ can be constructed via network $\mathfrak{N}$.
\end{thm}
\begin{proof}
The proof is composed of four parts. (1) By theorem 24, the parameters of $\mathfrak{N}$ after the $\nu+1$th layer must remain invariant to ensure the establishment of equations 4.7 and 4.8. (2) The output parameters of $u_{j}$'s have been reset by theorem 25; in combination with (1), $u_{i-1}$ may not produce knot $k_{i-1}$ as before. (3) To solve this problem, one method is by corollary 9, that is, adjusting the output parameters of $u_{j}$'s to enable $u_{i-1}$ to generate $k_{i-1}$. (4) On the basis of (3), $s_{i-1}(\boldsymbol{x})$ only depends on the parameters of the layers deeper than the one of $u_{i-1}$, which are not changed according to (1).
\end{proof}

\subsection{General Two-Sided Solutions}
The preceding results are typical and a general two-sided solution can be reduced to their combination or generalization. For instance, to the case of multiple negative units, we can process them as follows. Suppose that by the preceding sections, we obtained $K' = \{k_1, k_2, \dots, -k_i, \dots, k_{\zeta}\}$, a modified version of equation 3.31, in which $k_i$ is changed into its negative form $-k_i$. To change another knot $k_j$ for $j \ne i$ to be negative, we can regard $-k_i$ as an ordinary knot that needs not to be specially treated and again apply the method of a single negative unit. This procedure can be repeatedly done.

\section{Mechanism of Multiple Outputs}
The preceding sections studied the networks with a single output and and this section turns to the multiple-output case. Given a deep feedforward ReLU network $\mathfrak{N}_m$ with $m$ units for $m\ge 2$ in the output layer, how to simultaneously implement $m$ desired piecewise linear functions is an important problem, not only related to the expressive capability of $\mathfrak{N}_m$, but also associated with exploiting the efficiency of hidden-layer units to produce knots and functions. Thus, the mechanism of multiple outputs is also a component of the black box of deep ReLU networks.

\subsection{Solution of Multiple-Output Functions}
\begin{thm}[Mechanism of Multiple Outputs]
Let $\mathfrak{N}_m$ be a deep feedforward ReLU network having $m$ outputs for $m \ge 2$, each of which corresponds to a unit of the output layer with a linear activation function, denoted by $u_{i}$ for $i=1,2,\dots,m$. Suppose that paths $p_i$'s of $\mathfrak{N}_m$ are only different in the units of the output layer, with $u_i \in p_i$ and $u_i \notin p_j$ if $j\ne i$. Denote by $s_i(\boldsymbol{x})$ the linear function produced by $p_i$. Add a new unit $\mathcal{U}$ in the $\nu$th layer of $p_i$ with $\nu\ne \Phi$, yielding path $p_i'$. Suppose that $p_i'$ is adjacent to $p_i$. Let
\begin{equation}
s_i'(\boldsymbol{x}) = s_i(\boldsymbol{x}) + \lambda_i\sigma(\boldsymbol{w}^T\boldsymbol{x}+b)
\end{equation}
be the linear function of $p_i'$, continuous with $s_i(\boldsymbol{x})$ at knot $\mathcal{L}$ generated by $\mathcal{U}$, in which $\boldsymbol{w}^T\boldsymbol{x}+b$ is from the equation of $\mathcal{L}$ and $\lambda_i=\boldsymbol{\alpha_i}^T\boldsymbol{v}$ from equation 3.17, with $\boldsymbol{v}$ being the output-weight vector of $\mathcal{U}$ and $\boldsymbol{\alpha_i}$ being
\begin{equation}
\boldsymbol{\alpha}_i = W_{\nu+2}W_{\nu+3}\dots\boldsymbol{w}_i
\end{equation}
similarly to equation 3.18, where $\boldsymbol{w}_i$ is the input-weight vector of $u_i$. Suppose that the length of $\boldsymbol{\nu}$ is not less than $m$. Write
\begin{equation}
A =
\begin{pmat}({})
\boldsymbol{\alpha}_1, \boldsymbol{\alpha}_2, \dots, \boldsymbol{\alpha}_m \cr
\end{pmat}^T
\end{equation}
and
$\boldsymbol{\lambda} = [\lambda_1, \lambda_2, \dots, \lambda_m]^T$. Then arbitrary $s_i'(\boldsymbol{x})$'s for all $i$ with $s_i'(\boldsymbol{x})$ continuous with $s_i(\boldsymbol{x})$ can be simultaneously realized by $p_i'$'s, respectively, in terms of
\begin{equation}
A\boldsymbol{v} = \boldsymbol{\lambda},
\end{equation}
provided that the rank of matrix $A$ is $m$. Equation 5.4 is equivalent to
\begin{equation}
\begin{aligned}
\begin{cases}
{\boldsymbol{\alpha}_1}^T\boldsymbol{v} &= \lambda_{1} \\
{\boldsymbol{\alpha}_2}^T\boldsymbol{v} &= \lambda_{2} \\
&\vdots \\
{\boldsymbol{\alpha}_m}^T\boldsymbol{v} &= \lambda_{m}
\end{cases},
\end{aligned}
\end{equation}
whose each equation is the single-output case of deep feedforward ReLU networks.
\end{thm}
\begin{proof}
By equation 5.2, the difference between $\boldsymbol{\alpha}_{\nu}$ and $\boldsymbol{\alpha}_{\mu}$ for $1\le \nu, \mu \le m$ and $\nu \ne \mu$ lies in the input-weight vectors $\boldsymbol{w}_{\nu}$ and $\boldsymbol{w}_{\mu}$ of $u_{\nu}$ and $u_{\mu}$, respectively. All of $\boldsymbol{\alpha}_i$'s lead to the matrix $A$ of equations 5.3 and 5.4. The solution of $s_i'(\boldsymbol{x})$'s exists if the rank of $A$ is $m$.
\end{proof}

\begin{cl}[Solution of multiple-output functions]
In theorem 27, when only considering $u_1$ and ignoring other units of the output layer, the subnetwork of $\mathfrak{N}_m$ is denoted by $\mathscr{N}_1$. Suppose that any continuous piecewise linear function $g_1(\boldsymbol{x}) \in \mathfrak{C}(R)$ can be realized by $\mathscr{N}_1$, where the set $R$ of regions is derived from $\mathscr{N}_1$. Let $p^{(1)}_j$'s for $1\le j\le \zeta$ be the paths of $\mathscr{N}_1$ generating $g_1(\boldsymbol{x})$, with $p^{(1)}_1$ being the initial path. Suppose that $p^{(1)}_j$ for each $j$ corresponds to a set
\begin{equation}
P_j=\{p^{(i)}_j: 1\le i\le m\}
\end{equation}
of paths of $\mathfrak{N}_m$, subject to: (1) the paths of $P_j$ are only different in the output-layer unit as $p_i$'s of theorem 27; (2) each element $p^{(i)}_1$ of $P_1$ is the initial path of $\mathscr{N}_i$; (3) the paths of each $P_{\nu}$ for $2\le \nu\le m$ satisfy the conditions of $p_i'$'s of theorem 27 for adjacent-path relationships and linear-function production. Then any output-layer unit $u_i$ of $\mathfrak{N}_m$ for $i\ne 1$ is also capable of realizing an arbitrary $g_i(\boldsymbol{x}) \in \mathfrak{C}(R)$, independently of other units of the output layer.
\end{cl}
\begin{proof}
By theorem 17, each linear function of $g_i(\boldsymbol{x})$ is produced by adjacent paths on the basis of an initial one of $\mathscr{N}_i$. Note that subnetwork $\mathscr{N}_1$ includes the hidden layers of $\mathfrak{N}_m$, so after $\mathscr{N}_1$ having been set for $g_1(\boldsymbol{x}) \in \mathfrak{C}(R)$, the set $R$ and the associated set $K$ of the knots forming $R$ are shared by all the other subnetworks $\mathscr{N}_i$'s for $i\ne 1$. Then by theorem 27, the conclusion follows.
\end{proof}

\subsection{Principles of Multiple-Knot Control}
Write
\begin{equation}
\mathcal{N} = n\prod_{i=1}^{\nu-1}m_{i}1,
\end{equation}
a subnetwork of network $\mathfrak{N}$ up to the $\nu-1$th layer together with a unit $u_1$ of the $\nu+1$th layer. Let $R=\{r_1, r_2, \dots, r_{\zeta}\}$ be the set of the regions formed by the hidden layers of $\mathcal{N}$ and the corresponding set of the knots is denoted by $K$. Then $\mathcal{N}$ generates a function
\begin{equation}
g'(\boldsymbol{x}) = \sigma(g(\boldsymbol{x})),
\end{equation}
where $g(\boldsymbol{x}) \in \mathfrak{C}_n(R)$ (equation 3.1 of definition 8) is the output of $u_1$ when its activation function is a linear one and $\sigma(x)$ is the activation function of a ReLU.

\begin{prp}[An example of multiple-knot control]
Under the above notations, suppose that $K$ forms a strict partial order of equation 3.33 and that condition (1) of theorem 12 for an initial path is modified to $m_{\nu-1} \ge n$. If the conditions of theorem 23 are satisfied, then in the sense of implementing an arbitrary $g(\boldsymbol{x})$ of equation 5.8, any $g'(\boldsymbol{x})$ can be realized by $\mathcal{N}$.
\end{prp}
\begin{proof}
When only considering $g(\boldsymbol{x})$ regardless of the $\sigma$ operator of a ReLU, to produce a piecewise linear function, the only difference between $\mathcal{N}$ and $\mathfrak{N}$ is that the unit of the output layer of the former adds a bias parameter, for which
condition (1) of theorem 12 is modified to $m_{\nu-1} \ge n$. Then the conclusion follows by theorem 23.
\end{proof}

\begin{thm}[Principles of knot production-\Romannum{5}]
Notations from equations 5.7 and 5.8, each linear function $s_i(\boldsymbol{x})$ for $1\le i \le \zeta$ of $g(\boldsymbol{x})$ on $r_i$ corresponds to an $n-1$-dimensional hyperplane $l_i$ with equation $s_i(\boldsymbol{x}) = 0$. When $r_i \in l_i^+$ or $r_i \in l_i^0$, $l_i$ would not divide $r_i$; otherwise, $l_i$ subdivides $r_i$ and a knot $k_i$ is formed. When $s_i(\boldsymbol{x})$ can be designed to be a desired linear function through realizing a certain $g(\boldsymbol{x})$, we say that $k_i$ is controllable. Then if $g(\boldsymbol{x})$ can simultaneously implement $\theta$ desired linear functions, a single unit $u_1$ of the output layer of $\mathcal{N}$ can generates at most $\theta$ controllable knots on different regions.
\end{thm}
\begin{proof}
The conclusion is obvious. The difference between knot production via $\mathcal{N}$ and function realization via $\mathfrak{N}$ was discussed in the proof of proposition 7 and at the beginning of this theorem.
\end{proof}

\begin{rmk-3}
This theorem unifies knot production and function construction to a common framework, including not only the preceding knot-production principles but also other mechanisms, especially the continuity-restriction principle of theorem 21. The efficiency of parameter sharing is thus further exploited.
\end{rmk-3}

\begin{rmk-3}
Proposition 7 is an example of neural networks capable of multiple-knot control. To this theorem, if arbitrary piecewise linear function $g(\boldsymbol{x})$ can be constructed, multiple-knot control via $\mathcal{N}$ is also possible.
\end{rmk-3}

\begin{thm}[Principles of knot production-\Romannum{6}]
Let $N_{\nu} = n\prod_{i=1}^{\nu-1}m_{i}m_{\nu}$ be a subnetwork of $\mathfrak{N}$ up to the $\nu$th layer, which can also be obtained by adding $m_{\nu}-1$ units in the $\nu$th layer of $\mathcal{N}$ of equation 5.7. Denote by $u_i$ for $1\le i \le m_{\nu}$ the $i$th unit of the output layer of $N_{\nu}$, outputting a piecewise linear function $g'_i(\boldsymbol{x}) = \sigma(g_i(\boldsymbol{x}))$ analogous to equation 5.8. Then under corollary 14 and theorem 28, if arbitrary $g_1(\boldsymbol{x})$ via $u_1$ can be realized for generating multiple controllable knots, each unit $u_i$ for $i \ne 1$ also has this capability, independently of the other output-layer units.
\end{thm}
\begin{proof}
The proof is by corollary 14 and theorem 28.
\end{proof}

\begin{rmk-4}
Compared to theorem 28, this theorem again exploits the parameter efficiency through the principle of multiple outputs of corollary 14.
\end{rmk-4}

\begin{rmk-4}
This theorem is correlated with proposition 2 for the number of the regions of network $\mathfrak{N}$. If a layer of $\mathfrak{N}$ is exclusively designed by this theorem to enable the units to yield multiple controllable knots, the number of regions of this layer can satisfy the condition of proposition 2, such that exponential grow with respect to the depth is possible (corollary 3).
\end{rmk-4}

\section{Univariate Function Approximation}
To the mechanism of network $\mathfrak{N}$, the difference between input dimensionality $n=1$ and $n \ge 2$ lies in two aspects. First, there exists only one strict partial order over $[0, 1]$ and any spline realized by $\mathfrak{N}$ is either by theorem 23 for one-sided solutions or by the results of section 4 for two-sided solutions. Second, the continuity-restriction principle of theorem 21 is not applicable to $n=1$. Thus, the solution space of one-dimensional input is much simpler, for which we know more about this case, especially the minimum number of units required. However, the remaining principles still ensure the solution complexity to some extent, through which the mechanism of $\mathfrak{N}$ regardless of input dimensionality can be highlighted.

\subsection{A Solution of Universal Approximation}
We first give a trivial solution of universal approximation for arbitrary input dimensionality, in the sense that it is equivalent to that of a two-layer ReLU network.
\begin{prp}[Trivial universal approximation]
Let $R$ be a set of regions divided by the units of a two-layer ReLU $\mathscr{N}$. To realize a piecewise linear function $s(\boldsymbol{x}) \in \mathfrak{C}(R)$ via network $\mathfrak{N}$, except for the region of an initial path, the units for generating $R$ can all be introduced in the last hidden layer and the associate solution for $s(\boldsymbol{x})$ is equivalent to that of $\mathscr{N}$. In this sense, $\mathfrak{N}$ is also a universal approximator.
\end{prp}
\begin{proof}
Design an initial path of $\mathfrak{N}$ whose region $r_0$ includes $U$ by corollary 11 and construct a linear function of $s(\boldsymbol{x})$ over $r_0$ by theorem 12. After that, introduce the units in the last hidden layer of $\mathfrak{N}$ to obtain $R$ by corollary 10. The remaining linear functions of $s(\boldsymbol{x})$ are implemented by the principles of two-layer ReLU networks (see proposition 4 of section 3.3). Since $\mathscr{N}$ is a universal approximator \citep*{Huang2024}, the second conclusion follows.
\end{proof}

To one-dimensional input, given a knot $x$ generated by network $\mathfrak{N}$ with equation $wx+b=0$, if $w>0$, we call it a \textsl{positive knot}; otherwise, it is a \textsl{negative knot}. A region of one-dimensional input is called an \textsl{interval} in section 6.
\begin{prp}[An example of knot control]
Suppose that the depth of network $\mathfrak{N}$ is $\Phi=2$. Let $p_1$ be a path of $\mathfrak{N}$. Suppose that $r_1=[x_1, 1]$ is the interval of $p_1$ and knot $x_1 \in [0, 1)$ is generated by unit $u_{21}$ of the second layer. Add a unit $u_{11}$ in the first layer of $p_1$ to form a new path $p_2$ whose interval is $r_2=[x_2, x_3]$, with $x_3 \le 1$ and knot $x_2$ produced by $u_{11}$. The introduction of $u_{11}$ can change $x_1$ into $x_1'$. Suppose that $p_2$ is adjacent to $p_1$. Let $s_1(x)$ and $s_2(x)$ be the linear functions output by $p_1$ and $p_2$, respectively, satisfying $s_2(x)=s_1(x)+\lambda_1\sigma(x-x_2)$. Then if $|\lambda_{1}|$ is sufficiently small, we have $x_1' \in (-\infty, x_2)$.
\end{prp}
\begin{proof}
Because $u_{11}$ is in the first layer and $u_{21}$ in the second one, the activation of $u_{11}$ can influence the knot $x_1$ of $u_{21}$ and changes it into $x_1'$. The purpose of this proposition is to restrict $x_1'$ in $(-\infty, x_2)$ to avoid a knot greater than $x_2$. Under path $p_2$, to arbitrary $x \in [x_2, x_3]$, it activates both $u_{11}$ and $u_{21}$ and then
\begin{equation}
x \in x_1'^+ \cap x_2^+,
\end{equation}
for which there are two possibilities: (a) $x_1'$ is a positive knot with $x_1' \in (-\infty, x_2)$; (b) $x_1'$ is a negative knot with $x_1' \in (x_3, +\infty)$.

Let $w_1x+b_1=0$ and $w_2x+b_2=0$ be the equations of $x_1$ and $x_2$, respectively. Then $w_1, w_2>0$, since both $k_1$ and $k_2$ are positive knots. The equation of $x_1'$ is $w_1x+b_1 + \lambda_{1}(w_2x+b_2) = 0$ or
\begin{equation}
(w_1 + \lambda_{1}w_2)x + b_1 + \lambda_{1}b_2 = 0.
\end{equation}
If $|\lambda_{1}|$ is sufficiently small, $w_1 + \lambda_{1}w_2 > 0$ such that only case (a) is possible.
\end{proof}

\begin{rmk}
The knot $x_1'$ would be restored to the original $x_1$ when $x \in [x_1, x_2]$ and needs not to be processed in function construction over $[x_2, 1]$---the purpose of this knot control.
\end{rmk}

\begin{thm}[A solution of univariate universal approximation]
Any continuous function $f: [0, 1] \to \mathbb{R}$ can be approximated by $\mathfrak{N}$ as precisely as possible, whose depth could be an arbitrary integer $\Phi \ge 2$, through implementing a continuous linear spline $s(x) \in \mathfrak{S}_1(R, K)$ of equation 3.34, provided that the maximum length of the intervals derived from $K$ is sufficiently small. If to achieve an approximation error $\varepsilon$, $\zeta$ linear pieces of $s(x)$ are required, the number of the units of $\mathfrak{N}$ satisfies
\begin{equation}
\Theta \ge \Phi + \zeta.
\end{equation}
\end{thm}
\begin{proof}
This theorem is a solution of theorem 23. We first construct a $s(x)$ satisfying $\sup{|s(x)-f(x)|} \le \varepsilon$ and then realize it through $\mathfrak{N}$. First see an example of depth $\Phi =2$ with two hidden layers. Under theorem 23, an initial path $p_0$ is selected to realize the initial linear function $s_0(x)$ of $s(x)$ on $[0, 1]$ via corollary 11, with the first and second layers of $p_0$ having one and two units, respectively.

Then add a unit (denoted by $u_{23}$) in the second layer of $p_0$ to produce knot $x_1$ (corollary 10) and the new path is denoted by $p_1$; the parameter $\lambda_1$ for $s_1(x)$ on $[x_1, 1]$ is set by proposition 4. Next, introduce $u_{12}$ in the first layer of $p_1$ for knot $x_2$ and the modified path is $p_2$, and then the following steps are required: (1) check whether the influence coefficient vector $\boldsymbol{\alpha}_2=\boldsymbol{0}$; (2) if $\boldsymbol{\alpha}_2\ne \boldsymbol{0}$, set the output parameters of $u_{12}$ to produce the parameter $\lambda_2$ for $s_2(x)$ on $[x_2,x_3]$ by theorem 13; (3) check if there's a unit of the second layer generating a knot in $[x_2, x_3]$. If $\boldsymbol{\alpha}_2=\boldsymbol{0}$ or the answer to step (3) is yes, transfer the unit $u_{12}$ for $x_2$ to the second layer of $p_2$ instead. Proposition 9 is also a solution of step (3).

Knot $x_3$ is processed as $x_1$, and $x_4$ as $x_2$. In general, if $i$ is an odd number, knot $x_i$ is initially introduced in the second layer and otherwise in the first layer. The first case is processed by the method of $x_1$, while the second case by that of $x_2$.

The case of depth $\Phi = 3$ is similar. Knot $x_i$ is initially introduced in the first, second and third layer for $i=3m, 3m+2,3m+1$, respectively, where $m$ is a nonnegative integer. The processes of $x_{3m+1}$ and $x_{3m+2}$ are the same as those of $x_{2m+1}$ and $x_{2m}$ when $\Phi=2$, respectively. To $x_{3m}$, three steps analogous to those for $x_{3m+2}$ are also required, with the differences that in step (3) the units of both the second and third layers should be checked, and that if the answer to step (1) or (3) is yes, first transfer the unit generating $x_{3m}$ to the second layer and then use the method of processing $x_{3m+2}$.

By the examples of $\Phi =2, 3$, the general method for arbitrary depth $\Phi$ can be obtained, including three steps: (1) to each knot $x_{i}$, where $i=\Phi m+\nu$ with $\nu = 1, 2, \dots, \Phi-1, \Phi$, initially introduce a unit $u_i$ in the $\Phi-\nu+1$th layer of the previous path and set its parameters to produce both $x_i$ and the associated linear function; if $\nu \ne \Phi$: (2) check whether the associated influence coefficient vector is a zero vector; (3) check if there exists a unit in a layer deeper than $d_i$ generating a knot in $(x_i,x_{i+1})$, where $d_i$ is the depth of the layer of $u_i$. If the answer to (2) or (3) is yes, transfer $u_i$ to the $d_i+1$th layer, update $d_i$ and the parameters of $u_i$, and then go to steps (2) and (3) again until the answers to both steps (2) and (3) are not. (4) Set the parameters of the unit to generate $x_i$ and $s_i(x)$.

Note that the above algorithm always reaches a solution and the ``worst case'' is the trivial solution of proposition 8---that is, all the knots are produced in the last hidden layer of $\mathfrak{N}$.

To the number of units required, in the initial path $p_0$, by theorem 12, the last hidden layer requires at least two units, while each of the remaining hidden layers needs at least one, with the total number being $\Phi-1+2 = \Phi+1$. Each knot needs one unit to form it and the total number of the units is $\zeta-1$. Thus, the minimum number of the units required is $\Phi+1+\zeta-1=\Phi+\zeta$ and this proves inequality 6.3.
\end{proof}

\subsection{Number of Units for General Case}

\begin{lem}
The continuity-restriction principle of theorem 21 is not applicable to network $\mathfrak{N}$ with input dimensionality $n=1$.
\end{lem}
\begin{proof}
In this case, a region of $U$ is a subinterval of $[0,1]$ and a knot is a point. The corresponding continuity-restriction principle is that: to an interval $I_i=[k_i, k_{i+1}] \subset [0, 1]$, if the linear functions of $s(x) \in \mathfrak{S}_1(R, K)$ on the two endpoints $k_i$ and $k_{i+1}$ are realized by $\mathfrak{N}$ for the adjacent intervals of $I_i$, the linear function on $I_i$ is automatically implemented.

We prove that the condition of the above conclusion is impossible to be fulfilled. By theorem 17, the linear functions of $s(x)$ are realized one by one via adjacent paths on the basis of an initial path. To the univariate case, this means that the linear functions are formed one by one through adjacent intervals that share a common endpoint.

To arbitrary interval $[k_i, k_{i+1}]$ with respect to path $p_i$, it is impossible for $\mathfrak{N}$ to first implement linear functions on $[k_{i-1},k_i]$ and $[k_{i+1},k_{i+2}]$, because the corresponding paths $p_{i-1}$ and $p_{i+1}$ are different in two units for $k_i$ and $k_{i+1}$ and not adjacent, such that the linear functions from $p_{i-1}$ to $p_{i+1}$ must use $p_i$ as a bridge in the form of the order $p_{i-1}, p_i, p_{i+1}$ or $p_{i+1}, p_i, p_{i-1}$. This completes the proof.

\end{proof}

\begin{thm}[Minimum number of units required]
To be capable of realizing any continuous linear spline $s(x) \in \mathfrak{S}_1(R, K)$ of equation 3.34 via network $\mathfrak{N}$ with one-dimensional input, the minimum number of the units required is
\begin{equation}
\min \Theta = \Phi + \zeta,
\end{equation}
where $\Phi$ is the depth of $\mathfrak{N}$ and $\zeta$ is the number of the linear functions of $s(x)$.
\end{thm}
\begin{proof}
By lemma 5, to the univariate case, it is impossible for the linear functions of $s(x)$ to use continuity-restriction principle of theorem 21 to reduce the number of the units. As in the proof of theorem 30, besides the units for the linear function of an initial path whose number is at least $\Phi+1$, each of the remaining linear functions needs a distinct unit. Thus, $\Phi + \zeta$ is the minimum number.
\end{proof}

\begin{rmk}
Notice the difference between this theorem and inequality 6.3 of theorem 30. The latter is derived from a special solution, while the former is a general conclusion for all the solutions.
\end{rmk}

\section{Multivariate Function Approximation}
The solution space of network $\mathfrak{N}$ with input dimensionality $n \ge 2$ for function approximation is complicated due to the combination of multiple strict partial orders, continuity restriction and two-sided solutions. We will give typical results to include as much as possible the most general solutions and construct a special solution of universal approximation. Section 7.1 combines multiple strict partial orders with continuity restriction to realize a desired piecewise linear function. Section 7.2 investigates the two-sided solutions under multiple strict partial orders. Section 7.3 discusses universal approximation.

\subsection{Multiple Strict Partial Orders}
By \citet*{Huang2024}'s definitions from 15 to 17, an \textsl{initial region} of a strict partial order $\mathscr{O}$ of equation 3.33 is the one $r_0 \subset l_1^0$ with $\dim{r_0 \cap l_1}=n-1$, whose corresponding linear function $s_0(\boldsymbol{x})$ is called the \textsl{initial linear function} of $s(\boldsymbol{x}) \in \mathfrak{S}_n(R, K)$. On the bias of $\mathscr{O}$, another strict partial order $\mathscr{O}'$ can be formed with $r_0$ or one of the ordered regions of $\mathscr{O}$ being the initial region, and the third one $\mathscr{O}''$ whose initial region is from $r_0$ and all the ordered regions of $\mathscr{O}$ and $\mathscr{O}'$; this process of order production can be recursively done. All the formed strict partial orders comprise an ``\textsl{order tree}'', denote by $\mathcal{T}$, with $r_0$ being the \textsl{root region} of $\mathcal{T}$; intuitive examples of an order tree can be found in Figures \ref{Fig.3}a, \ref{Fig.4}b, \ref{Fig.5}b and \citet*{Huang2024}'s Figure 7.

\begin{thm}[Function construction over multiple strict partial orders]
Suppose that the set $K$ of the knots derived from network $\mathfrak{N}$ with input dimensionality $n \ge 2$ forms an order tree $\mathcal{T}$. Then $\mathcal{T}$ is composed of multiple strict partial orders $\mathscr{O}_i$'s for $i=1, 2, \dots, \psi$. Let $R$ be the set of the regions of $U$ partitioned by $\mathfrak{N}$. Under the adjacent-path assumption (assumption 1 in section 4) for each $\mathscr{O}_i$, any continuous piecewise linear function $s(\boldsymbol{x}) \in \mathfrak{C}(R)$ can be realized by $\mathfrak{N}$, provided that:
\begin{itemize}
\item[\rm{(1)}] The linear function on root region $r_0$ of $\mathcal{T}$ is implemented by an initial path;
\item[\rm{(2)}] The conditions of theorem 23 for function construction over each single strict partial order $\mathscr{O}_i$ are satisfied;
\item[\rm{(3)}] Conditions \Romannum{2} and \Romannum{3} of \citet*{Huang2024}'s theorem 7 for $\mathscr{O}_i$'s are fulfilled;
\item[\rm{(4)}] The linear functions on the regions that are not included in $\mathcal{T}$ can be realized by the continuity-restriction principle of theorem 21.
\end{itemize}
\end{thm}
\begin{proof}
Condition (1) is the basis of forming the remaining linear functions. Condition (2) ensures the implementation of a spline over each $\mathscr{O}_i$. Condition (3) is necessary for realizing a piecewise linear function over the whole $\mathcal{T}$, through excluding or preserving the influences between $\mathscr{O}_i$'s (see the proof of \citet*{Huang2024}'s theorem 7). Condition (4) yields the linear functions that cannot be directly produced by $\mathcal{T}$.
\end{proof}

\begin{rmk-2}
The adjacent-path assumption of section 4 is used here for two reasons. One is to provide the necessary conditions for each strict partial order and the other is to preserve the knot of regions when setting the parameters for linear functions. The method of resolving this assumption has been discussed in the remark of the assumption in section 4. A key point is the simultaneous knot control when constructing linear functions and the solutions were given in section 2.3; in section 8.3 concrete examples will be provided.
\end{rmk-2}

\begin{rmk-2}
This theorem is a typical result and when there's more than one order tree, the underlying principles can still be applied. The central problem is how to combine multiple strict partial orders with the continuity-restriction principle and how to dealt with the influences among the strict partial orders.
\end{rmk-2}

\subsection{Two-sided Solutions}
Under multiple strict partial orders, if at least one of their units is changed into its negative form as in section 4, the associated solution for function construction is called a \textsl{two-sided solution}. The mechanism is the combination of the results of section 4 for a single strict partial order with those of section 2.3 for knot production; the former ensures the construction of linear functions and the latter generates the desired knots.

For instance, by condition (3) of theorem 32, the strict partial orders $\mathscr{O}_i$'s of $\mathcal{T}$ can be arranged in an order
\begin{equation}
\mathscr{O}_{i_1}, \mathscr{O}_{i_2}, \dots, \mathscr{O}_{i_{\psi}}
\end{equation}
with $1\le i_{\nu} \le \psi$ for $1\le \nu \le \psi$. Let $k_{i_{\nu}j_{\nu}}$'s for $j_{\nu}=1, 2, \dots, \zeta_{i_\nu}$ be the knots of $\mathscr{O}_{i_{\nu}}$, satisfying
\begin{equation}
k_{i_{\nu}1} \prec k_{i_{\nu}2} \prec \dots \prec k_{i_{\nu}\zeta_{i_{\nu}}}.
\end{equation}
We arrange all the knots of $\mathcal{T}$ into
\begin{equation}
\begin{aligned}
\{k_{i_{\nu}j_{\nu}}: 1&\le i_{\nu}\le \psi, 1\le j_{\nu} \le \zeta_{i_{\nu}}, 1\le \nu \le \psi\} = \\
\{k_{i_{1}1},k_{i_{1}2},\dots,&k_{i_{1}\zeta_{i_1}}, k_{i_{2}1},k_{i_{2}2},\dots,k_{i_{2}\zeta_{i_2}}, \dots, k_{i_{\psi}1},k_{i_{\psi}2},\dots,k_{i_{\psi}\zeta_{i_{\psi}}}\},
\end{aligned}
\end{equation}
in which the knots of the same $\mathscr{O}_i$ are in the order of equation 7.2, while those with different $\mathscr{O}_i$'s are in accordance with the order of equation 7.1.

By condition \Romannum{2} of \citet*{Huang2024}'s theorem 7, on the basis of equation 7.3, if any knot, say, $k_{i_{\nu}j_{\nu}}$ for $\nu \ge 2$, is modified to its negative form $-k_{i_{\nu}j_{\nu}}$, the influence of $-k_{i_{\nu}j_{\nu}}^+$ on $\mathscr{O}_{i_{\mu}}$'s for $\mu < \nu$ is the type of a global unit; and if further the influence on $\mathscr{O}_{i_{\mu}}$'s for $\mu > \nu$ is zero, we can regard equation 7.3 as a single strict partial order and use the results of section 4 to obtain a two-sided solution.

Even though the above conditions are not satisfied, the general principle mentioned at the beginning of this section can still be applied. The key point is to clarify the orders of the knots as well as the depths of the layers forming the knots, through which the influences of negative knots can be known and the corresponding operation can be selected.

Notice that all the results of section 4 are under adjacent-path assumption, in which the ordered regions are assumed to be preserved during the construction of two-sided solutions. However, some knot of a region may be modified during the construction. The solution is by the principles of knot production of section 2.3---that is, resetting the associated parameters to reproduce the original knot. We will give some examples in later section 8.3 to further explain this mechanism.

\subsection{A Solution of Universal Approximation}

\begin{figure}[!t]    
\captionsetup{justification=centering}
\centering
\subfloat[A standard partition.]{\includegraphics[width=2.3in, trim = {3.7cm 1.9cm 2.8cm 1.0cm}, clip]{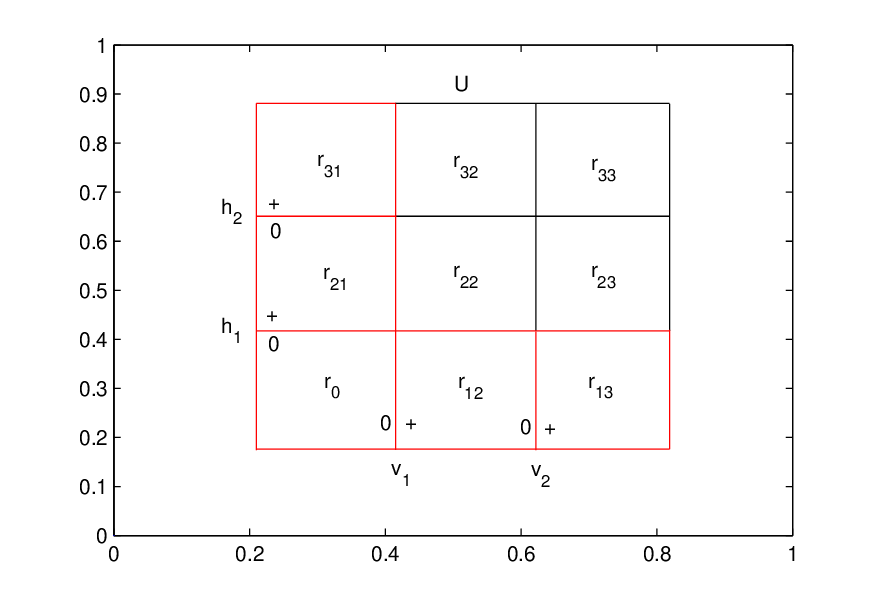}} \quad
\subfloat[Modified version of (a).]{\includegraphics[width=2.3in, trim = {3.7cm 1.9cm 2.8cm 1.0cm}, clip]{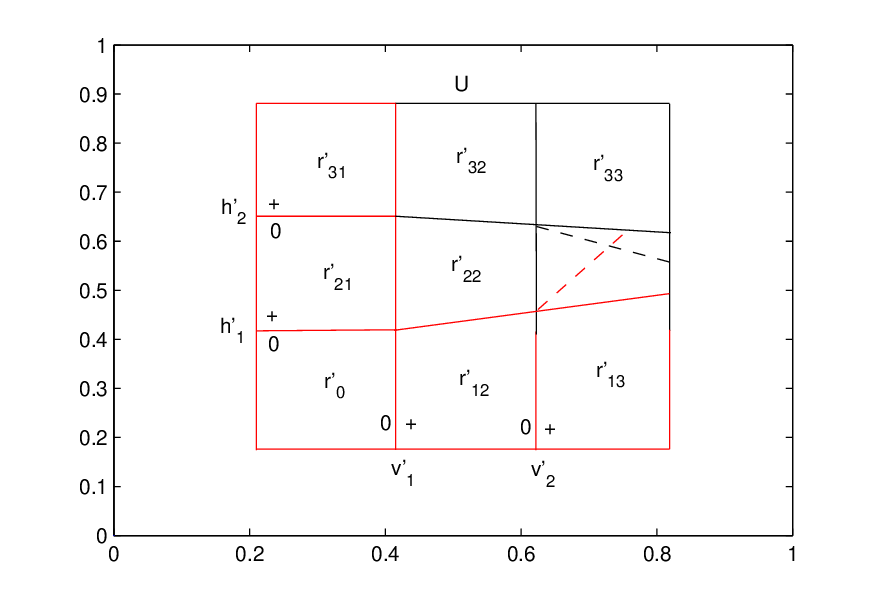}}
\caption{Modified standard partition via deep networks.}
\label{Fig.3}
\end{figure}

In two-dimensional case, when $U$ is divided by vertical and horizontal lines and this forms a \textsl{standard partition} of $U$, whose rigorous definition can be found in \citet*{Huang2024}'s definition 21.
\begin{thm}[A solution of multivariate universal approximation]
Any continuous function $f: [0, 1]^2 \to \mathbb{R}$ can be approximated by network $\mathfrak{N}$ with arbitrary precision, through realizing a continuous piecewise linear function $s(\boldsymbol{x})$ approximating $f(\boldsymbol{x})$. The depth $\Phi$ of $\mathfrak{N}$ could be an arbitrary positive integer $\Phi \ge 2$. To ensure an approximation error $\varepsilon$, if $\zeta$ linear pieces of $s(\boldsymbol{x})$ are required, the number of the units needed satisfies
\begin{equation}
\Theta \ge 2\zeta^{1/2} + 2\Phi - 1.
\end{equation}
\end{thm}
\begin{proof}
The first step is to use vertical and horizontal lines to form a standard partition $\mathscr{P}_0$ of $U$ such as Figure \ref{Fig.3}a. A piecewise linear function $s_0(\boldsymbol{x})$ approximating $f(\boldsymbol{x})$ can be constructed over $\mathscr{P}_0$ by \citet*{Huang2024}'s lemma 6; if the area of each region is sufficiently small, any desired approximation error can be assured. Denote by $r_{ij}$ for $i = 1, 2, \dots, m_1$ and $j= 1,2, \dots, m_2$ each region of $\mathscr{P}_0$ and write $r_0:=r_{11}$, as an example of Figure \ref{Fig.3}a.

By \citet*{Huang2024}'s definition 22, a partition $\mathscr{P}'$ is isomorphic to $\mathscr{P}_0$, if each region $r_{ij}'$ of $\mathscr{P}'$ exactly corresponds to $r_{ij}$ of $\mathscr{P}_0$, including the one-to-one maps of their knots as well as their adjacent relationships with other regions. If a partition $\mathscr{P}$ is obtained by network $\mathfrak{N}$, the only difference between $\mathscr{P}$ and $\mathscr{P}'$ is that the former can use a piecewise linear curve to divide $U$, and the concept of isomorphic partitions is applicable to the relationship between $\mathscr{P}$ and $\mathscr{P}_0$.

Suppose that partition $\mathscr{P}$ via $\mathfrak{N}$ is isomorphic to $\mathscr{P}_0$. It we say that $\mathscr{P}$ is a \textsl{slight-modification} of $\mathscr{P}_0$, it means that the area of each $r'_{ij}$ of $\mathscr{P}$ approximates its counterpart $r_{ij}$ of $\mathscr{P}_0$ with a desired accuracy, denoted by
\begin{equation}
|r'_{ij}-r_{ij}| < \epsilon
\end{equation}
for all $i$ and $j$, where $\epsilon$ is a sufficiently small positive real number; an example of $\mathscr{P}$ is shown in Figure \ref{Fig.3}b (regardless of the dashed lines). By the method of \citet*{Huang2024}'s lemma 6, a piecewise linear function $s(\boldsymbol{x})$ over $\mathscr{P}$ can be constructed to approximate $f(\boldsymbol{x})$ with arbitrary precision, despite the slight difference between $\mathscr{P}$ and $\mathscr{P}_0$.

The second step is to use network $\mathfrak{N}$ to realize partition $\mathscr{P}$ and function $s(\boldsymbol{x})$. Note that in this process, $\mathscr{P}$ may vary if necessary, and the function $s(\boldsymbol{x})$ should be accordingly updated; the constraint that $\mathscr{P}$ is a slight modification of $\mathscr{P}_0$ must be always satisfied during the process to ensure the approximation error.

We use a strategy similar to theorem 30 for univariate functions: introduce units from deep layers to shallow layers. First see the case of depth $\Phi=2$. In the example of Figure \ref{Fig.3}a, the partition $\mathscr{P}_0$ is via a two-layer ReLU network; construct an order tree $\mathcal{T}$ with root region $r_0$ that is composed of two strict partial orders $\mathscr{O}_1=v_{1} \prec v_{2}$ and $\mathscr{O}_2=h_{1} \prec h_{2}$, where $v_i$ and $h_i$ for $i=1,2$ are vertical and horizontal lines, respectively. To deep network $\mathfrak{N}$, we first construct an initial path $p_0$ by corollary 11, whose region includes $U$ as a subset and whose output function is the linear function of $s(\boldsymbol{x})$ on $r_0$. Next, generate a partition $\mathscr{P}$ isomorphic to $\mathscr{P}_0$ via $\mathfrak{N}$. As shown in Figure \ref{Fig.3}b, the order tree $\mathcal{T}$ of $\mathscr{P}_0$ should have a counterpart $\mathcal{T}'$ of $\mathscr{P}$ comprising $\mathscr{O}'_1=v'_{1} \prec v'_{2}$ and $\mathscr{O}'_2=h'_{1} \prec h'_{2}$, where $v'_{i}$ (or $h'_{i}$) is the piecewise linear curve corresponding to $v_i$ (or $h_i$). The construction method is as follows.

Let $x_1 = i/M$ for $1\le i \le M-1$ and $M=3$ be the equation of $v_i$ and $x_2 = i/M$ be the equation of $h_i$. We first realize $h_i$'s and the associated linear functions in the second layer of $\mathfrak{N}$ via corollary 10 and proposition 4. Then add units $u_{1i}$'s in the first layer one by one for $v_i$'s, respectively. For example, to $v_1$, use corollary 10 and theorem 13 to set the parameters of $u_{11}$, and then each horizontal line $h_i$ formed in the second layer would change its direction to become $h_i'$ (corollary 4).

We then investigate the changes of the regions due to the alteration from $h_i$ to $h_i'$ for all $i$. Because the unit $u_{12}$ to be introduced for $v_2$ doesn't affect the direction of $h'_i$'s in $v_2'^{0}$, we assume that $v_2'$ already exists such that the regions $r'_{\nu2}$'s for $\nu=1,2,3$ have been formed for investigation. Check the following three conditions: (1) whether the partition $\mathscr{P}$ up to now is isomorphic to the corresponding part of $\mathscr{P}_0$; (2) whether inequality 7.5 is fulfilled for all the modified regions; (3) whether the associated influence coefficient vector is not a zero vector. If any of the three conditions is not satisfied, transfer $u_{11}$ to the second layer to avoid line-direction changes. Note that condition (1) implicitly includes the check of the adjacent-path condition, since if the associated paths are not adjacent, the isomorphic relationship would be destroyed.

Suppose that the introduction of $u_{11}$ is accepted by the above three conditions. Adding $u_{12}$ in the first layer for $v'_{2}$ could further change the direction of $h'_{i}$'s, and the dotted lines in Figure \ref{Fig.3}b mean that $h'_{1}$ intersect $h'_{2}$ due to this line-direction change. In this case condition (1) is not satisfied, so we should transfer $u_{12}$ to the second layer for $v'_{2}$. Even though condition (1) is fulfilled, if $|r'_{\nu3}-r_{\nu3}| \ge \epsilon$ for some $\nu$ or the associated influence coefficient vector is a zero vector, the previous operation is also required.

It is possible that after transferring a unit to the second layer, condition (1), (2) or (3) is still violated due to the effect of the preceding operations in a new part of $U$, for which we should further move the associated unit of the first layer to the second one. In general, after any operation, the above three conditions should be checked and the next step is based on the checking result. The process of $v_2'$ ends until all the three conditions are satisfied.

The continuity-restriction principle of theorem 21 is needed, since the function construction above is only for the regions of order tree $\mathcal{T}'$. The proof is similar to that of \citet*{Huang2024}' lemma 9. The mechanism is that, for example, when the linear functions on the regions of Figure \ref{Fig.3}b composed of red sides are realized by network $\mathfrak{N}$, those on the remaining regions are automatically implemented.

As the univariate case of theorem 30, the above construction algorithm always reaches a solution, since the ``worst case'' is the solution of proposition 8 when all the knots are generated in the last hidden layer of $\mathfrak{N}$.

To the the number of units required, by \citet*{Huang2024}'s theorem 10, a two-layer ReLU network needs at least $2\zeta^{1/2} + 1$ units for a standard partition including the three ones for initial region $r_0$, while the case of deep network $\mathfrak{N}$ is a modification of that. The difference is that the initial path $p_0$ of $\mathfrak{N}$ needs at least $2(\Phi-1)+3$ units and thus the minimum number of units is $2\zeta^{1/2}+1+2(\Phi-1)=2\zeta^{1/2}+2\Phi-1$.

The case of depth $\Phi=3$ can be dealt with similarly to $\Phi=2$ and the proof of theorem 30. For example, we can realize horizontal lines in the third layer of $\mathfrak{N}$ and introduce the units for vertical lines in shallower layers by the order in the proof of theorem 30; note that according to proposition 3, parallel vertical lines realized in different layers also result in vertical lines despite the influences from shallower layers and this could simplify the construction; after introducing or transferring a unit, conditions from (1) to (3) should be checked and transfer a unit into a deeper layer if necessary. The case of arbitrary depth $\Phi$ is similar.
\end{proof}

\begin{rmk-5}
This theorem can be easily generalized to arbitrary input dimensionality $n \ge 3$ on the basis of \citet*{Huang2024}'s theorem 10 and the counterpart of inequality 7.4 is
\begin{equation}
\Theta \ge n\zeta^{1/n} + 1 + n(\Phi-1).
\end{equation}
Although the solution is trivial in the sense that it is derived from that of two-layer ReLU networks, it may be useful to clarify whether a more efficient partition exists to reduce the lower bound of inequality 7.6.
\end{rmk-5}

\begin{rmk-5}
To a partition $\mathscr{P}$ of $\mathfrak{N}$, the basic structure of one standard partition or several ones combined can still exist as in two-layer ReLU networks, and the difference is that the former can use piecewise linear curve to form $\mathscr{P}$. To a training solution, it may not be necessary to restrict the line-direction change as this theorem, while modifying the direction is to enable a piecewise linear curve to fit the geometric feature of data.
\end{rmk-5}

\section{Explanation of Training Solutions}
We have developed the theory by deduction in the preceding sections and the purpose is to explain the training solution of experiments obtained by the back-propagation algorithm. That's the usual way in theoretical physics but still not pervasive in the area of neural networks. Although it's not the first time of successfully applying this methodology (see \citet*{Huang2024}), due to the complexity of deep neural networks as well as their impressive applications, the success of this paper is a milestone. Section 8.1 summarizes the main principles to be used. Section 8.2 proposes an algorithm to draw the knots of a partition derived from a training solution. Sections 8.3 and 8.4 explain two concrete solutions.

\subsection{Preliminaries}
The case of two-dimensional input is to be investigated for its easy intuitive demonstration. The main principles employed are as follows: (1) the results of section 2.3 for knot control or production; (2) corollary 4 for piecewise linear curves generated by a unit through deep layers; (3) theorem 4 and corollary 7 for the influence of a knot or piecewise linear curve on regions; (4) theorem 10 for the continuous property of the function output by neural networks and theorem 21 for the continuity-restriction principle; (5) theorem 22 for determining output weights via local information; (6) theorem 23 for spline construction over a single strict partial order; (7) the results of section 4 for two-sided solutions; (8) theorem 32 for functions over multiple strict partial orders.

We will use the above principles to show how to manually set the parameters for a given training solution and what the meaning of the derived parameters is, through which the ``black box'' of deep feedforward ReLU networks is revealed.

Note that when applying those principles, a cited conclusion may not strictly fit the experimental solution, but this problem can usually be solved by simple generalization or modification of the original result. For instance, theorem 15 has a constraint on the number of the units required; however, it is to ensure the capability of generating arbitrary knot and not for a certain knot; so even if it is not satisfied, some knot can still be realized.

In order to intuitively observe a training solution, in the next section, we first develop an algorithm by theorem 8 to draw the knots of a solution, which is applicable to arbitrary input dimensionality.

\subsection{Algorithm of Drawing Knots}
\begin{prp}[Principle of drawing knots]
Let $p_1$ and $p_2$ be two adjacent paths of network $\mathfrak{N}$ with unit $\mathcal{U} \in p_2$ but $\mathcal{U} \notin p_1$, whose regions are $r_1$ and $r_2$, respectively. By theorem 8, the knot $\mathcal{K}=r_1 \cap r_2$ is generated by $\mathcal{U}$. Let $\boldsymbol{x}_1 \in r_1$ and $\boldsymbol{x}_2 \in r_2$ be two points. Denote by $d_i=\min_{\boldsymbol{x} \in \mathcal{K}}|\boldsymbol{x}_i-\boldsymbol{x}|$ for $i=1,2$, namely the minimum distance from $\boldsymbol{x}_i$ to $\mathcal{K}$. Then, if $d_1$ and $d_2$ are sufficiently small, the line segment $\mathcal{L}$ connecting $\boldsymbol{x}_1$ and $\boldsymbol{x}_2$ intersects $\mathcal{K}$ or
\begin{equation}
\mathcal{L} \cap \mathcal{K} \ne \emptyset;
\end{equation}
moreover,
\begin{equation}
\lim_{|\boldsymbol{x}_2-\boldsymbol{x}_1|\to 0} \boldsymbol{x}_2 = \lim_{|\boldsymbol{x}_2-\boldsymbol{x}_1|\to 0} \boldsymbol{x}_1 = \boldsymbol{x} \in \mathcal{K}.
\end{equation}
\end{prp}
\begin{proof}
The conclusion is obvious by the relationship between $r_1$ and $r_2$.
\end{proof}

\begin{algorithm}
\caption{Drawing knots of a training solution}  
\begin{itemize}
\item[(1)] Discretize each dimension of $ \boldsymbol{x} = [x_1, x_2, \dots, x_n]^T \in U$ by step $\Delta x$ to form a set $D$ of points with cardinality $|D| = N^n$, where $N=[1/\Delta x]$ is the integer part of $1/\Delta x$.
\item[(2)] Let $D'$ be the set of points of $D$ whose elements are arranged by an $n$-dimensional array, that is,
    \begin{equation}
    D'(i_1, i_2, \dots, i_n) = [i_1\Delta x, i_2\Delta x, \dots, i_n\Delta x]^T
    \end{equation}
    where $i_k = 0, 1, \dots, N-1$ for $k = 1, 2, \dots, n$. Denote by $p(D'(i_1, i_2, \dots, i_n))$ the path of network $\mathfrak{N}$ activated by point $D'(i_1, i_2, \dots, i_n)$. Then do the operations as follows:
\begin{algorithmic}
\FOR{$i_1, i_2, \dots, i_n=1$ to $N-1$}
\IF{$p(D'(i_1, i_2, \dots, i_n)) \ne p(D'(i_1-1, i_2, \dots, i_n))$ \OR $\ne p(D'(i_1, i_2-1, \dots, i_n))$ \OR $\dots$ \OR $\ne p(D'(i_1, i_2, \dots, i_n-1))$}
\STATE Set $D'(i_1, i_2, \dots, i_n)$ to be a knot point, draw it in the figure and save it in the set $K'$.
\ENDIF
\ENDFOR
\end{algorithmic}
\end{itemize}
\end{algorithm}

\begin{prp}[Effect of algorithm 1]
Let $K = \{k_1, k_2, \dots, k_{\zeta}\}$ be the set of the knots of network $\mathfrak{N}$. To algorithm 1, we have
\begin{equation}
\lim_{\Delta x\to 0}K' \subset K
\end{equation}
and
\begin{equation}
\lim_{\Delta x\to 0}K' \cap k_i \ne \emptyset
\end{equation}
for $i=1, 2, \dots, \zeta$.
\end{prp}
\begin{proof}
First see the case of adjacent paths and regions. Equation 8.4 is by proposition 10. To equation 8.5, we first see an example of two-dimensional input (i.e., $n=2$). In this case, the condition of the ``if'' statement in step (2) of algorithm 1 is $p(D'(i_1, i_2) \ne p(D'(i_1-1, i_2)$ or $\ne p(D'(i_1, i_2-1)$. The first condition ensures that all the knots except for the vertical ones can be detected by the algorithm, while to the second condition only horizontal knots cannot be detected; and their combination can reach a knot with arbitrary direction. The general $n$-dimensional case can be analogously proved.

It's possible that algorithm 1 encounters paths that are not adjacent, while this case can still reach a point of knots. For example, let $r_1$, $r_2$ and $r_3$ be three regions, satisfying $O=(r_1 \cap r_2 \cap r_3) \ne \emptyset$ and $r_1$ being not adjacent to $r_3$. Denote by $\boldsymbol{x}_1$ and $\boldsymbol{x}_2$ two points, with $\boldsymbol{x}_1 \in r_1$, $\boldsymbol{x}_3 \in r_3$ and line segment $\boldsymbol{x}_1\boldsymbol{x}_3 \cap O \ne \emptyset$. Then paths $p_1$ and $p_3$ for $r_1$ and $r_3$, respectively, are not adjacent but still provide a knot point belonging to $O$ by algorithm 1 as $\Delta x = |\boldsymbol{x}_3-\boldsymbol{x}_1|\to 0$. The general principle is similar.
\end{proof}

\subsection{Solution Explanation: First Example}

\begin{figure}[!t]    
\captionsetup{justification=centering}
\centering
\subfloat[Example \Romannum{1}: data-fitting effect.]{\includegraphics[width=2.5in, trim = {1.1cm 0.7cm 0.9cm 0.1cm}, clip]{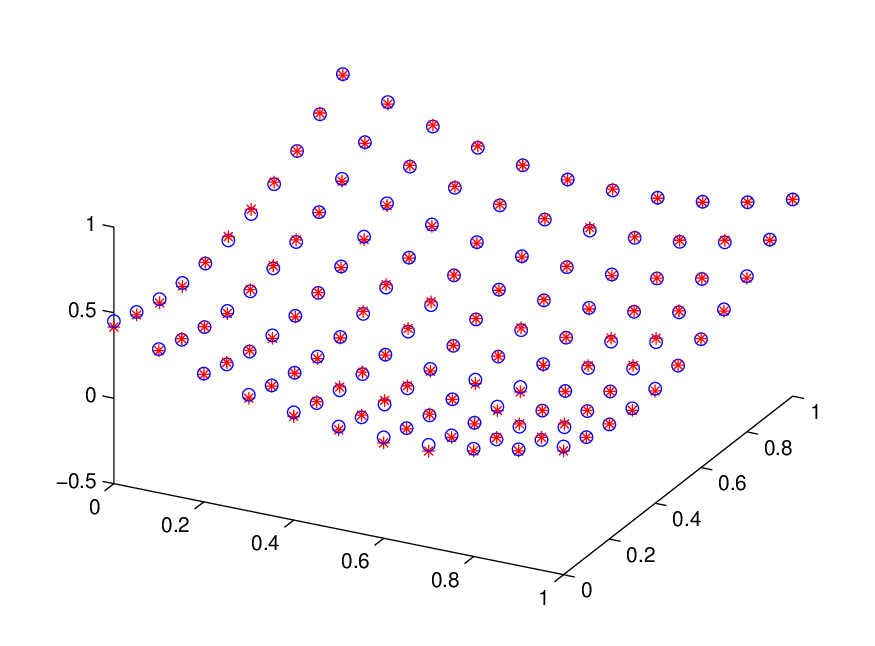}} \quad
\subfloat[Solution of example \Romannum{1}.]{\includegraphics[width=3.1in, trim = {1.2cm 0.7cm 0.9cm 0.1cm}, clip]{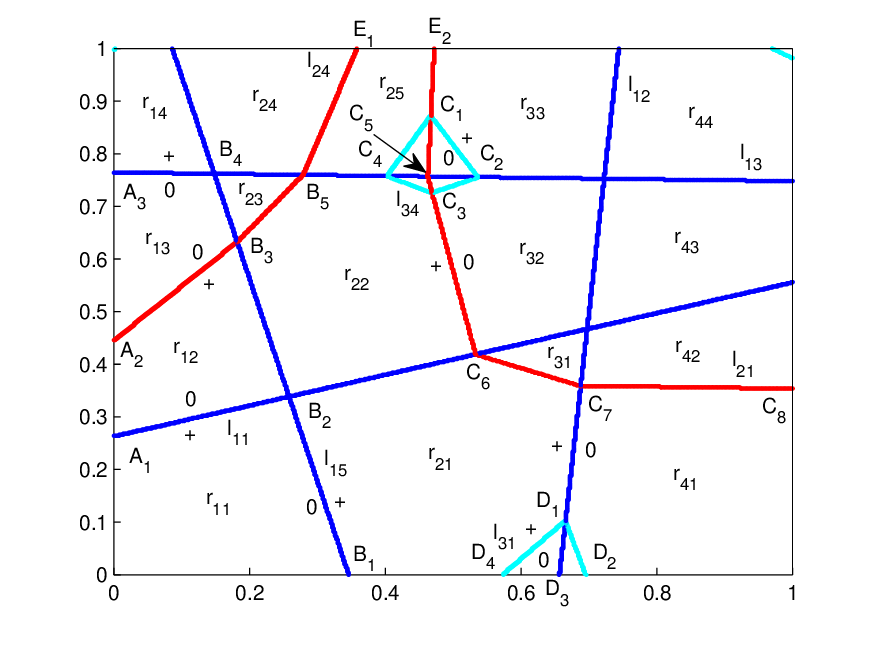}}
\caption{Explanation of training solutions-example \Romannum{1}.}
\label{Fig.4}
\end{figure}

In Figure \ref{Fig.4}a, the points of blue circles are from function $z=(x-0.6)^2+(y-0.3)^2$ with both $x$ and $y$ dimensions being discretized by step $\Delta=0.1$. The red asterisks are produced by network $\mathfrak{N}$ whose parameter settings are as follows: the depth or the number of hidden layers is $\Phi=3$, with each layer having 6 units; the learning rate is $c=0.05$; the number of training iterations is $N=5000$; the weights and bias of each unit are initialized by uniform distribution on interval $[-1, 1]$.

Through algorithm 1 of section 8.2 with parameter $\Delta x=0.002$, Figure \ref{Fig.4}b depicts the knots of the training solution of $\mathfrak{N}$ that yields the red asterisk points of Figure \ref{Fig.4}a. The blue, red or cyan curves correspond to the units of the first, second and third layers of $\mathfrak{N}$, respectively. A piecewise linear curve of Figure \ref{Fig.4}b is the one that has the same color and only changes its direction when meeting a knot generated in a shallower layer. For simplicity of descriptions, we use the term ``curve'' to represent both lines and piecewise linear curves. Let $u_{ij}$ for $1\le i\le 3$ and $1\le j\le 6$ be the $j$th unit of the $i$th layer of $\mathfrak{N}$. Denote by $l_{ij}$ the curve produced $u_{ij}$.

The notation $l_{ij}^+$ (or $l_{ij}^0$) represents the part of $U$ divided by $l_{ij}$ that can activate (or inactivate) the unit $u_{ij}$. A region is denoted by $r_{\nu\mu}$ for some $\nu$ and $\mu$, its corresponding path by $p_{\nu\mu}$ and the linear function on it by $s_{\nu\mu}(\boldsymbol{x})$. A capital letter, such as $A_1$, is of a point that is usually the intersection of knots; when a region is very small, it is denoted by its vertexes (e.g., $D_1D_2D_4$ of Figure \ref{Fig.4}b).

We now see how to interpret the training solution of Figure \ref{Fig.4}b. The first step is for one-sided solutions. The knots are regarded as being arranged in strict partial orders, regardless of negative units. This step contains the main principles of solution construction. The second step turns to the actual Figure \ref{Fig.4}b and emphasizes the mechanism of two-sided solutions.

\subsubsection{One-sided solutions}
First see the function construction in polygon $R_1 = l_{15}^0\cap U$ of Figure \ref{Fig.4}b. The knots in $R_1$ form a strict partial order
\begin{equation}
\mathscr{O}_1 = A_1B_2 \prec A_2B_3 \prec A_3B_4.
\end{equation}
The path $p_{11}$ for region $r_{11}$ is selected to be the initial path and the linear function $s_{11}(\boldsymbol{x})$ on $r_{11}$ is realized by the universal global units of the whole domain $U$ by corollary 11. After that, according to theorem 23, the remaining linear functions on $R_1$ can be implemented one by one via adding units in $p_{11}$. The knots in equation 8.6 are realized by corollary 10 through adjusting the input parameters of units.

Next, deal with the second strict partial order
\begin{equation}
\mathscr{O}_2 = B_1B_2 \prec D_1C_7.
\end{equation}
Introduce unit $u_{15}$ in the first layer of $p_{11}$ to form path $p_{21}$. The input parameters of $u_{15}$ are set to yield knot $B_1B_2$ and the output weights for the linear function $s_{21}(\boldsymbol{x})$ on $r_{21}$. Since the line of $B_1B_2$ intersects $A_2B_3$, $u_{15}$ could change the direction of $A_2B_3$; we use theorem 14 to control the production of knot $B_3B_5$, through adjusting the output weights of $u_{15}$ simultaneously without influencing $s_{21}(\boldsymbol{x})$. Knot $B_5E_1$ can be analogously produced via setting the output weights of $u_{13}$.

The previous operation for $u_{15}$ can lead to the linear functions on several other regions by the continuity-restriction principle of theorem 21. In Figure \ref{Fig.4}b, we find that
\begin{equation}
r_{12} + r_{21} \to r_{22},
\end{equation}
which means that the linear functions on $r_{12}$ and $r_{21}$ lead to the one on $r_{22}$ by theorem 21. Recursively, we have $r_{22}+r_{13} \to r_{23}$, $r_{23}+r_{14} \to r_{24}$, and $r_{24}+r_{22} \to r_{25}$. By only one operation for region $r_{21}$, nearly all the linear functions on $l_{15}^+ \cap l_{21}^+\cap l_{12}^+$ are automatically realized, demonstrating the power of continuity-restriction principle. Then add unit $u_{12}$ in the first layer of $p_{21}$, whose input parameters are for knot $D_1C_7$ and output parameters for function $s_{41}(\boldsymbol{x})$ on $r_{41}$.

The third strict partial order is
\begin{equation}
\mathscr{O}_3 = B_1B_2 \prec D_4D_1.
\end{equation}
Introduce $u_{31}$ in the third layer of $p_{21}$ for both knot $D_4D_1$ and the linear function on $D_4D_1D_3$; knot $D_1D_2$ is by the output-weight vector $\boldsymbol{v}_{12}$ of $u_{12}$. We also have $D_4D_1D_3+r_{41} \to D_1D_2D_3$.

The fourth one is
\begin{equation}
\mathscr{O}_4 = B_1B_2 \prec C_6C_7.
\end{equation}
Add unit $u_{21}$ in the second layer of $p_{21}$, whose input parameters are for knot $C_6C_7$ and output parameters for $s_{31}(\boldsymbol{x})$ on $r_{31}$. Besides $C_6C_7$, we see how the remaining part of curve $l_{21}$ are constructed. Knot $C_7C_8$ is controlled by the output-weight vector $\boldsymbol{v}_{12}$ of $u_{12}$. Note that $D_1D_2$ above is also formed by $\boldsymbol{v}_{12}$ and this means that $\boldsymbol{v}_{12}$ should simultaneously realize both $C_7C_8$ and $D_1D_2$; this is possible because the problem amounts to a solution of two linear equations for unknown $\boldsymbol{v}_{12}$, analogously to theorem 15. Knot $C_6C_3$ is by the output parameters of $u_{11}$. Knot $C_3E_2$ is through the output-weight vector $\boldsymbol{v}_{13}$ of $u_{13}$, which is also responsible for $B_5E_1$.

Because curve $l_{21}$ is long, its effect of continuity-restriction principle is significant. In Figure \ref{Fig.4}b, we find that $r_{22}+r_{31} \to r_{32}$, $r_{32}+r_{25} \to r_{33}$, $r_{31}+r_{41} \to r_{42}$, $r_{42}+r_{32} \to r_{43}$ and $r_{43}+r_{33} \to r_{44}$.

The last one is
\begin{equation}
\mathscr{O}_5= B_2C_6 \prec C_3C_4.
\end{equation}
In the third layer of $p_{22}$ introduce unit $u_{34}$, whose input and output parameters are set for knot $C_3C_4$ and the linear function on $C_3C_4C_5$, respectively. Knots $C_4C_1$ and $C_1C_2$ are controlled by the output-weight vectors of $u_{13}$ and $u_{21}$, respectively; after the previous steps, knot $C_2C_3$ can be automatically implemented because its two endpoints have been fixed by the preceding operations. Continuity-restriction principle leads to $C_3C_4C_5+r_{25} \to C_1C_4C_5$, $C_1C_4C_5+r_{33} \to C_1C_2C_5$, $C_1C_2C_5+r_{32} \to C_2C_3C_5$.

The minor region at the right top corner of $U$ can be easily processed. Up to now, all the regions of Figure \ref{Fig.4}b are covered by the principles of strict partial orders and continuity restriction, and a desired piecewise linear function on them is constructed.

\subsubsection{Two-sided solutions}
First, dealt with $-A_1B_2$ and $-A_2A_3$ for $\mathscr{O}_1 = A_1B_2 \prec A_2B_3 \prec A_3B_4$. On the basis of the preceding one-sided solution of $\mathscr{O}_1$, change $A_1B_2$ into $-A_1B_2$ by theorem 24; after that, use theorem 25 or proposition 6 to modify $A_2B_3$ to $-A_2B_3$. The altered order is denoted by $\mathscr{O}_1'=-A_1B_2 \prec -A_2B_3 \prec A_3B_4$, corresponding to the actual case of Figure \ref{Fig.4}b; the solution of $\mathscr{O}_1'$ is denote by $\mathcal{S}^{(1)}_1$, where the superscript ``(1)'' indicates the updating times and the subscript is the order index; the rule of notation ``$\mathcal{S}^{(1)}_1$'' is applicable to the remaining orders.

Then turn to $-D_1C_7$ for $\mathscr{O}_2 = B_1B_2 \prec D_1C_7$. The one-sided solution of $\mathscr{O}_2$ should be first updated, because $\mathscr{O}'_1$ can affect it. The influence of $-D_1C_7$ on $\mathscr{O}_1'$ is equivalent to that of a global unit; and in combination with the depth of $u_{12}$, it can be processed by theorem 24. The result is $\mathscr{O}_2'=B_1B_2 \prec -D_1C_7$ and the solutions $\mathcal{S}^{(2)}_1, \mathcal{S}^{(1)}_2$.

The third one is $-D_4D_1$ for $\mathscr{O}_3 = B_1B_2 \prec D_4D_1$. First update the one-sided solution of $\mathscr{O}_3$ based on $\mathscr{O}_1'$ and $\mathscr{O}_2'$. According to the relationships between $\mathscr{O}_1',\mathscr{O}_2',\mathscr{O}_3$ as well as the depth of $u_{31}$, theorem 25 can be applied for the two-sided solution of $-D_4D_1$. The parameters for $B_3B_5$, $B_5E_1$ and $D_1D_2$ should also be updated accordingly. We obtain $\mathscr{O}_3' = B_1B_2 \prec -D_4D_1$ and the solutions $\mathcal{S}^{(3)}_1, \mathcal{S}^{(2)}_2, \mathcal{S}^{(1)}_3$.

The fourth is to process $-C_6C_7$ for $\mathscr{O}_4 = B_1B_2 \prec C_6C_7$. Also first update the one-sided solution of $\mathscr{O}_4$ based on $\mathscr{O}_1',\mathscr{O}_2', \mathscr{O}_3'$. As shown in Figure \ref{Fig.4}b, by the depth of $u_{21}$ and the fact that $u_{31}$ is introduced before $u_{21}$ and lies in a deeper layer, the parameters should be reset by theorem 26. Note that theorem 26 includes the operation of theorem 24 for unit $u_{24}$ that is in the same layer as $u_{21}$. The obtained result is $\mathscr{O}_4' = B_1B_2 \prec -C_6C_7$ and the solutions $\mathcal{S}^{(4)}_1, \mathcal{S}^{(3)}_2, \mathcal{S}^{(2)}_3, \mathcal{S}^{(1)}_4$.

The last one is $-C_3C_4$ for $\mathscr{O}_5 = B_2C_6 \prec C_3C_4$. First, reconstruct the one-sided solution of $\mathscr{O}_5$ on the basis of $\mathscr{O}_i'$ for $1\le i\le4$. Second, by the depth of $u_{34}$, the parameters of the units are updated by theorem 25; and theorem 25 contains the operation of theorem 24 for $u_{31}$. We get $\mathscr{O}_5' = B_2C_6 \prec -C_3C_4$ and the solutions $\mathcal{S}^{(5)}_1, \mathcal{S}^{(4)}_2, \mathcal{S}^{(3)}_3, \mathcal{S}^{(2)}_4, \mathcal{S}^{(5)}_1$. This is the final two-sided solution.

\begin{figure}[!t]
\captionsetup{justification=centering}
\centering
\subfloat[Example \Romannum{2}: data-fitting effect.]{\includegraphics[width=2.5in, trim = {1.2cm 0.7cm 0.9cm 0.1cm}, clip]{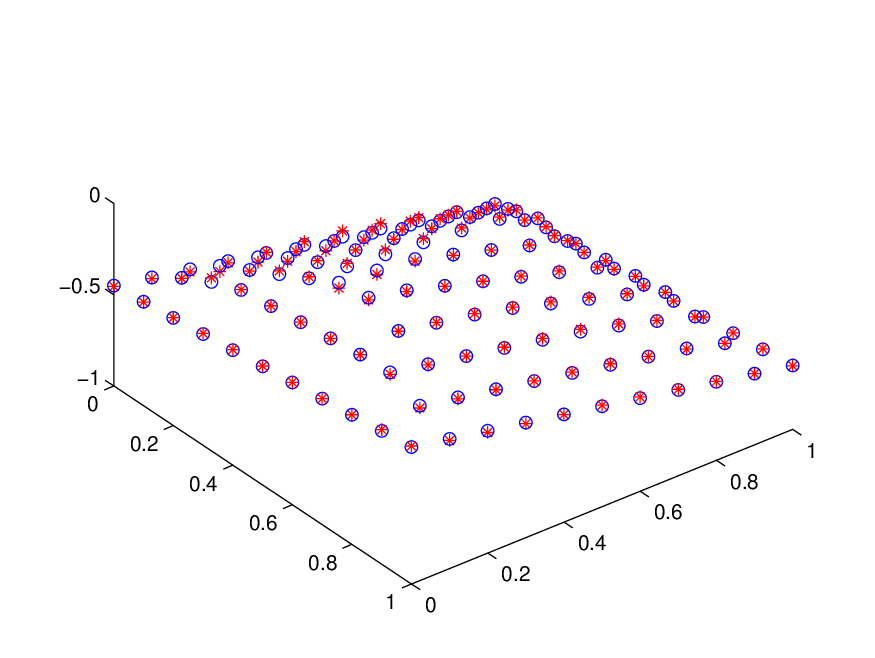}} \quad
\subfloat[Solution of example \Romannum{2}.]{\includegraphics[width=3.1in, trim = {1.2cm 0.7cm 0.9cm 0.1cm}, clip]{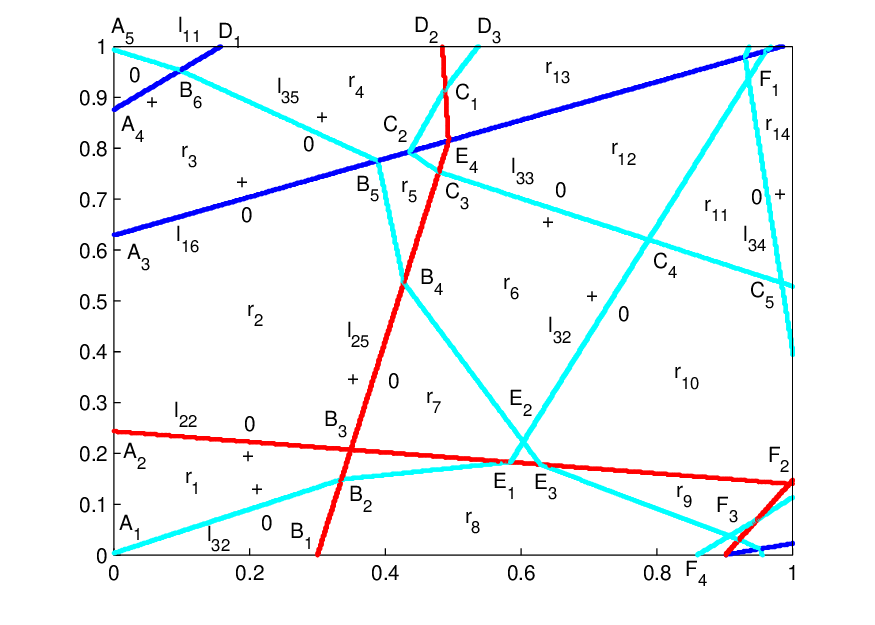}}
\caption{Explanation of training solutions-example \Romannum{2}.}
\label{Fig.5}
\end{figure}

\subsection{Solution Explanation: Second Example}

The blue-circle points of Figure \ref{Fig.5}a is discretized from a continuous piecewise linear function on $U$ and the red-asterisk points are produced by a network $\mathfrak{N}$ with three hidden layers. The parameter settings of $\mathfrak{N}$ are the same as those of the first example of section 8.3. Figure \ref{Fig.5}b is depicted analogously to Figure \ref{Fig.4}b and the notations are also similar.

\Romannum{1}. \textbf{One-sided solution}. In this case, we summarize the method of section 8.3 into several steps, with the details and explanations omitted since the underlying principles have been discussed in section 8.3.
\begin{itemize}
\item[(1)] An initial path of $\mathfrak{N}$ implements the linear function on root region $A_1B_1B_2$. Construct the linear functions on polygon $A_1B_1B_4B_5B_6A_5$ based on the strict partial order
\begin{equation}
\mathscr{O}_1 = A_1B_2 \prec A_2B_3 \prec A_3B_5 \prec A_4B_6.
\end{equation}

\item[(2)] Process
\begin{equation}
\mathscr{O}_2 = B_1B_2 \prec E_3F_3
\end{equation}
for the linear functions on $r_8$ and $r_9$. By the continuity-restriction principle of theorem 21, as shown in Figure \ref{Fig.5}b, we have: $r_1 + r_8 \to B_2E_1B_3$, $B_2E_1B_3 + r_2 \to r_7$, $r_7 + r_8 \to E_1E_2E_3$, $E_1E_2E_3 + r_9 \to r_{10}$, $r_{10} + r_7 \to r_6$, $r_6 + r_2 \to r_5$, $r_5 + r_3 \to r_4$, $r_4 + A_4B_6A_5 \to A_5B_6D_1$.
\item[(3)] The third strict partial order is
\begin{equation}
\mathscr{O}_3 = E_3F_2 \prec C_4C_5
\end{equation}
generating the linear functions on $r_{10}$ and $r_{11}$, after which we have: $r_6 + r_{11} \to r_{12}$, $r_{12} + r_5 \to C_2C_3E_4$, $C_2C_3E_4 + r_4 \to C_2E_4C_1$, $C_2E_4C_1+r_{12} \to r_{13}$, $r_{13}+r_4 \to C_1D_2D_3$.

\item[(4)] The fourth one is
\begin{equation}
\mathscr{O}_4 = C_4C_5 \prec F_1C_5
\end{equation}
for the linear function on $r_{14}$.
\end{itemize}

The remaining regions at the right bottom and top corners of $U$ are too small to be labelled and their process is analogous to the above regions. Up to now, all the regions of $U$ are covered by the strict partial orders and the continuity-restriction principle, through which a desired piecewise linear function is constructed; the knot control for the partition is by the principles of section 2.3.

\Romannum{2}. \textbf{Two-sided solution}: As shown in Figure \ref{Fig.5}b, $\mathscr{O}_1' = A_1B_1 \prec -A_2B_3 \prec A_3B_5 \prec -A_4B_6$ is processed by theorem 24. $\mathscr{O}_2'= -B_1B_2 \prec E_3F_3$ is through theorem 26. $\mathscr{O}_3' = E_3F_2 \prec -C_4C_5$ is by theorem 25. The one-sided solution of the remaining strict partial orders should be updated based on the preceding operations.

\section{Summary for Black box}
We summarize the results of this paper and highlight the key points: section 9.1 is the summary of the main principles; sections 9.2, 9.3 and 9.4 are for the emphasizing of the mechanism of hidden-layer units, the effect of deep layers, the continuity-restriction principle and the information extraction from data, respectively.

\subsection{Main Principles}
Partition formation and function implementation are the two components of the black box of network $\mathfrak{N}$, which are correlated with each other since the knots for the former could be produced by the latter. To realize a desired partition $\mathcal{P}=\{r_i: 1\le i\le \zeta\}$ of $U$ as well as a desired piecewise linear function $f(\boldsymbol{x})$ on $\mathcal{P}$ via $\mathfrak{N}$, the principles required are summarized as follows:
\begin{itemize}
\item[\Romannum{1}] The concept of a path. Each region $r_i$ of $\mathcal{P}$ corresponds to a unique path $p_i$ of $\mathfrak{N}$. The adjacent paths of $p_i$ yield the geometrically adjacent regions of $r_i$ (theorems 8 and 9) and all the cases of $p_i$'s contribute to the whole partition $\mathcal{P}$. A path $p_i$ realizes both the linear function on $r_i$ and at least one knot of $r_i$ (theorem 13).

\item[\Romannum{2}] The mechanism of knot production. There are two main principles for a unit $\mathcal{U}$ to generate a knot $\mathcal{K}$. One is by directly setting the input parameters of $\mathcal{U}$ (corollary 10) and the other is by indirectly adjusting the output parameters of some other units in the layers shallower than that of $\mathcal{U}$ (corollary 9).

\item[\Romannum{3}] The principles of function implementation. The mechanism of realizing $f(\boldsymbol{x})$ via $\mathfrak{N}$ is similar to that of a two-layer ReLU network, including a recurrence formula $s_{\nu}(\boldsymbol{x}) = s_{\nu-1}(\boldsymbol{x}) + \lambda_{\nu}\sigma(\boldsymbol{w}_{\nu}^T\boldsymbol{x} + b_{\nu})$ for the linear functions on adjacent regions (theorem 13), multiple strict partial orders (theorem 32) and the continuity-restriction principle (theorem 21).

\item[\Romannum{4}] Hidden-layer units for both knot production and function implementation. The output-weight vector of a unit of hidden layers can be used simultaneously for the above two purpose (e.g., theorems 14 and 15).

\item[\Romannum{5}] Piecewise linear manifold of a unit. To the case of two-dimensional input, this means that a unit of $\mathfrak{N}$ can form a piecewise linear curve to divide the input pace, and even a closed one (see polygon $C_1C_2C_3C_4$ of Figure \ref{Fig.4}b), through which the formed partition could more flexibly fit the geometric feature of $f(\boldsymbol{x})$.

\item[\Romannum{6}] Two-sided solutions. This point accounts for the diversity of training solutions obtained by the back-propagation algorithm. The mechanism is to modify the one-side solution with the aid of the principles of knot production and function implementation.

\item[\Romannum{7}] The mechanism of multiple outputs. This principle is similar to that of the above item \Romannum{4} and can also be attributed to the effect of hidden-layer units (theorem 27 and corollary 14).

\item[\Romannum{8}] The combination of the preceding principles could yield complex instantiations or various concrete solutions including the training ones.

\end{itemize}

\subsection{Effect of Hidden-Layer Units}
Throughout this paper, a unit in a hidden layer of $\mathfrak{N}$ plays an important role in forming both knots and linear functions, manifested by such as equations 3.17, 3.20, 3.21 and 5.4. By an example, write
\begin{equation}
\begin{cases}
\begin{aligned}
{\boldsymbol{\alpha}_i}^T\boldsymbol{v}_i &= \lambda_{1} \\
\ \boldsymbol{\alpha}_{i1}^T\boldsymbol{v}_{i} &= \beta_{1} \\
\boldsymbol{\alpha}_{i2}^T\boldsymbol{v}_{i} &= \beta_{2}
\end{aligned}
\end{cases}
\end{equation}
generalized from equation 3.20 of theorem 14, in which, for instance,
\begin{equation}
\boldsymbol{\alpha}_i = W_{\nu+2}W_{\nu+3}\dots \boldsymbol{w}_{\Phi+1}
\end{equation}
is a constant vector, where $ W_{j}$'s for $\nu+1\le j \le \Phi$ and $\boldsymbol{w}_{\Phi+1}$ are the input matrices of the associated layers. Equation 9.1 uses the output-weight vector $\boldsymbol{v}_i$ of a unit $\mathcal{U}$ of the $\nu$th layer to realize a linear function (via parameter $\lambda_{1}$) as well as two knots (through parameters $\beta_{1}$ and $\beta_{2}$) simultaneously. For the existence of a solution of $\boldsymbol{v}_{i}$, the length $m_{\nu+1}$ of vector $\boldsymbol{v}_{i}$, also the number of the units of the $\nu+1$th layer, should be greater than or equal to 3; this is a meaning of the number of the units of a hidden layer.

Equation 9.1 can be expressed as
\begin{equation}
A\boldsymbol{v}_{i}=\boldsymbol{b},
\end{equation}
with $A=[{\boldsymbol{\alpha}_i}, {\boldsymbol{\alpha}_{i1}}, {\boldsymbol{\alpha}_{i1}}]^T$ and $\boldsymbol{b}=[\lambda_{i}, \beta_{1}, \beta_{2}]^T$. The input matrices of the layers determine the rank of matrix $A$, which is also related to the existence of a solution of $\boldsymbol{v}_{i}$; and this is one meaning of $W_j$'s and $\boldsymbol{w}_{\Phi+1}$ of equation 9.2, the input matrices of the hidden layers.

Notice that $\boldsymbol{v}_{i}$ and $A$ of equation 9.3 are independent of the input space $\mathbb{R}^n$, that is, the solution space of $\boldsymbol{v}_{i}$ is not related to the input space, such that the complexity of the solution of network $\mathfrak{N}$ is not restricted by the input dimensionality.

Another example
\begin{equation}
\begin{cases}
\begin{aligned}
{\boldsymbol{\alpha}_i}^T\boldsymbol{v}_i &= \lambda_{1} \\
\ \boldsymbol{\alpha}_{i1}^T\boldsymbol{v}_{i} &= \lambda_{2} \\
\boldsymbol{\alpha}_{i2}^T\boldsymbol{v}_{i} &= \beta_{1}
\end{aligned}
\end{cases}
\end{equation}
is the combination of theorems 27 and 14, whose first two formulas are for two linear functions output by multiple units as in equation 5.5 of theorem 27 and the third formula is for a knot. Equation 9.4 can be explained similarly to equation 9.1.

The above two examples indicate that more units in a hidden layer are correlated with the expressive capability of $\mathfrak{N}$, since the larger the length of vector $\boldsymbol{v}_i$, the more linear functions or knots that a unit $\mathcal{U}$ may produce.

Under all the paths of $\mathfrak{N}$, the effect of the equations similar to equations 9.1 and 9.4 could be complicated in forming various knots and linear functions, for which the hidden-layer units account for a significant part of the mechanism of deep ReLU networks.

\subsection{Effect of Deep Layers}

The theory of network $\mathfrak{N}$ developed in this paper can include that of a two-layer ReLU network $\mathcal{N}$ \citep*{Huang2024} as a special case, in terms of the concept of a path. A path of $\mathcal{N}$ is the simplest type with only one hidden layer, such that the adjacent-path condition can be neglected. This simplicity on one side may be useful in more easily finding a training solution, but on the other side restricts the solution complexity.

The path complexity of $\mathfrak{N}$, especially in terms of the adjacent-path condition, may be the reason of the difficulty in training $\mathfrak{N}$, for which even a pre-training step is required \citep*{Bengio2009}. However, the solution of $\mathfrak{N}$ is usually better than that of $\mathcal{N}$ such that deep learning is predominant nowadays.

Network $\mathfrak{N}$ can form piecewise linear manifolds to divide the input space $\mathbb{R}^n$, while $\mathcal{N}$ can only use $n-1$-dimensional hyperplanes. The former is thus more flexible and powerful in generating complicated partitions to fit the geometric feature of a function to be approximated.

As the depth $\Phi$ of $\mathfrak{N}$ grows, a unit of deeper layers is more likely to encounter a unit of shallower layers and influenced by it in terms of changing the normal vector of the associated knot. When this effect is intensive enough, a piecewise linear manifold would appear as a smooth manifold or the former could approximate the latter with a desired accuracy, and it can be imagined that when the depth $\Phi$ and the number of units are sufficiently large, $\mathfrak{N}$ can nearly use a smooth manifold to divide the input space.

By the remarks of theorem 7, it is possible that a deeper network $\mathfrak{N}$ leads to smaller regions or finer partitions of the input space. First, the adjacent-path condition tends to force a region to be divided when introducing units in a shallow layer. Second, a new unit introduced in a deeper layer could lead to the subdivision of the regions formed in the shallower layers. Both of the above two cases result in finer partitions.

\subsection{Effect of Continuity-Restriction Principle}
This mechanism accounts for one of the ingredients of parameter sharing for function construction---that is, a set of parameters can simultaneously generate multiple desired linear functions by this principle. Without continuity restriction, one can hardly imagine how to adjust the parameters to simultaneously fulfil so many linear functions of a piecewise linear function.

When only considering strict partial orders, the solution is trivial and simple. However, when coupled with continuity restriction as well as the geometric feature of a partition, the combined effect could be rich and complicated, leading to various concrete solutions; and this is one of the sources of the expressive capability of network $\mathfrak{N}$.

Continuity restriction plays a central role in forming a complex solution and is the essence that distinguishes a ReLU network from other approximators (such as Fourier series), without which a solution can hardly be formed.

\subsection{Embedding of Data Information}
It's enlightening that \citet*{Simon2026} emphasized the importance of the data to be fitted. The geometric information of data set $D$ is embedded in network $\mathfrak{N}$ through two forms. One is of the parameters for partition $\mathcal{P}$ of the input space and the other is of the parameters implementing the linear functions on the regions of $\mathcal{P}$. Different partitions yield different piecewise linear functions, leading to distinct generalization capability or property. Since deep network $\mathfrak{N}$ can use piecewise linear manifolds rather than only hyperplanes to divide the input space, more suitable information can be extracted from $D$ to make the generalization more powerful or precise.

\section{Discussion}
This paper tried to exhaust all the basic principles of deep feedforward ReLU networks for function approximation or data interpolation. If the principles given are complete---that is, any concrete solution can be derived from them or explained by them, the theoretical framework is accomplished; otherwise, more mechanisms need to be discovered.

Despite the simplicity of the principles, their combination could yield complicated instantiations including the training solution obtained by the back-propagation algorithm. How to manually construct a solution is a direction of future researches. Specifically, given a multivariate function $f(\boldsymbol{x})$ or a data set $D$ discretized from $f(\boldsymbol{x})$, we want to know what a partition $\mathcal{P}$ via piecewise linear manifolds fits it and how to realize $\mathcal{P}$ as well as a desired piecewise linear function over $\mathcal{P}$. A manually constructed solution, instead of the one obtained by training methods, is controllable, interpretable and may be much more economical.

Throughout the paper, we can see the rich content of the solution space of deep feedforward ReLU networks, which is related to high-dimensional geometries. A deep ReLU network is a new representation of multivariate function $f(\boldsymbol{x})$, fitting its geometric feature in a novel way. The property of this new representation may be rich and interesting.

The research methodology is essentially of theoretical physics. Our success proves that the complexity of neural networks is by no means beyond the capability of a traditional way of obtaining knowledge---deduction, whose root dates back to Euclidean geometry of ancient Greece.


\begin{thebibliography}{100}
\providecommand{\natexlab}[1]{#1}
\expandafter\ifx\csname urlstyle\endcsname\relax
  \providecommand{\doi}[1]{doi:\discretionary{}{}{}#1}\else
  \providecommand{\doi}{doi:\discretionary{}{}{}\begingroup
  \urlstyle{rm}\Url}\fi

%
%

\bibitem[{Argerich \& Pati\~{n}o-Mart\'{i}nez(2024)Argerich \& Pati\~{n}o-Mart\'{i}nez}]{Argerich2024}
Argerich, M. F. \& Pati\~{n}o-Mart\'{i}nez, M. (2024).
\newblock Measuring and improving the energy efficiency of large language models inference.
\newblock \emph{IEEE Access}, 12, 80194--80207.

\bibitem[{Bengio (2009)Bengio}]{Bengio2009}
Bengio Y. (2009).
\newblock Learning deep architectures for AI.
\newblock \emph{Foundations and Trends in Machine Learning}, 2(1), 1--127.

\bibitem[{Bengio et~al.(2025)Bengio et~al.}]{Bengio2025}
Bengio, Y., Mindermann, S., Privitera, D., Besiroglu, T., Bommasani, R., Casper, S., \dots \& Zeng, Y. (2025).
\newblock International ai safety report.
\newblock \emph{arXiv:2501.17805}.

\bibitem[{Bunge(1973)Bunge}]{Bunge1973}
Bunge, M. (1973).
\newblock Philosophy of Physics.
\newblock \emph{D. Reidel Publishing Company}.

\bibitem[{Daubechies et~al.(2022)Daubechies et~al.}]{Daubechies2022}
Daubechies, I., DeVore, R., Foucart, S., Hanin, B., \& Petrova, G. (2022).
\newblock Nonlinear approximation and (deep) ReLU networks.
\newblock \emph{Constructive Approximation}, 55(1), 127-172.

\bibitem[{DeVore, Hanin, \& Petrova(2021)DeVore, Hanin, \& Petrova}]{DeVore2021}
DeVore, R., Hanin, B., \& Petrova, G. (2021).
\newblock Neural network approximation.
\newblock \emph{Acta Numerica}, 327--444.

\bibitem[{Du et~al.(2022)Du et~al.}]{Du2022}
Du, K. L., Leung, C. S., Mow, W. H., \& Swamy, M. N. S. (2022).
\newblock Perceptron: Learning, generalization, model selection, fault tolerance, and role in the deep learning era.
\newblock \emph{Mathematics}, 10(24), p.4730.

\bibitem[{Erhan et~al.(2010)Erhan et~al.}]{Erhan2010}
Erhan, D., Courville, A., Bengio, Y., \& Vincent, P. (2010).
\newblock Why does unsupervised pre-training help deep learning?.
\newblock \emph{In proceedings of the 13th international conference on artificial intelligence and statistics} (AISTATS), 201--208.

\bibitem[{Gauch(2003)Gauch}]{Gauch2003}
Gauch, H. G. (2003).
\newblock Scientific method in practice.
\newblock Cambridge University Press, 269--326.

\bibitem[{Girin et~al.(2022)Girin et~al.}]{Girin2022}
Girin, L., Leglaive, S., Bie, X., Diard, J., Hueber, T., \& Alameda-Pineda, X., (2022).
\newblock Dynamical variational autoencoders: A comprehensive review.
\newblock \emph{Foundations and Trends in Machine Learning}, 15(1-2), 1--175.

\bibitem[{Glorot\& Bengio(2010)Glorot \& Bengio}]{Glorot2010}
Glorot, X \& Bengio, Y. (2010).
\newblock Understanding the difficulty of training deep feedforward neural networks.
\newblock \emph{In proceedings of the 13th international conference on artificial intelligence and statistics} (AISTATS), 249--256.

\bibitem[{Guth et~al.(2024)Guth et~al.}]{Guth2024}
Guth, F., M\'{e}nard, B., Rochette, G., \& Mallat, S. (2024).
\newblock A rainbow in deep network black boxes.
\newblock \emph{Journal of Machine Learning Research}, 25(350), 1--59.

\bibitem[{Haykin(2009)Haykin}]{Haykin2009}
Haykin, S. (2009).
\newblock Neural networks and learning machines (3rd ed.).
\newblock Pearson Prentice Hall, 122--221.

\bibitem[{Hendrycks(2025)Hendrycks}]{Hendrycks2025}
Hendrycks, D. (2025).
\newblock Introduction to AI safety, ethics, and society.
\newblock CRC Press.

\bibitem[{Huang(2020)Huang}]{Huang2020}
Huang, C. (2020).
\newblock ReLU networks are universal approximators via piecewise linear or constant functions.
\newblock \emph{Neural Computation}, \emph{32(11)}, 2249--2278.

\bibitem[{Huang(2024)Huang}]{Huang2024}
Huang, C. (2024).
\newblock On the principles of ReLU networks with one-hidden layer.
\newblock arXiv:2411.06728.

\bibitem[{Jagtap, Kawaguchi, \& Karniadakis(2020)Jagtap, Kawaguchi, \& Karniadakis}]{Jagtap2020}
Jagtap, A. D., Kawaguchi, K., \& Karniadakis, G. E. (2020).
\newblock Adaptive activation functions accelerate convergence in deep and physics-informed neural networks.
\newblock \emph{Journal of Computational Physics}, 404, 109--136.

\bibitem[{Kandel et~al.(2021)Kandel et~al.}]{Kandel2021}
Kandel, E. R., Koester, J. D., Mack, S. H., \& Siegelbaum, S. A. (Eds.). (2021).
\newblock Principles of neural science (6th ed.).
\newblock New York: McGraw-hill, 84--88.

\bibitem[{Ramanujan et~al.(2020)Ramanujan et~al.}]{Ramanujan2020}
Ramanujan, V., Wortsman, M., Kembhavi, A., Farhadi, A., \& Rastegari, M. (2020).
\newblock What's hidden in a randomly weighted neural network?.
\newblock \emph{In proceedings of the IEEE/CVF conference on computer vision and pattern recognition} (CVPR), 11893--11902.

\bibitem[{Raissi, Perdikaris, \& Karniadakis(2019)Raissi, Perdikaris, \& Karniadakis}]{Raissi2019}
Raissi, M., Perdikaris, P., \& Karniadakis, G. E. (2019).
\newblock Physics-informed neural networks: A deep learning framework for solving forward and inverse problems involving nonlinear partial differential equations.
\newblock \emph{Journal of Computational physics}, 378, 686--707.

\bibitem[{Rumelhart, Hinton, \& Williams(1986)Rumelhart, Hinton, \& Williams}]{Rumelhart1986}
Rumelhart, D. E., Hinton, G. E., \& Williams, R. J. (1986).
\newblock Learning representations by back-propagating errors.
\newblock \emph{Nature}, 323, 533--536.

\bibitem[{Shen, Yang, \& Zhang(2021)Shen, Yang, \& Zhang}]{Shen2021}
Shen, Z., Yang, H., \& Zhang, S. (2021).
\newblock Deep network with approximation error being reciprocal of width to power of square root of depth.
\newblock \emph{Neural Computation}, 33(4), 1005--1036.

\bibitem[{Simon et~al.(2026)Simon et~al.}]{Simon2026}
Simon, J., Kunin, D., Atanasov, A., Boix-Adser\`{a}, E., Bordelon, B., Cohen, J., Ghosh, N., Guth, F., Jacot, A., Kamb, M., Karkada, D., Michaud, E. J., Ottlik, B., \& Turnbull, J. 2026.
\newblock There will be a scientific theory of deep learning.
\newblock arXiv:2604.21691.

\bibitem[{Strubell, Ganesh, \& McCallum(2019)Strubell, Ganesh, \& McCallum}]{Strubell2019}
Strubell, E., Ganesh, A., \& McCallum, A. (2019).
\newblock Energy and policy considerations for deep learning in NLP.
\newblock \emph{In proceedings of the 57th annual meeting of the association for computational linguistics}, 3645--3650.

\bibitem[{Vaswani et~al.(2017)Vaswani et~al.}]{Vaswani2017}
Vaswani, A., Shazeer, N., Parmar, N., Uszkoreit, J., Jones, L., Gomez, A. N., \& Kaiser, L. (2017).
\newblock Attention is all you need.
\newblock \emph{In proceedings of advances in neural information processing systems} (NIPS).

\bibitem[{Yang \& Zhou(2025)Yang \& Zhou}]{Yang2025}
Yang, Y. \& Zhou, D. X. (2025).
\newblock Optimal rates of approximation by shallow $\text{ReLU}^k$ neural networks and applications to nonparametric regression.
\newblock \emph{Constructive Approximation}, 62(2), 329--360.

\bibitem[{Yarotsky(2017)Yarotsky}]{Yarotsky2017}
Yarotsky, D. (2017).
\newblock Error bounds for approximations with deep ReLU networks.
\newblock \emph{Neural Networks}, \emph{94}, 103--114.

\bibitem[{Zhang et~al.(2016)Zhang et~al.}]{Zhang2016}
Zhang, C., Bengio, S., Hardt, M., Recht, B. \& Vinyals, O. (2016).
\newblock Understanding deep learning requires rethinking generalization.
\newblock arXiv:1611.03530.


\end{thebibliography}
\end{document}